\documentclass[11 pt]{article}
\usepackage{style}
\title{\LARGE\bfseries A Lyapunov Theory for Finite-Sample Guarantees of Asynchronous Q-Learning and TD-Learning Variants}
\author{Zaiwei Chen\textsuperscript{$1,*$}, Siva Theja Maguluri\textsuperscript{$1,\dagger$}, Sanjay Shakkottai\textsuperscript{$2$}, and Karthikeyan Shanmugam\textsuperscript{$3$}\\
{\small
\textsuperscript{$1$}\textit{Georgia Institute of Technology,}  \href{mailto:zchen458@caltech.edu}{\textit{\textsuperscript{$*$}zchen458@caltech.edu}}, \href{mailto:siva.theja@gatech.edu}{\textit{\textsuperscript{$\dagger$}siva.theja@gatech.edu}}
}\\
{\small
	\textsuperscript{$2$}\textit{The University of Texas at Austin,}  \href{mailto:sanjay.shakkottai@utexas.edu}{\textit{\textsuperscript{$2$}sanjay.shakkottai@utexas.edu}}
}\\
{\small
	\textsuperscript{$3$}\textit{IBM Research AI group,} \href{mailto:KarthikeyanShanmugam88@gmail.com}{\textit{\textsuperscript{$3$}KarthikeyanShanmugam88@gmail.com}}
}
}
\date{\vspace{-0.4 in}}
\begin{document}
\maketitle

\begin{abstract}
	This paper develops an unified framework to study finite-sample convergence guarantees of a large class of value-based asynchronous reinforcement learning (RL) algorithms. We do this by first reformulating the RL algorithms as \textit{Markovian Stochastic Approximation}  (SA) algorithms to solve fixed-point equations. We then develop a Lyapunov analysis and derive mean-square error bounds on the convergence of the Markovian SA. Based on this result, we establish finite-sample mean-square convergence bounds for asynchronous RL algorithms such as $Q$-learning, $n$-step TD, TD$(\lambda)$, and off-policy TD algorithms including  V-trace. As a by-product, by analyzing the convergence bounds of $n$-step TD and TD$(\lambda)$, we provide theoretical insights into the bias-variance trade-off, i.e., efficiency of bootstrapping in RL. This was first posed as an open problem in \citep{sutton1999open}.
\end{abstract}

\section{Introduction}\label{sec:intro}

Reinforcement learning (RL) is a promising approach to solve sequential decision making problems in complex and stochastic systems \citep{sutton2018reinforcement}. RL has seen remarkable successes in solving many practical problems, such as the game of Go \citep{silver2017mastering}, health care \citep{dann2019policy}, and robotics \citep{kober2013reinforcement}. Despite such empirical successes, the convergence properties of many RL algorithms are not well understood.

Most of the value-based RL algorithms can be viewed as stochastic approximation (SA) algorithms for solving suitable Bellman equations. Due to the nature of sampling in RL, many such algorithms inevitably perform the so-called \textit{asynchronous} update. That is, in each iteration, only a subset of the components of the vector-valued iterate is updated. Moreover, the components being updated are usually selected in a stochastic manner along a {\em single trajectory} based on an underlying Markov chain. Handling such asynchronous updates is one of the main challenges in analyzing the behavior of RL algorithms. In this paper, we study such asynchronous RL algorithms through the lens of Markovian SA algorithms, and develop a \textit{unified Lyapunov approach} to establish finite-sample bounds on the mean-square error. The results enable us to tackle the long-standing problem about the efficiency of bootstrapping in RL \citep{sutton1999open}.

\subsection{Main Contributions}\label{subsec:contribution}
We next summarize our main contributions in the following.

\paragraph{Finite-Sample Bounds for Markovian SA.} We establish finite-sample convergence guarantees (under various choices of stepsizes) of a stochastic approximation algorithm, which involves a contraction mapping, and is driven by both Markovian and martingale difference noise. Specifically, when using constant stepsize $\alpha$, the convergence rate is geometric, with asymptotic accuracy approximately $\mathcal{O}(\alpha\log(1/\alpha))$. When using diminishing stepsizes of the form $\alpha/(k+h)^\xi$ (where $\xi\in (0,1]$), the convergence rate is $\mathcal{O}(\log(k)/k^\xi)$, provided that $\alpha$ and $h$ are appropriately chosen.

\paragraph{Finite-Sample Bounds for $Q$-Learning.} We establish finite-sample convergence bounds of the asynchronous $Q$-learning algorithm. In the constant stepsize regime, our result implies a sample complexity of 
\begin{align*}
	\mathcal{O}\left(\frac{\log^2(1/\epsilon)}{\epsilon^2(1-\gamma)^5N_{\min}^3}\right),
\end{align*} where $N_{\min}$ is the minimal component of the stationary distribution on the state-action space induced by the behavior policy. Our result improves the state-of-the-art mean square bound of asynchronous $Q$-learning \citep{beck2013improved} by a factor of at least $|\mathcal{S}||\mathcal{A}|$. See Section \ref{subsubsec:Qliterature} for a detailed comparison with related literature.

\paragraph{Finite-Sample Bounds for V-trace.} We establish for the \textit{first} time finite-sample convergence bounds of the V-trace algorithm when performing \textit{asynchronous}  update \citep{espeholt2018impala}. The V-trace algorithm can be viewed as an off-policy variant of the $n$-step TD-learning algorithm, and uses two truncation levels $\bar{c}$ and $\bar{\rho}$ in the importance sampling ratios to control the bias and variance in the algorithm. It was discussed in \citep{espeholt2018impala} that qualitatively, $\bar{\rho}$ mainly determines the limit point of V-trace, while $\bar{c}$ mainly controls the variance in the estimate. Our finite-sample analysis quantitatively justifies this observation by showing a sample complexity bound proportional to $\bar{\rho}^2\left(\sum_{i=0}^{n}(\gamma\bar{c})^i\right)^2$. Based on this result, we see that $\bar{c}$ is the main reason for variance reduction, and we need to aggressively choose the truncation level $\bar{c}\leq 1/\gamma$ to avoid an exponential factor in the sample complexity.

\paragraph{Finite-Sample Bounds for $n$-Step TD.} We establish finite-sample convergence guarantees of the on-policy $n$-step TD-learning algorithm. In $n$-step TD, the parameter $n$ adjusts the degree of bootstrapping in the algorithm. In particular, $n=1$ corresponds to extreme bootstrapping (TD$(0)$), while $n=\infty$ corresponds to no bootstrapping (Monte Carlo method). Despite empirical observations \citep{sutton2018reinforcement}, the choice of $n$ that leads to the optimal performance of the algorithm is not theoretically understood. Based on our finite-sample analysis, we show that the parameter $n$ appears as $n/(1-\gamma^n)^2$ in the sample complexity result, therefore demonstrate an explicit trade-off between bootstrapping (small $n$) and Monte Carto method (large $n$). In addition, based on the sample complexity bound, we show that in order to achieve the optimal performance of the $n$-step TD-learning algorithm, the parameter $n$ should be chosen approximately as $\min(1,1/\log(1/\gamma))$.

\paragraph{Finite-Sample Bounds for TD$(\lambda)$.} We establish finite-sample convergence bounds of the TD$(\lambda)$ algorithm for any $\lambda$ in the interval $(0,1)$. The TD$(\lambda)$ update can be viewed as a convex combination of all $n$-step TD-learning update. Similar to $n$-step TD, the parameter $\lambda$ is used to adjust the degree of bootstrapping in TD$(\lambda)$, and there is a long-standing open problem about the efficiency of bootstrapping \citep{sutton1999open}. By deriving explicit finite-sample performance bounds of the TD$(\lambda)$ algorithm as a function of $\lambda$, we provide theoretical insight into the bias-variance trade-off in choosing $\lambda$. Specifically, in the constant-stepsize TD$(\lambda)$ algorithm, after the $k$-th iteration, the ``bias" is of the size $ (1-\Theta(1/(1-\beta\lambda)))^k$, which is in favor of large $\lambda$ (more Monte Carlo), while the ``variance" is of the size $ \Theta(1/[(1-\beta\lambda)\log(1/(\beta\lambda))])$, and is in favor of small $\lambda$ (more bootstrapping).

\subsection{Motivation and Technical Approach}\label{subsec:motivation}

In this Section, we illustrate our approach of dealing with asynchronous RL algorithms using the $Q$-learning algorithm as a motivating example. 

\subsubsection{Illustration via Q-Learning} 
The $Q$-learning algorithm is a recursive approach for finding the optimal policy corresponding to a Markov decision process (MDP) (see Section \ref{subsec:Qlearning} for details). At time step $k$, the algorithm updates a vector (of dimension state-space size $\times$ action-space size) $Q_k$, which is an estimate of the optimal $Q$-function, using noisy samples collected along a single trajectory (aka. sample-path). After a sufficient number of iterations, the vector $Q_k$ is a close approximation of the true $Q$-function, which (after some straightforward computations) delivers the optimal policy for the MDP. Concretely, let $\{(S_k,A_k)\}$ be a sample trajectory of state-action pairs collected by applying some behavior policy to the model. The $Q$-learning algorithm performs a scalar update to the (vector-valued) iterate $Q_k$ based on:
\begin{align}\label{Q:motivation1}
	Q_{k+1}(s,a)=Q_k(s,a)+\alpha_k\Gamma_1(Q_k,S_k,A_k,S_{k+1})
\end{align}
when $(s,a)=(S_k,A_k)$, and $Q_{k+1}(s,a)=Q_k(s,a)$ otherwise. Further, 
\begin{align*}
	\Gamma_1(Q_k,S_k,A_k,S_{k+1})= \mathcal{R}(S_k,A_k)+\gamma\max_{a'\in\mathcal{A}}Q_k(S_{k+1},a')- Q_k(S_k,A_k)
\end{align*}
is a function representing the \textit{temporal difference} in the $Q$-function iterate. 

At a high-level, this recursion approximates the fixed-point of the Bellman equation through samples along a single trajectory. There are, however, two sources of noise in this approximation: (1) {\em asynchronous update} where only one of the components in the vector $Q_k$ is updated (component corresponding to the state-action pair $(S_k, A_k)$ encountered at time $k$), and  other components in the vector $Q_k$ are left unchanged, and (2) {\em stochastic noise} due to the expectation in the Bellman operator being replaced by a single sample estimate $\Gamma_1(\cdot)$ at time step $k.$

\subsubsection{Reformulation through Markovian SA} To overcome the challenge of asynchronism (aka. scalar update of the vector $Q_k$), our first step is to reformulate asynchronous $Q$-learning as a {\em Markovian SA algorithm} \citep{borkar2009stochastic} by introducing an operator that captures asynchronous updates along a trajectory. A Markovian SA algorithm is an iterative approach to solve fixed-point equations (see Section \ref{sec:sa}), and leads to recursions of the form:
\begin{align}\label{eq:MSA}
	x_{k+1}=\;x_k + \alpha_k (F(x_k, Y_k) - x_k + w_k).
\end{align}
Here $x_k$ is the main iterate, $\alpha_k$ is the stepsize, $F(\cdot)$ is an operator that is (in an appropriate expected sense) contractive with respect to a suitable norm, $Y_k$ is noise derived from the evolution of a Markov chain, and $w_k$ is additive noise (see Section \ref{sec:sa} for details). To cast $Q$-learning as a Markovian SA, let $F:\mathbb{R}^{|\mathcal{S}||\mathcal{A}|}\times \mathcal{S}\times\mathcal{A}\times\mathcal{S}\mapsto\mathbb{R}^{|\mathcal{S}||\mathcal{A}|}$ be an operator defined by $[F(Q,s_0,a_0,s_1)](s,a)
=\mathbbm{1}_{\{(s_0,a_0)=(s,a)\}}\Gamma_1(Q,s_0,a_0,s_1)+Q(s,a)$
for all $(s,a)$. Then the $Q$-learning algorithm (\ref{Q:motivation1}) can be rewritten as:
\begin{align}\label{Q:motivation}
	Q_{k+1}=Q_k+\alpha_k \left(F(Q_k,S_k,A_k,S_{k+1})-Q_k\right),
\end{align}
which is of the form of \eqref{eq:MSA} with $x_k$ replaced by $Q_k$, $w_k=0$, and $Y_k=(S_k,A_k,S_{k+1})$. The key takeaway is that in (\ref{Q:motivation}), the various noise terms (both due to performing asynchronous update and due to samples replacing an expectation in the Bellman equation) are encoded through introducing the operator $F(\cdot)$ and the associated evolution of the Markovian noise $\{Y_k\}$. 

\subsubsection{Analyzing the Markovian SA} To study the Markovian SA algorithm (\ref{Q:motivation}), let  $\bar{F}(\cdot)$ be the expectation of $F(\cdot,S_k,A_k,S_{k+1})$, where the expectation is taken with respect to the stationary distribution of the Markov chain $\{(S_k,A_k,S_{k+1})\}$. Under mild conditions, we show that $\bar{F}(Q)=N\mathcal{H}(Q)+(I-N)Q$. Here $\mathcal{H}(\cdot)$ is the Bellman's optimality operator for the $Q$-function \citep{bertsekas1996neuro}. The matrix $N$ is a diagonal matrix with $\{p(s,a)\}_{(s,a)\in\mathcal{S}\times\mathcal{A}}$ sitting on its diagonal, where $p(s,a)$ is the stationary visitation probability of the state-action pair $(s,a)$. 

An important insight about the operator $\bar{F}(\cdot)$ is that it can be viewed as an asynchronous variant of the Bellman operator $\mathcal{H}(\cdot)$. To see this, consider a state-action pair $(s,a)$. The value of $[\bar{F}(Q)](s,a)$ can be interpreted as the expectation of a random variable, which takes $[\mathcal{H}(Q)](s,a)$ w.p. $p(s,a)$, and takes $Q(s,a)$ w.p. $1-p(s,a)$. This precisely captures the asynchronous update in the $Q$-learning algorithm (\ref{Q:motivation1}) in that, at steady-state, $Q_k(s,a)$ is updated w.p. $p(s,a)$, and remains unchanged otherwise. Moreover, since it is well-known that $\mathcal{H}(\cdot)$ is a contraction mapping with respect to $\|\cdot\|_\infty$, we also show that $\bar{F}(\cdot)$ is a contraction mapping with respect to $\|\cdot\|_\infty$, with the optimal $Q$-function being its unique fixed-point.

Thus, by recentering the iteration in (\ref{Q:motivation}) about $\bar{F}(\cdot),$ we have:
\begin{align}
	Q_{k+1}=Q_k+\underbrace{\alpha_k \left(\bar{F}(Q_k)-Q_k\right)}_{\text{expected update}}+\underbrace{\alpha_k(F(Q_k,S_k,A_k,S_{k+1})-\bar{F}(Q_k))}_{\text{Markovian noise}}.\label{Q:motivation2}
\end{align}
In summary, we have recast asynchronous $Q$-learning as an iterative update that decomposes the $Q_k$ update into an expected update (averaged over the stationary distribution of the noise Markov chain) and a ``residual update" due to the Markovian noise. As will see in Section~\ref{sec:sa}, this update equation has the interpretation of solving the fixed-point equation $\bar{F}(Q)=Q$, with Markovian noise in the update.

\subsubsection{Finite-Sample Bounds for Markovian SA} We use a unified {\em Lyapunov approach} for deriving finite-sample bounds on the update in (\ref{Q:motivation2}). Specifically, the Lyapunov approach handles both (1) non-smooth $\|\cdot\|_\infty$-contraction of the averaged operator $\bar{F}(\cdot)$, and (2) Markovian noise that depends on the state-action trajectory. To handle $\|\cdot\|_\infty$-contraction, or more generally arbitrary norm contraction, inspired by \citep{chen2020finite}, we use the Generalized Moreau Envelope as the Lyapunov function. To handle the Markovian noise, we use the conditioning argument along with the geometric mixing of the underlying Markov chain \citep{bertsekas1996neuro,srikant2019finite}. Finally, for recursions beyond $Q$-learning, we deal with additional extraneous martingale difference noise through the tower property of the conditional expectation.

As we later discuss, beyond $Q$-learning, TD-learning variants such as off-policy V-trace, $n$-step TD, and TD$(\lambda)$ can all be modeled by Markovian SA algorithms involving a contraction mapping (possibly with respect to different norm), and Markovian noise. Therefore, our approach \textit{unifies} the finite-sample analysis of value-based RL algorithms.

\subsection{Related Literature}\label{subsec:literature}

 In this section, we discuss related literature on SA algorithms. We defer the discussion on related literature on finite-sample bounds of RL algorithms (such as $Q$-learning, V-trace, $n$-step TD, and TD$(\lambda)$) to the corresponding sections where we introduce these results. 

Stochastic approximation method was first introduced in \citep{robbins1951stochastic} for iteratively solving systems of equations. Since then, SA method is widely used in the context of optimization and machine learning. For example, in optimization, a special case of SA known as stochastic gradient descent (SGD) is a popular approach for iteratively finding the stationary points of some objective function \citep{bottou2018optimization}. In reinforcement learning, SA method is commonly used to solving the Bellman equation \citep{bertsekas1996neuro}, as will be studied in this paper.

Early literature on SA focuses on the asymptotic convergence \citep{benveniste2012adaptive,kushner2010stochastic,kushner2012stochastic}. A popular approach there is to view the SA algorithm as a stochastic and discrete counterpart of an ordinary differential equation (ODE), and show that SA algorithms converges asymptotically as long as the ODE is stable. See \citep{borkar2000ode,borkar2009stochastic} for more details about such ODE approach. Beyond convergence, the asymptotic convergence rate of SA algorithms are studied in \citep{devraj2018zap,chen2020accelerating}.

More recently, finite-sample convergence guarantees of SA algorithms have seen a lot of attention. For SA algorithms with i.i.d. or martingale difference noise, finite-sample analysis was performed in \citep{wainwright2019stochastic} under a cone-contraction assumption, and in \citep{chen2020finite} under a contraction assumption. The Lyapunov function we use in this paper is indeed inspired by \citep{chen2020finite}. However, \citep{chen2020finite} studies SA under martingale difference noise while we have both martingale and Markovian noise.

For SA algorithms with Markovian noise, finite-sample convergence bounds were established in \citep{bhandari2018finite,srikant2019finite} for linear SA. For nonlinear SA with Markovian noise, \citep{chen2019finitesample} established the convergence bounds under a strong monotone assumption. In this paper, the operator we work with is neither linear nor strongly monotone. 

In the context of optimization, convergence rates of SGD algorithms have been studied thoroughly in the literature. See \citep{lan2020first,bottou2018optimization} and the references therein for more details. In SGD algorithm, the update involves the gradient of some objective function, while in our setting, we do not have such gradient. Therefore, the SA algorithm we study in this paper is different from SGD (except when minimizing a smooth and strongly convex function, in which case there is contractive operator with respect to the Euclidean norm \citep{ryu2016primer}).

\section{Markovian Stochastic Approximation}\label{sec:sa}

In this section, we present finite-sample convergence bounds for a general stochastic approximation algorithm, which serves as a universal model for the RL algorithms we are going to study in Section \ref{sec:RL}.
\subsection{Problem Setting}\label{subsec:sa:setting}
Suppose we want to solve for $x^*\in\mathbb{R}^d$ in the equation
\begin{align}\label{eq:sa}
	\mathbb{E}_{Y\sim \mu}[F(x,Y)]=x,
\end{align}
where $Y \in\mathcal{Y}$ is a random variable with distribution $\mu$, and $F:\mathbb{R}^d\times \mathcal{Y}\mapsto\mathbb{R}^d$ 
is a general nonlinear operator. We assume the set $\mathcal{Y}$ is finite, and denote $\bar{F}(x)=\mathbb{E}_{Y\sim \mu}[F(x,Y)]$ as the expected operator. 

In the case where $\bar{F}(\cdot)$ is known, Eq. (\ref{eq:sa}) can be solved using the simple fixed-point iteration $x_{k+1}=\bar{F}(x_k)$, which is guaranteed to convergence when $\bar{F}(\cdot)$ is a contraction operator. When the distribution $\mu$ of the random variable $Y$ is unknown, and hence $\bar{F}(\cdot)$ is unknown, we consider solving Eq. (\ref{eq:sa}) using the stochastic approximation method described in the following. 

Let $\{Y_k\}$ be Markov chain with stationary distribution $\mu$. Then the SA algorithm iteratively updates the estimate $x_k$ by:
\begin{align}\label{algo:sa}
	x_{k+1}=x_k+\alpha_k\left(F(x_k,Y_k)-x_k+w_k\right),
\end{align}
where $\{\alpha_k\}$ is a sequence of stepsizes, and $\{w_k\}$ is a random process representing the additive extraneous noise. To establish finite-sample convergence bounds of Algorithm (\ref{algo:sa}), we next formally state our assumptions. Most of them are naturally satisfied in the RL algorithms we are going to study in Section \ref{sec:RL}. Let $\|\cdot\|_c$ be some arbitrary norm in $\mathbb{R}^d$. 
\begin{assumption}\label{as:F}
	There exist $A_1,B_1>0$ such that 
	\begin{enumerate}[(1)]
		\item $\|F(x_1,y)-F(x_2,y)\|_c\leq A_1\|x_1-x_2\|_c$ for any $x_1,x_2\in\mathbb{R}^d$ and $y\in\mathcal{Y}$.
		\item $\|F(\bm{0},y)\|_c\leq B_1$ for any $y\in\mathcal{Y}$.
	\end{enumerate}
\end{assumption}
\begin{assumption}\label{as:barF}
	The operator $\bar{F}(\cdot)$ is a contraction mapping with respect to $\|\cdot\|_c$, with contraction factor $\beta\in (0,1)$. That is, it holds for any $x_1,x_2\in\mathbb{R}^d$ that $\|\bar{F}(x_1)-\bar{F}(x_2)\|_c\leq \beta\|x_1-x_2\|_c$.
\end{assumption}
\begin{remark}
	By applying Banach fixed-point theorem \citep{banach1922operations}, Assumption \ref{as:barF} guarantees that the target equation (\ref{eq:sa}) has a unique solution, which we have denoted by $x^*$. 
\end{remark}
\begin{assumption}\label{as:MC}
	The Markov chain $\mathcal{M}=\{Y_k\}$ has a unique stationary distribution $\mu\in \Delta^{|\mathcal{Y}|}$, and there exist constants $C>0$ and $\sigma\in (0,1)$ such that $\max_{y\in\mathcal{Y}}\|P^k(y,\cdot)-\mu(\cdot)\|_{\text{TV}}\leq C\sigma^k$
	for all $k\geq 0$, where $\|\cdot\|_{\text{TV}}$ stands for the total variantion distance \citep{levin2017markov}.
\end{assumption}
\begin{remark}
	Since the state-space $\mathcal{Y}$ of the Markov chain $\{Y_k\}$ is finite, Assumption \ref{as:MC} is satisfied when $\{Y_k\}$ is irreducible and aperiodic \citep{levin2017markov}. 
\end{remark}

Under Assumption \ref{as:MC}, we next introduce the notion of Markov chain mixing, which will be frequently used in our derivation.

\begin{definition}\label{def:mixing_time}
	For any $\delta>0$, the mixing time $t_\delta(\mathcal{M})$ of the Markov chain $\mathcal{M}=\{Y_k\}$ with precision $\delta$ is defined by $t_\delta(\mathcal{M})=\min\{k\geq 0:\max_{y\in\mathcal{Y}}\|P^k(y,\cdot)-\mu(\cdot)\|_{\text{TV}}\leq \delta\}$.
\end{definition}
\begin{remark}
	Note that under Assumption \ref{as:MC}, we have $t_\delta\leq \frac{\log(C/\sigma)+\log(1/\alpha)}{\log(1/\sigma)}$ for any $\delta>0$, which implies $\lim_{\delta\rightarrow 0}\delta t_\delta=0$. This property is important in our analysis for controlling the Markovian noise $\{Y_k\}$.
\end{remark}
For simplicity of notation, in this section, we will just write $t_\delta$ for $t_\delta(\mathcal{M})$, and further use $t_k$ for $t_{\alpha_k}$, where $\alpha_k$ is the stepsize used in the $k$-th iteration of Algorithm (\ref{algo:sa}). To state our last assumption regarding the additive noise $\{w_k\}$. let $\mathcal{F}_k$ is the Sigma-algebra generated by $\{(x_i,Y_i,w_i)\}_{0\leq i\leq k-1}\cup \{x_k\}$.
\begin{assumption}\label{as:noise_w}
	The random process $\{w_k\}$ satisfies
	\begin{enumerate}[(1)]
		\item $\mathbb{E}[w_k|\mathcal{F}_k]=0$ for all $k\geq 0$.
		\item $\|w_k\|_c\leq A_2\|x_k\|_c+B_2$ for all $k\geq 0$, where $A_2,B_2>0$ are numerical constants.
	\end{enumerate} 
\end{assumption}
\begin{remark}
     Assumption \ref{as:noise_w} states that $\{w_k\}$ is a martingale difference sequence with respect to the filtration $\mathcal{F}_k$, and it can grow at most linear with respect to the iterate $x_k$.
\end{remark}

Finally, we specify the requirements for choosing the stepsize sequence $\{\alpha_k\}$. We will consider using stepsizes of the form $\alpha_k=\frac{\alpha}{(k+h)^\xi}$, where $\alpha,h>0$ and $\xi\in [0,1]$. The constants $\bar{\alpha}$ and $\bar{h}$ used in stating the following condition are specified in Appendix \ref{ap:pf:sa:stepsizes}. 
\begin{condition}\label{con:stepsize}
	\textit{(1) Constant Stepsize.} When $\xi=0$, there exists a threshold $\bar{\alpha}\in (0,1)$ such that the stepsize $\alpha$ is chosen to be in $(0,\bar{\alpha})$. \textit{(2) Linear Stepsize.} When $\xi=1$, for each $\alpha>0$, there exists a threshold $\bar{h}>0$  such that $h$ is chosen to be at least $\bar{h}$.  \textit{(3) Polynomial Stepsize.} For any $\xi\in (0,1)$ and $\alpha>0$, there exists a threshold $\bar{h}>0$ such that $h$ is chosen to be at least $\bar{h}$.
\end{condition}

\subsection{Finite-Sample Convergence Guarantees}\label{subsec:sa:bounds}
Under the assumptions stated above, we next present the finite-sample bounds of Algorithm (\ref{algo:sa}).

For simplicity of notation, let $A=A_1+A_2+1$, which can be viewed as the combined effective Lipschitz constant, and let $B=B_1+B_2$. Let $c_1=(\|x_0-x^*\|_c+\|x_0\|_c+B/A)^2$, and $c_2=(A\|x^*\|_c+B)^2$. 
The constants 
$\{\varphi_i\}_{1\leq i\leq 3}$ we are going to use
are defined explicitly in 
Appendix \ref{pf:thm:sa}, and depend only on the contraction norm $\|\cdot\|_c$ and the contraction factor $\gamma$. We will revisit them later in Lemma \ref{le:constants}. Define $K=\min\{k\geq 0:k\geq t_k\}$, which is well-defined under Assumption \ref{as:MC}. We now state the finite-sample convergence bounds of Algorithm (\ref{algo:sa}) in the following.
\begin{theorem}\label{thm:sa}
	Consider $\{x_k\}$ of Algorithm (\ref{algo:sa}). Suppose that Assumptions \ref{as:F}, \ref{as:barF}, \ref{as:MC} and \ref{as:noise_w} are satisfied. Then we have the following results.
	\begin{enumerate}[(1)]
		\item When $k\in [0,K-1]$, we have $\|x_k-x^*\|_c^2\leq c_1$ almost surely.
		\item When $k\geq K$, we have the following finite-sample convergence bounds.
		\begin{enumerate}[(a)]
			\item Under Condition \ref{con:stepsize} (1), we have:
			\begin{align*}
				\mathbb{E}[\|x_k-x^*\|_c^2]
				\leq  \varphi_1c_1(1-\varphi_2\alpha)^{k-t_\alpha}+\frac{\varphi_3c_2}{\varphi_2}\alpha t_\alpha.
			\end{align*}
			\item Under Condition \ref{con:stepsize} (2), we have:
			\begin{enumerate}[(i)]
				\item when $\alpha<1/\varphi_2$:
				\begin{align*}
					\mathbb{E}[\|x_k-x^*\|_c^2]\leq 
					\varphi_1c_1\left(\frac{K+h}{k+h}\right)^{\varphi_2\alpha}+\frac{8\alpha^2\varphi_3c_2}{1-\varphi_2\alpha}\frac{t_k}{(k+h)^{\varphi_2\alpha}}.
				\end{align*}
				\item when $\alpha=1/\varphi_2$:
				\begin{align*}
					\mathbb{E}[\|x_k-x^*\|_c^2]\leq 
					\varphi_1c_1\frac{K+h}{k+h}+8\alpha^2\varphi_3c_2\frac{t_k\log(k+h)}{k+h}.
				\end{align*}
				\item when $\alpha>1/\varphi_2$:
				\begin{align*}
					\mathbb{E}[\|x_k-x^*\|_c^2]\leq 
					\varphi_1c_1\left(\frac{K+h}{k+h}\right)^{\varphi_2\alpha}+\frac{8e\alpha^2\varphi_3c_2}{\varphi_2\alpha-1}\frac{t_k}{k+h}.
				\end{align*}
			\end{enumerate}
			\item Under Condition \ref{con:stepsize} (3), we have:
			\begin{align*}
				\mathbb{E}[\|x_k-x^*\|_c^2]\leq \varphi_1c_1 e^{-\frac{\varphi_2\alpha}{1-\xi}\left((k+h)^{1-\xi}-(K+h)^{1-\xi}\right)}+\frac{4\varphi_3c_2\alpha}{\varphi_2}\frac{t_k}{(k+h)^\xi}.
			\end{align*}
		\end{enumerate}
	\end{enumerate}
\end{theorem}
\begin{remark}
	Recall that $t_\delta\leq \frac{\log(C/\sigma)+\log(1/\alpha)}{\log(1/\sigma)}$ under Assumption \ref{as:MC}. Therefore, we have $t_k\leq \frac{\xi\log(k+h)+\log(C/(\alpha\sigma))}{\log(1/\sigma)}$, which introduces an additional logarithmic factor in the bound.
\end{remark}

In all cases of Theorem \ref{thm:sa}, we state the results as a combination of two terms. The first term is usually viewed as the ``bias", and it involves the error in the initial estimate $x_0$ (through the constant $c_1$), and the geometric decay term (for constant stepsize case). The second term is usually understood as the ``variance", and hence involves the constant $c_2$, which represents the noise variance at $x^*$. This form of convergence bounds is common in related literature. See for example in \citep{bottou2018optimization} about the results for the popular SGD algorithm.

From Theorem \ref{thm:sa}, we see that constant stepsize is very efficient in driving the bias the zero, but cannot eliminate the variance. When using linear stepsize, the convergence bounds crucially depend on the value of $\alpha$. In order to balance the bias and the variance terms to achieve the optimal convergence rate, we need to choose $\alpha>1/\varphi_2$, and the resulting optimal convergence rate is roughly $\mathcal{O}(\log(k)/k)$. When using polynomial stepsize, although the convergence rate is the sub-optimal $\mathcal{O}(\log(k)/k^\xi)$, it is more robust in the sense that it does not depend on $\alpha$. 

Switching focus, we now revisit the constants $\{\varphi_i\}_{1\leq i\leq 3}$ in Theorem \ref{thm:sa}, which as mentioned earlier, depend only on the contraction norm $\|\cdot\|_c$ and the contraction factor $\beta$. In the following lemma, we consider two cases where $\|\cdot\|_c=\|\cdot\|_2$ and $\|\cdot\|_c=\|\cdot\|_\infty$. Both of them will be useful when we study convergence bounds of RL algorithms. The proof of the following result is presented in Appendix \ref{pf:le:constants}.
\begin{lemma}\label{le:constants}
	The following bounds hold regarding the constants $\{\varphi_i\}_{1\leq i\leq 3}$.
	\begin{enumerate}[(1)]
		\item When $\|\cdot\|_c=\|\cdot\|_2$, we have $\varphi_1\leq 1$, $\varphi_2\geq 1-\beta$, and $\varphi_3\leq 228$.
		\item When $\|\cdot\|_c=\|\cdot\|_\infty$, we have $\varphi_1\leq 3$, $\varphi_2\geq \frac{1-\beta}{2}$, and $\varphi_3\leq \frac{456e\log(d)}{1-\beta}$.
	\end{enumerate}
\end{lemma}

Note that when compared to $\|\cdot\|_2$-contraction, where the constant $\varphi_3$ is bounded by a numerical constant, the upper bound for $\varphi_3$ has an additional $\frac{\log(d)}{1-\beta}$ factor under the $\|\cdot\|_\infty$-contraction. It was argued in \citep{chen2020finite} that in general such $\log(d)$ factor is unimprovable.

\subsection{Outline of the Proof}\label{subsec:sa:proof}
In this Section, we present the key ideas in proving Theorem \ref{thm:sa}. The detailed proof is presented in Appendix \ref{pf:thm:sa}. At a high level, we use a \textit{Lyapunov} approach. That is, we find a function $M:\mathbb{R}^d\mapsto\mathbb{R}$ such that the following one-step contractive inequality holds: 
\begin{align}\label{eq:one-step-contraction}
	\mathbb{E}[M(x_{k+1}-x^*)]\leq  (1-\mathcal{O}(\alpha_k)+o(\alpha_k))\mathbb{E}[M(x_k-x^*)]+o(\alpha_k),
\end{align}
which then can be repeatedly used to derive finite-sample bounds of the SA algorithm (\ref{algo:sa}).

\subsubsection{Generalized Moreau Envelope as a Lyapunov Function}\label{subsubsec:Lyapunov}
Inspired by \citep{chen2020finite}, we will use $M(x)=\min_{u\in\mathbb{R}^d}\{\frac{1}{2}\|u\|_c^2+\frac{1}{2\theta}\|x-u\|_p^2\}$ as the Lyapunov function, where $\theta>0$ and $p\geq 2$ are tunable parameters. The function $M(\cdot)$ is called the Generalized Moreau Envelope, which is known to be a smooth approximation of the function $\frac{1}{2}\|x\|_c^2$, with smoothness parameter $\frac{p-1}{\theta}$. See \citep{chen2020finite} for more details about using the Generalized Moreau Envelope as a Lyapunov function.

Using the smoothness property of $M(\cdot)$ and the update equation (\ref{algo:sa}), we have for all $k\geq 0$:
\begin{align}\label{eq:composition}
	&\mathbb{E}[M(x_{k+1}-x^*)]\nonumber\\
	\leq \;&\mathbb{E}[M(x_k-x^*)]+\underbrace{\alpha_k\mathbb{E}[\langle \nabla M(x_k-x^*),\bar{F}(x_k)-x_k\rangle]}_{T_1:\text{ Expected update}}+\underbrace{\alpha_k\mathbb{E}[\langle \nabla M(x_k-x^*),F(x_k,Y_k)-\bar{F}(x_k)\rangle]}_{T_2:\text{ Error due to Markovian noise } Y_k}\nonumber\\
	&+\underbrace{\alpha_k\mathbb{E}[ \langle\nabla M(x_k-x^*),w_k\rangle]}_{T_3:\text{ Error due to Martingale difference noise }w_k}+\underbrace{\frac{(p-1)\alpha_k^2}{2\theta}\mathbb{E}[\|F(x_k,Y_k)-x_k+w_k\|_p^2]}_{T_4:\text{ Error due to discretization and noises}}.
\end{align}
What remains to do is to bound the terms $T_1$ to $T_4$. The term $T_1$ represents the expected update. We show that it is negative and is of the order $\mathcal{O}(\alpha_k)$, hence giving the negative drift term in the target one-step contractive inequality (\ref{eq:one-step-contraction}). Using the assumption that $\{w_k\}$ is a martingale difference sequence and the tower property of conditional expectation, we show that the error term $T_3$ is indeed zero. Also, we show that the error term $T_4$ is of the size $\mathcal{O}(\alpha_k^2)=o(\alpha_k)$. The main challenge here is to control the error term $T_2$, which arises due to the Markovian noise $\{Y_k\}$.

\subsubsection{Handling the Markovian Noise}\label{subsubsec:Error}

To control the term $T_2$, we need to carefully use a conditioning argument along with the geometric mixing of $\{Y_k\}$. Specifically, we first show that the error is small when we replace $x_k$ by $x_{k-t_k}$ in the term $T_2$, where we recall that $t_k$ is the mixing time of the Markov chain $\{Y_k\}$ with precision $\alpha_k$. Now, consider the resulting term
\begin{align*}
	\tilde{T}_2=\alpha_k\mathbb{E}[\langle \nabla M(x_{k-t_k}-x^*),F(x_{k-t_k},Y_k)-\bar{F}(x_{k-t_k})\rangle].
\end{align*}
First taking expectation conditioning on $x_{k-t_k}$ and $Y_{k-t_k}$, then we have
\begin{align*}
	\tilde{T}_2=\alpha_k\mathbb{E}[\langle \nabla M(x_{k-t_k}-x^*),\underbrace{\mathbb{E}[F(x_{k-t_k},Y_k)\mid x_{k-t_k},Y_{k-t_k}]-\bar{F}(x_{k-t_k})}_{=o(1)\text{ by geometric mixing}}\rangle].
\end{align*}
Using the mixing time (cf. Definition \ref{def:mixing_time}) of $\{Y_k\}$, we see that the difference between $\mathbb{E}[F(x_{k-t_k},Y_k)\mid x_{k-t_k},Y_{k-t_k}]$ and $\bar{F}(x_{k-t_k})$ (which can written as $\mathbb{E}_{\mu}[F(x,Y)]$ evaluated at $x=x_{k-t_k}$) is of the size $o(1)$, hence concluding that $\tilde{T}_2=o(\alpha_k)$ by the tower property of conditional expectation.

This type of conditioning argument was first introduced in \citep{bertsekas1996neuro} [Section 4.4.1 The Case of Markov Noise] to establish the asymptotic convergence of linear SA with Markovian noise. Later, it was used more explicitly in \citep{srikant2019finite} to study finite-sample bounds of linear SA, and in \citep{chen2019finitesample} to study nonlinear SA under a strong monotone condition. In this paper, we study nonlinear SA under arbitrary norm contraction, which is fundamentally different from \citep{srikant2019finite,chen2019finitesample}.

Using the upper bounds we have for the terms $T_1$ to $T_4$ in Eq. (\ref{eq:composition}), we obtain the desired one-step contractive inequality (\ref{eq:one-step-contraction}). The rest of the proof follows by repeatedly using this inequality and evaluating the final expression for using different stepsize sequence $\{\alpha_k\}$.

In summary, we have stated finite-sample convergence bounds of a general stochastic approximation algorithm, and highlighted the key ideas in the proof. Next, we use Theorem \ref{thm:sa} as a universal tool to study the convergence bounds of reinforcement learning algorithms.

\section{Finite-Sample Guarantees of Reinforcement Learning Algorithms}\label{sec:RL}
We begin by introducing the underlying model for the RL problem. The RL problem is usually modeled by an MDP where the transition dynamics are unknown. In this work we consider an MDP consisting of a finite set of states $\mathcal{S}$, a finite set of actions $\mathcal{A}$, a set of unknown transition probability matrices that are indexed by actions $\{P_a\in\mathbb{R}^{|\mathcal{S}|\times|\mathcal{S}|}\mid a\in\mathcal{A}\}$, a reward function $\mathcal{R}:\mathcal{S}\times\mathcal{A}\mapsto \mathbb{R}$, and a discount factor $\gamma\in (0,1)$. We assume without loss of generality that the range of the reward function is $[0,1]$.

The goal in RL is to find an optimal policy $\pi^*$ so that the cumulative reward received by using $\pi^*$ is maximized. More formally, given a policy $\pi$, define its state-value function $V_\pi:\mathcal{S}\mapsto\mathbb{R}$ by 
\begin{align*}
	V_{\pi}(s)=\mathbb{E}_{\pi}\left[\sum_{k=0}^{\infty}\gamma^k\mathcal{R}(S_k,A_k)\;\middle|\; S_0=s\right]
\end{align*} for all $s$, where $\mathbb{E}_\pi[\;\cdot\;]$ means that the actions are selected according to the policy $\pi$. Then, a policy $\pi^*$ is said to be optimal if $V_{\pi^*}(s)\geq V_{\pi}(s)$ for any state $s$ and policy $\pi$. Under mild conditions, it was shown that such an optimal policy always exists \citep{puterman1995markov}.

In RL, the problem of finding an optimal policy is called the \textit{control} problem, which is solved with popular algorithms such as $Q$-learning \citep{watkins1992q}. A sub-problem is to find the value function of a given policy, which is called the \textit{prediction} problem. This is solved with TD-learning and its variants such as TD$(\lambda)$, $n$-step TD \citep{sutton2018reinforcement}, and the off-policy V-trace \citep{espeholt2018impala}. We next show that our SA results can be used to establish finite-sample convergence bounds of all the RL algorithms listed above, hence unifies the finite-sample analysis of value-based RL algorithms with asynchronous update.

\subsection{Off-Policy Control: Q-Learning}\label{subsec:Qlearning}

We first introduce the $Q$-learning algorithm proposed in \citep{watkins1992q}. Define the $Q$-function associated with a policy $\pi$ by 
\begin{align*}
	Q_\pi(s,a)=\mathbb{E}_\pi\left[\sum_{k=0}^{\infty}\gamma^k\mathcal{R}(S_k,A_k)\;\middle|\;S_0=s,A_0=a\right]
\end{align*}
for all $(s,a)$. Denote $Q^*$ as the $Q$-function associated with an optimal policy $\pi^*$.  (all optimal policies share the same optimal $Q$-function). The motivation of the $Q$-learning algorithm is based on the following result \citep{bertsekas1996neuro,sutton2018reinforcement}: 
\begin{center}
	$\pi^*$ is an optimal policy $\Leftrightarrow$ $\pi^*(a|s)\in \arg\max_{a\in\mathcal{A}}Q^*(s,a)$ for any $(s,a)$.
\end{center}
The above result implies that knowing the optimal $Q$-function alone is enough to compute an optimal policy. 

The $Q$-learning algorithm is an iterative method to estimate the optimal $Q$-function. First, a sample trajectory $\{(S_k,A_k)\}$ is collected using a suitable behavior policy $\pi_b$. Then, initialize $Q_0\in\mathbb{R}^{|\mathcal{S}||\mathcal{A}|}$. For each $k\geq 0$ and state-action pair $(s,a)$, the iterate $Q_k(s,a)$ is updated by
\begin{align}\label{eq:Q-learning}
	Q_{k+1}(s,a)=
	Q_k(s,a)+\alpha_k \Gamma_1(Q_k,S_k,A_k,S_{k+1})
\end{align}
when $(s,a)=(S_k,A_k)$, and $Q_{k+1}(s,a)=Q_k(s,a)$ otherwise. Here $\Gamma_1(Q_k,S_k,A_k,S_{k+1})=\mathcal{R}(S_k,A_k)+\gamma\max_{a'\in\mathcal{A}}Q_k(S_{k+1},a')- Q_k(S_k,A_k)$ is the temporal difference. To establish the finite-sample bounds of the $Q$-learning algorithm, we make the following assumption.
\begin{assumption}\label{as:Q}
	The behavior policy $\pi_b$ satisfies $\pi_b(a|s)>0$ for all $(s,a)$, and the Markov chain $\mathcal{M}_S=\{S_k\}$ induced by $\pi_b$ is irreducible and aperiodic. 
\end{assumption}

The requirement that $\pi_b(a|s)>0$ for all $(s,a)$ is necessary even for the asymptotic convergence of $Q$-learning \citep{tsitsiklis1994asynchronous}. The irreducibility and aperiodicity assumption is also standard in related work \citep{tsitsiklis1997analysis,tsitsiklis1999average}. Since we work with finite-state MDPs, Assumption \ref{as:Q} on $\mathcal{M}_S$ implies that $\mathcal{M}_S$ has a unique stationary distribution, denoted by $\kappa_b\in\Delta^{|\mathcal{S}|}$, and $\mathcal{M}_S$ mixes at a geometric rate \citep{levin2017markov}. 

\subsubsection{Properties of the Q-Learning Algorithm}
To derive finite-sample guarantees of the $Q$-learning algorithm,
we will follow the road map described in Section \ref{subsec:motivation}.
We begin by formally remodeling the $Q$-learning algorithm. Let $Y_k=(S_k,A_k,S_{k+1})$ for all $k\geq 0$. Note that the random process $\mathcal{M}_Y=\{Y_k\}$ is also a Markov chain, whose state-space is denoted by $\mathcal{Y}$, and is finite. Define an operator $F:\mathbb{R}^{|\mathcal{S}||\mathcal{A}|}\times \mathcal{Y}\mapsto\mathbb{R}^{|\mathcal{S}||\mathcal{A}|}$ by 
\begin{align*}
	[F(Q,y)](s,a)=[F(Q,s_0,a_0,s_1)](s,a)
	=\mathbbm{1}_{\{(s_0,a_0)=(s,a)\}}\Gamma_1(Q,s_0,a_0,s_1)+Q(s,a)
\end{align*}
for all $(s,a)$. Then $Q$-learning algorithm (\ref{eq:Q-learning}) can be written by
\begin{align*}
	Q_{k+1}=Q_k+\alpha_k \left(F(Q_k,Y_k)-Q_k\right),
\end{align*}
which is in the same form of the SA algorithm (\ref{algo:sa}) with $w_k$ being identically equal to zero. Next, we establish the properties of the operator $F(\cdot,\cdot)$ and the Markov chain $\{Y_k\}$ in the following proposition, which guarantees that Assumptions \ref{as:F} -- \ref{as:MC} are satisfied in the context of $Q$-learning. 

Let $N\in \mathbb{R}^{|\mathcal{S}||\mathcal{A}|\times |\mathcal{S}||\mathcal{A}|}$ be the diagonal matrix with $\{\kappa_b(s)\pi_b(a|s)\}_{(s,a)\in\mathcal{S}\times\mathcal{A}}$ sitting on its diagonal. Let $N_{\min}=\min_{(s,a)}\kappa_b(s)\pi_b(a|s)$, which is positive under Assumption \ref{as:Q}. The proof of the following proposition is presented in Appendix \ref{pf:prop:Q}. 
\begin{proposition}\label{prop:Q-learning}
	Suppose that Assumption \ref{as:Q} is satisfied, Then we have the following results. 
	\begin{enumerate}[(1)]
		\item The operator $F(\cdot,\cdot)$ satisfies $\|F(Q_1,y)-F(Q_2,y)\|_\infty\leq 2\|Q_1-Q_2\|_\infty$ and $\|F(\bm{0},y)\|_\infty\leq 1$ for any $Q_1,Q_2\in\mathbb{R}^{|\mathcal{S}||\mathcal{A}|}$, and $y\in \mathcal{Y}$.
		\item The Markov chain $\mathcal{M}_Y=\{Y_k\}$ has a unique stationary distribution $\mu$, and there exist $C_1>0$ and $\sigma_1\in (0,1)$ such that $\max_{y\in\mathcal{Y}}\|P^{k+1}(y,\cdot)-\mu(\cdot)\|_{\text{TV}}\leq C_1\sigma_1^k$ for any $k\geq 0$.
		\item Define the expected operator $\bar{F}:\mathbb{R}^{|\mathcal{S}||\mathcal{A}|}\mapsto\mathbb{R}^{|\mathcal{S}||\mathcal{A}|}$ of $F(\cdot,\cdot)$ by $\bar{F}(Q)=\mathbb{E}_{Y\sim \mu}[F(Q,Y)]$. Then
		\begin{enumerate}[(a)]
			\item $\bar{F}(\cdot)$ is explicitly given by $\bar{F}(Q)=N \mathcal{H}(Q)+(I-N)Q$, where $\mathcal{H}(\cdot)$ is the Bellman operator for the $Q$-function.
			\item $\bar{F}(\cdot)$ is a contraction mapping with respect to $\|\cdot\|_\infty$, with contraction factor $\beta_1:=1-N_{\min}(1-\gamma)$.
			\item $\bar{F}(\cdot)$ has a unique fixed-point $Q^*$.
		\end{enumerate}
	\end{enumerate}
\end{proposition}
As we see, the $(s,a)$-th entry of $\bar{F}(Q)$ is given by 
\begin{align*}
	\kappa_b(s)\pi_b(a|s)[\mathcal{H}(Q)](s,a)+(1-\kappa_b(s)\pi_b(a|s))Q(s,a),
\end{align*}
which captures the nature of performing asynchronous update as illustrated in Section \ref{subsec:motivation}. We shall refer to $\bar{F}(\cdot)$ as the asynchronous Bellman operator in the following.

\subsubsection{Finite-Sample Bounds of Q-Learning}
Proposition \ref{prop:Q-learning} enables us to apply Theorem \ref{thm:sa} and Lemma \ref{le:constants} (2) to the $Q$-learning algorithm. For ease of exposition, we only present the result of using constant stepsize, whose proof and the result for using diminishing stepsizes are presented in Appendix \ref{pf:thm:Q}.
\begin{theorem}\label{thm:Q}
	Consider $\{Q_k\}$ of Algorithm (\ref{eq:Q-learning}). Suppose that Assumption \ref{as:Q} is satisfied, and $\alpha_k=\alpha$ for all $k\geq 0$, where $\alpha$ is chosen such that $\alpha t_\alpha(\mathcal{M}_Y)\leq c_{Q,0} \frac{(1-\beta_1)^2}{\log(|\mathcal{S}||\mathcal{A}|)}$ ($c_{Q,0}$ is a numerical constant). Then we have for all $k\geq t_\alpha(\mathcal{M}_Y)$:
	\begin{align*}
		\mathbb{E}[\|Q_k-Q^*\|_\infty^2]\leq c_{Q,1}\left(1-\frac{(1-\beta_1)\alpha}{2}\right)^{k-t_\alpha(\mathcal{M}_Y)}+c_{Q,2}\frac{\log(|\mathcal{S}||\mathcal{A}|)}{(1-\beta_1)^2}\alpha t_\alpha(\mathcal{M}_Y),
	\end{align*}
	where $c_{Q,1}=3(\|Q_0-Q^*\|_\infty+\|Q_0\|_\infty+1)^2$ and $c_{Q,2}=912e(3\|Q^*\|_\infty+1)^2$.
\end{theorem}
\begin{remark}
	Recall that $t_\alpha(\mathcal{M}_Y)$ is the mixing time of the Markov chain $\{Y_k\}$ with precision $\alpha$. Using Proposition \ref{prop:Q-learning} (2), we see that $t_\alpha(\mathcal{M}_Y)$ produces an additional $\log(1/\alpha)$ factor in the bound.
\end{remark}

Similar to Theorem \ref{thm:sa}, we view the first term on the RHS of the convergence bound as the the bias, and the second term as the variance. Since we are using constant stepsize, the bias term goes to zero geometrically fast while the variance is of the size $\mathcal{O}(\alpha\log(1/\alpha))$. 

Based on Theorem \ref{thm:Q}, we next derive the sample complexity of $Q$-learning. The proof of the following result is presented in Appendix \ref{pf:co:Q}.
\begin{corollary}\label{co:Q}
	In order to make $\mathbb{E}[\|Q_k-Q^*\|_\infty]\leq \epsilon$, where $\epsilon>0$ is a given accuracy, the total number of samples required is of the size
	\begin{align*}
		\underbrace{\mathcal{O}\left(\frac{\log^2(1/\epsilon)}{\epsilon^2}\right)}_{\text{Accuracy}}\underbrace{\tilde{\mathcal{O}}\left(\frac{1}{(1-\gamma)^5}\right)}_{\text{Effective horizon}}\underbrace{\tilde{\mathcal{O}}(N_{\min}^{-3})}_{\text{Quality of exploration}}.
	\end{align*}
\end{corollary}
\begin{remark}
	In the $\tilde{\mathcal{O}}(\cdot)$ notation, we ignore all the polylogarithmic terms. Moreover, we upper bound $\|Q^*\|_\infty$ by $1/(1-\gamma)$ in deriving the sample complexity result.
\end{remark}

From Corollary \ref{co:Q}, we see that the dependence on the accuracy $\epsilon$ is $\mathcal{O}(\epsilon^{-2}\log^2(1/\epsilon))$, and the dependence on the effective horizon is $\tilde{\mathcal{O}}((1-\gamma)^{-5})$. These two results match with known results in the literature \citep{beck2013improved}. The parameter $N_{\min}$ is defined to be $\min_{s,a}\kappa_b(s)\pi_b(a|s)$, hence captures the quality of exploration of the behavior policy $\pi_b$. Since $N_{\min}\geq 1/|\mathcal{S}||\mathcal{A}|$, we see that the best possible dependence on the size of the state-action space is $\tilde{\mathcal{O}}(|\mathcal{S}|^3|\mathcal{A}|^3)$.

\subsubsection{Related Literature on Q-learning} \label{subsubsec:Qliterature}

The $Q$-learning algorithm \citep{watkins1992q} is perhaps one of the most well-known algorithms in the RL literature. The asymptotic convergence of $Q$-learning was established in \cite{tsitsiklis1994asynchronous,jaakkola1994convergence,borkar2000ode}, and the asymptotic convergence rate in \cite{szepesvari1997asymptotic,devraj2017zap}. Beyond asymptotic behavior, finite-sample analysis of $Q$-learning was also thoroughly studied in the literature \citep{even2003learning,li2020sample,beck2013improved,qu2020finite,jin2018q}. The state-of-the-art sample complexity for asynchronous $Q$-learning goes to \cite{li2020sample}\footnote{While the results in \cite{li2020sample} are stated in terms of high probability bounds, due to the boundedness of $Q$-learning, their concentration bounds can be translated into a mean-square bound with the same sample complexity. See \cite{li2023q} for a proof.}, 
which has a better dependence on the size of the state-action space compared to this work. In addition to being a contractive SA, $Q$-learning has many other properties, such as the update equation being asynchronous, the iterates being uniformly bounded by a constant \citep{gosavi2006boundedness}, which are used in \citep{li2020sample} for their analysis. While our SA framework did not exploit these properties of $Q$-learning (which results in a sub-optimal sample complexity), it is a more general framework that enables us to study a wide variety of algorithms beyond $Q$-learning. A typical example is the V-trace algorithm studied in the previous section. Due to off-policy sampling, the iterates of V-trace do not admit a uniform upper bound.

\subsection{Off-Policy Prediction: V-Trace}\label{subsec:Vtrace}
We next switch our focus to solving the prediction problem using TD-learning variants. Specifically, we first consider the V-trace algorithm for off-policy TD-learning \citep{espeholt2018impala}. Let $\pi_b$ be a behavior policy used to collect samples, $\pi$ be the target policy (i.e., we want to evaluate $V_\pi$), and $n$ be a positive integer. Let  
\begin{align*}
    c(s,a)=\min\left(\bar{c},\frac{\pi(a|s)}{\pi_b(a|s)}\right),\quad \text{and}\quad \rho(s,a)=\min\left(\bar{\rho},\frac{\pi(a|s)}{\pi_b(a|s)}\right)
\end{align*}
be the truncated importance sampling ratios at $(s,a)$, where $\bar{\rho}\geq \bar{c}\geq 1$ are the two truncation levels. Suppose a sequence of state-action pairs $\{(S_k,A_k)\}$ is collected under the behavior policy $\pi_b$. Then, with initialization $V_0\in\mathbb{R}^{|\mathcal{S}|}$, for each $k\geq 0$ and $s\in\mathcal{S}$, the V-trace algorithm updates the estimate $V_k(s)$ by
\begin{align}\label{algo:V-trace}
	V_{k+1}(s)=V_k(s)+\alpha_k\sum_{i=k}^{k+n-1}\gamma^{i-k}\left(\prod_{j=k}^{i-1}c(S_j,A_j)\right)\rho(S_i,A_i)\Gamma_2(V_k,S_i,A_i,S_{i+1})
\end{align}
when $s=S_k$, and $V_{k+1}(s)=V_k(s)$ otherwise. Here  $\Gamma_2(V_k,S_i,A_i,S_{i+1})=\mathcal{R}(S_i,A_i)+\gamma V_k(S_{i+1})-V_k(S_i)$ is the temporal difference. Note that when $\pi_b=\pi$, and $\bar{c}=\bar{\rho}=1$, Eq. (\ref{algo:V-trace}) reduces to the update equation for the on-policy $n$-step TD \citep{sutton2018reinforcement}. To establish finite-sample convergence bounds of Algorithm (\ref{algo:V-trace}), we make the following assumption.
\begin{assumption}\label{as:Vtrace}
	The behavior policy $\pi_b$ satisfies for all $s\in\mathcal{S}$: $\{a\in\mathcal{A}\mid \pi(a|s)>0\}\subseteq \{a\in\mathcal{A}\mid \pi_b(a|s)>0\}$, and the Markov chain $\mathcal{M}_\mathcal{S}=\{S_k\}$ induced by $\pi_b$ is irreducible and aperiodic.
\end{assumption}

The first part of Assumption \ref{as:Vtrace} is call the \textit{coverage} assumption, which states that, for any state, if it is possible to explore a specific action under the target policy $\pi$, then it is also possible to explore such an action under the behavior policy $\pi_b$. This requirement is necessary for off-policy RL. The second part of Assumption \ref{as:Vtrace} implies that $\{S_k\}$ has a unique stationary distribution, denoted by $\kappa_b\in\Delta^{|\mathcal{S}|}$. Moreover, the Markov chain $\{S_k\}$ mixes at a geometric rate \citep{levin2017markov}. 

\subsubsection{Properties of the V-Trace Algorithm}
To establish the convergence bounds of the V-trace algorithm, similar to $Q$-learning, we first model the V-trace algorithm in the form of SA algorithm (\ref{algo:sa}). For any $k\geq 0$, let $Y_k=(S_k,A_k,...,S_{k+n-1},A_{k+n-1},S_{k+n})$. It is clear that $\{Y_k\}$ is also a Markov chain, whose state space is denoted by $\mathcal{Y}$. Define an operator $F:\mathbb{R}^{|\mathcal{S}|}\times\mathcal{Y}\mapsto\mathbb{R}^{|\mathcal{S}|}$ by
\begin{align*}
	[F(V,y)](s)=\mathbbm{1}_{\{s_0=s\}}\sum_{i=0}^{n-1}\gamma^{i}\left(\prod_{j=0}^{i-1}c(s_j,a_j)\right)\rho(s_i,a_i)\Gamma_2(V,s_i,a_i,s_{i+1})+V(s)
\end{align*}
for all $s\in\mathcal{S}$. Then the V-trace update equation (\ref{algo:V-trace}) can be equivalently written by
\begin{align}\label{algo:Vtrace_new}
	V_{k+1}=V_k+\alpha_k(F(V_k,Y_k)-V_k).
\end{align}
Under Assumptions \ref{as:Vtrace}, we next establish the properties of the operator $F(\cdot)$ and the Markov chain $\{Y_k\}$, which allow us to call for our main results in Section \ref{sec:sa}. Before that, we need to introduce more notation in the following.

\begin{notation}
For any policy $\pi$, let $P_\pi\in\mathbb{R}^{|\mathcal{S}|\times|\mathcal{S}|}$ be the transition probability matrix under policy $\pi$, i.e., $P_\pi(s,s')=\sum_{a\in\mathcal{A}}\pi(a|s)P_a(s,s')$. Also, we let $R_\pi\in\mathbb{R}^{|\mathcal{S}|}$ be such that $R_\pi(s)=\sum_{a\in\mathcal{A}}\pi(a|s)\mathcal{R}(s,a)$. Let $C,D\in\mathbb{R}^{|\mathcal{S}|\times|\mathcal{S}|}$ be diagonal matrices such that 
\begin{align*}
	C(s)=\sum_{a\in\mathcal{A}}\min(\bar{c}\pi_b(a|s),\pi(a|s)), \quad \text{and} \quad D(s)=\sum_{a\in\mathcal{A}}\min(\bar{\rho}\pi_b(a|s),\pi(a|s)),\;\forall\; s\in\mathcal{S}.
\end{align*}
Let $C_{\min}=\min_{s\in\mathcal{S}}C(s)$ and $D_{\min}=\min_{s\in\mathcal{S}}D(s)$. Note that we have $0<C_{\min}\leq D_{\min}\leq 1$ under Assumption \ref{as:Vtrace}. Let $\mathcal{K}\in\mathbb{R}^{|\mathcal{S}|\times|\mathcal{S}|}$ be a diagonal matrix with diagonal entries $\{\kappa_b(s)\}_{s\in\mathcal{S}}$, and let $\mathcal{K}_{\min}=\min_{s\in\mathcal{S}}\kappa_b(s)$. Define two policies $\pi_{\bar{c}}$ and $\pi_{\bar{\rho}}$ by
\begin{align*}
	\pi_{\bar{c}}(a|s)=\frac{\min(\bar{c}\pi_b(a|s),\pi(a|s))}{C(s)},\quad \text{and}\quad \pi_{\bar{\rho}}(a|s)=\frac{\min(\bar{\rho}\pi_b(a|s),\pi(a|s))}{D(s)},\quad \forall\;(s,a).
\end{align*}
\end{notation}

\begin{proposition}\label{prop:Vtrace}
	Under Assumptions \ref{as:Vtrace}, the V-trace algorithm (\ref{algo:Vtrace_new}) has the following properties:
	\begin{enumerate}[(1)]
		\item The operator $F(\cdot)$ satisfies:
		\begin{enumerate}[(a)]
			\item $\|F(V_1,y)-F(V_2,y)\|_\infty\leq (2\bar{\rho}+1) \eta(\gamma,\bar{c})
			\|V_1-V_2\|_\infty$ for all $V_1,V_2\in\mathbb{R}^{|\mathcal{S}|}$ and $y\in\mathcal{Y}$, where $\eta(\gamma,\bar{c})=\frac{1-(\gamma\bar{c})^n}{1-\gamma\bar{c}}$ when $\gamma\bar{c}\neq 1$, and $\eta(\gamma,\bar{c})=n$ when $\gamma\bar{c}=1$.
			\item $\|F(\bm{0},y)\|_\infty\leq\bar{\rho} \eta(\gamma,\bar{c})$ for all $y\in\mathcal{Y}$.
		\end{enumerate}
		\item The Markov chain $\{Y_k\}$ has a unique stationary distribution, denoted by $\mu$. Moreover, there exists $C_2>0$ and $\sigma_2\in (0,1)$ such that $\max_{y\in\mathcal{Y}}\|P^{k+n}(y,\cdot)-\mu(\cdot)\|_{\text{TV}}\leq C_2\sigma_2^k$ for all $k\geq 0$.
		\item Define the expected operator $\bar{F}:\mathbb{R}^{|\mathcal{S}|}\mapsto\mathbb{R}^{|\mathcal{S}|}$ of $F(\cdot)$ by $\bar{F}(V)=\mathbb{E}_{Y\sim \mu}[F(V,Y)]$ for all $V\in\mathbb{R}^{|\mathcal{S}|}$. Then
		\begin{enumerate}
			\item $\bar{F}(\cdot)$ is explicitly given by 
			\begin{align*}
				\bar{F}(V)=\left[I-\mathcal{K}\sum_{i=0}^{n-1}(\gamma CP_{\pi_{\bar{c}}})^iD(I-\gamma P_{\pi_{\bar{\rho}}})\right]V+\mathcal{K}\sum_{i=0}^{n-1}(\gamma CP_{\pi_{\bar{c}}})^iDR_{\pi_{\bar{\rho}}}.
			\end{align*}
			\item $\bar{F}(\cdot)$ is a contraction mapping with respect to $\|\cdot\|_\infty$, with contraction factor 
			\begin{align*}
				\beta_2:=1-\mathcal{K}_{\min}\frac{(1-\gamma)(1-(\gamma C_{\min})^n)D_{\min}}{1-\gamma C_{\min}}.
			\end{align*}
			\item $\bar{F}(\cdot)$ has a unique fixed-point $V_{\pi_{\bar{\rho}}}$, which is the value function of the policy $\pi_{\bar{\rho}}$.
		\end{enumerate} 
	\end{enumerate}
\end{proposition}
The proof of Proposition \ref{prop:Vtrace} is presented in Appendix \ref{ap:V-trace}. Observe from Proposition \ref{prop:Vtrace} (3) that the asynchronous Bellman operator $\bar{F}(\cdot)$ associated with the V-trace algorithm is a $\beta_2$-contraction with respect to $\|\cdot\|_\infty$. A similar contraction property for synchronous V-trace was shown in \citep{espeholt2018impala} and \citep{chen2020finite}, but with a different contraction factor.

\subsubsection{Finite-Sample Bounds of V-Trace}
We here present the convergence bounds of V-trace for using constant stepsize, whose proof and the result for using diminishing stepsize are presented in Appendix \ref{ap:V-trace}. 
\begin{theorem}\label{thm:Vtrace}
	Consider $\{V_k\}$ of Algorithm (\ref{algo:V-trace}). Suppose that Assumption \ref{as:Vtrace} is satisfied, and $\alpha_k=\alpha$ for all $k\geq 0$, where $\alpha$ is chosen such that $\alpha (t_\alpha(\mathcal{M}_S)+n)\leq c_{V,0} \frac{(1-\beta_2)^2}{(\bar{\rho}+1)^2\eta^2(\gamma,\bar{c})\log(|\mathcal{S}|)}$ ($c_{V,0}$ is a numerical constant). Then we have for all $k\geq t_\alpha(\mathcal{M}_S)+n$:
	\begin{align*}
		\mathbb{E}[\|V_k-V_{\pi_{\bar{\rho}}}\|_\infty^2]\leq c_{V,1}\left(1-\frac{1-\beta_2}{2}\alpha\right)^{k-(t_\alpha(\mathcal{M}_S)+n)}+ c_{V,2}\frac{\log(|\mathcal{S}|) (\bar{\rho}+1)^2\eta(\gamma,\bar{c})^2}{(1-\beta_2)^2} \alpha (t_\alpha(\mathcal{M}_S)+n),
	\end{align*}
	where $c_{V,1}=3(\|V_0-V_{\pi_{\bar{\rho}}}\|_\infty+\|V_0\|_\infty+1)^2$, and $c_{V,2}=3648e(\|V_{\pi_{\bar{\rho}}}\|_\infty+1)^2$.
\end{theorem}
\begin{remark}
	Similarly as in $Q$-learning, under Assumption \ref{as:Vtrace}, we can further bound the mixing time $t_\alpha(\mathcal{M}_S)$ by $L(\log(1/\alpha)+1)$, where $L>0$ is a constant that solely depends on the underlying Markov chain $\{(S_k,A_k)\}$.
\end{remark}

The rate of convergence (geometric convergence with accuracy $\mathcal{O}(\alpha\log(1/\alpha))$) is similar to that of $Q$-learning. The truncation level $\bar{\rho}$ determines the limit point $V_{\pi_{\bar{\rho}}}$. The truncation level $\bar{c}$ mainly controls the variance term. These observations agree with results in \citep{chen2020finite}, where synchronous V-trace is studied. To formally characterize how the parameters of V-trace impact the convergence rate, we next derive the sample complexity bound.

When $\bar{\rho}= 1/\min_{s,a}\pi_b(a|s)\geq  \max_{s,a}\pi(a|s)/\pi_b(a|s)$, the bias due to introducing the truncation level $\bar{\rho}$ is eliminated and hence we have $V_{\pi_{\bar{\rho}}}=V_\pi$ and also $D_{\min}=1$. In this case, based on Theorem \ref{thm:Vtrace}, we have the following sample complexity bound, whose proof is identical to that of Corollary \ref{co:Q} and is omitted.
\begin{corollary}\label{co:sc:Vtrace}
	When $\bar{\rho}=1/\min_{s,a}\pi_b(a|s)$, in order to make $\mathbb{E}[\|V_k-V_\pi\|_\infty]\leq \epsilon$, the number of samples required for the V-trace algorithm (\ref{algo:V-trace}) is of the size
	\begin{align*}
		\underbrace{\mathcal{O}\left(\frac{\log^2(1/\epsilon)}{\epsilon^2}\right)}_{\text{Accuracy}}\underbrace{\tilde{\mathcal{O}}\left(\frac{1}{(1-\gamma)^5}\right)}_{\text{Effective horizon}}\underbrace{\tilde{\mathcal{O}}\left(\frac{n\bar{\rho}^2\eta(\gamma,\bar{c})^2(1-\gamma C_{\min})^3}{(1-(\gamma C_{\min})^n)^3}\right)}_{\text{Off-policy $n$-step TD}}\underbrace{\tilde{\mathcal{O}}\left(\mathcal{K}_{\min}^{-3}\right)}_{\text{Quality of exploration}}.
	\end{align*}
\end{corollary}

We use the upper bound $1/(1-\gamma)$ for $\|V_{\pi_{\bar{\rho}}}\|_\infty$ when deriving the sample complexity result. From Corollary \ref{co:sc:Vtrace}, we see that the dependence on the accuracy $\epsilon$, the effective horizon $1/(1-\gamma)$, and the parameter $\mathcal{K}_{\min}$ that captures the quality of exploration are the same as $Q$-learning. 

Another term that arises in the sample complexity of V-trace is $\tilde{\mathcal{O}}\left(\frac{n\bar{\rho}^2\eta(\gamma,\bar{c})^2(1-\gamma C_{\min})^3}{D_{\min}^3(1-(\gamma C_{\min})^n)^3}\right)$, which is a consequence of performing $n$-step off-policy TD with truncated importance sampling ratios. The impact of the parameter $n$ will be analyzed in detail in Section \ref{subsec:n-step TD}, where we study on-policy $n$-step TD and the efficiency of bootstrapping. We here focus on the two truncation levels $\bar{c}$ and $\bar{\rho}$. Note that $\bar{\rho}=1/\min_{s,a}\pi_b(a|s)\geq 1/|\mathcal{A}|$, hence the side effect of ensuring $V_{\pi_{\bar{\rho}}}=V_\pi$ by choosing large enough $\bar{\rho}$ is to introduce at least a factor of $|\mathcal{A}|^{-2}$ in the sample complexity. This can also be viewed as a measure of the quality of exploration through the behavior policy. Therefore, the total dependence on the size of the state-action space is at least $\tilde{\mathcal{O}}(|\mathcal{S}|^3|\mathcal{A}|^2)$. We want to point out that this lowest possible value may not be achievable since  $\mathcal{K}_{\min}=1/|\mathcal{S}|$ and $\min_{s,a}\pi_b(a|s)= 1/|\mathcal{A}|$ may not hold simultaneously.

The dependence of the sample complexity on the truncation level $\bar{c}$ is through the term $\eta(\gamma,\bar{c})$. In view of the expression of the function $\eta(\gamma,\bar{c})$ given in Proposition \ref{prop:Vtrace} (1), we see that to avoid an exponential factor of $n$ we need to aggressively truncate the importance sampling ratios by choosing $\bar{c}< 1/\gamma$. 

\subsubsection{Related Literature on V-trace}
The V-trace algorithm was first proposed in \citep{espeholt2018impala} as an off-policy variant of $n$-step TD-learning. The key novelty in V-trace is that the two truncation levels $\bar{c}$ and $\bar{\rho}$ are introduced in the importance sampling ratios to separately control the bias and the variance. The asymptotic convergence of V-trace in the case where $n=\infty$ was established in \citep{espeholt2018impala}. As for finite-sample guarantees, \citep{chen2020finite} studies $n$-step V-trace with synchronous update. The main difference between the sample complexity of the asynchronous V-trace studied in this paper and the synchronous V-trace studied in \citep{chen2020finite} is that there is an additional factor of $\mathcal{K}_{\min}^{-3}$ in our bound (cf. Corollary \ref{co:sc:Vtrace}), which captures the quality of exploration and is the key feature of asynchronous RL algorithms. Other algorithms that are closely related to V-trace are the off-Policy $Q^\pi(\lambda)$ \citep{harutyunyan2016q},  Tree-backup TB($\lambda$) \citep{precup2000eligibility}, Retrace$(\lambda)$ \citep{munos2016safe}, and $Q$-trace \citep{khodadadian2021finite}.

\subsection{On-Policy Prediction: $n$-Step TD}\label{subsec:n-step TD}
In this section, we study the convergence bounds of the on-policy $n$-step TD-learning algorithm, which can be viewed as a special case of the V-trace algorithm with $\pi_b=\pi$ and $\bar{c}=\bar{\rho}=1$. Therefore, one can directly apply Theorem \ref{thm:Vtrace} to this setting and obtain finite-sample bounds for $n$-step TD. However, we will show that due to on-policy sampling there are better properties (i.e., $\|\cdot\|_2$-contraction) of the $n$-step TD algorithm we can exploit, which enables us to obtain tighter bounds. Observe that in the case of on-policy $n$-step TD, the update equation (\ref{algo:V-trace}) simplifies to:
\begin{align}\label{algo:TDn}
	V_{k+1}(s)=V_k(s)+\alpha_k\Gamma_3(V_k,S_k,A_k,...,S_{k+n})
\end{align}
when $s=S_k$, and $V_{k+1}(s)=V_k(s)$ otherwise, where 
\begin{align*}
	\Gamma_3(V_k,S_k,A_k,...,S_{k+n})=\sum_{i=0}^{n-1}\gamma^{i}\mathcal{R}(S_{k+i},A_{k+i})+\gamma^n V_k(S_{k+n})-V_k(S_k)
\end{align*} 
is the $n$-step temporal difference. 

An important idea in the $n$-step TD is to use the parameter $n$ to adjust the bootstrapping effect. When $n=0$, Eq. (\ref{algo:TDn}) is the standard TD$(0)$ update, which corresponds to extreme bootstrapping. When $n=\infty$, Eq. (\ref{algo:TDn}) is the Monte Carlo method for estimating $V_\pi$, which corresponds to no bootstrapping. A long-standing question in RL is about the efficiency of bootstrapping, i.e., the choice of $n$ that leads to the optimal performance of the algorithm \citep{sutton2018reinforcement}.

In the following sections, we will establish finite-sample convergence bounds of the $n$-step TD-learning algorithm. By evaluating the resulting sample complexity bound as a function of $n$, we provide theoretical insight into the bias-variance trade-off in terms of $n$, as well as an estimate of the optimal value of $n$. Similarly as in the previous sections, we make the following assumption.
\begin{assumption}\label{as:TDn}
	The Markov chain $\mathcal{M}_\mathcal{S}=\{S_k\}$ induced by the target policy $\pi$ is irreducible and aperiodic.
\end{assumption}

Since we are using on-policy sampling in $n$-step TD, the target policy must be explorative. Assumption \ref{as:TDn} ensures this property, and also implies that $\{S_k\}$ has a unique stationary distribution (denoted by $\kappa\in\Delta^{|\mathcal{S}|}$), and the geometric mixing property \citep{levin2017markov}. 

\subsubsection{Properties of the $n$-Step TD-Learning Algorithm}
To apply Theorem \ref{thm:sa}, we begin by rewriting the update equation (\ref{algo:TDn}) in the form of the SA algorithm studied in Section \ref{sec:sa}. Let a sequence $\{Y_k\}$ be defined by $Y_k=(S_k,A_k,...,S_{k+n-1},A_{k+n-1},S_{k+n})$ for all $k\geq 0$. It is clear that $\{Y_k\}$ is a Markov chain, whose state-space is denoted by $\mathcal{Y}$ and is finite. Define an operator $F:\mathbb{R}^{|\mathcal{S}|}\times\mathcal{Y}\mapsto\mathbb{R}^{|\mathcal{S}|}$ by
\begin{align*}
	[F(V,y)](s)=[F(V,s_0,a_0,...,s_n)](s)=\mathbbm{1}_{\{s_0=s\}}\Gamma_3(V,s_0,a_0,...,s_n)+V(s),\quad \forall\;s\in\mathcal{S}.
\end{align*}
Then the $n$-step TD algorithm (\ref{algo:TDn}) can be equivalently written by 
\begin{align*}
	V_{k+1}=V_k+\alpha_k(F(V_k,Y_k)-V_k).
\end{align*}
We next establish the properties of the $n$-step TD algorithm in the following proposition, whose proof is presented in Appendix \ref{ap:TDn}. Let $\mathcal{K}\in\mathbb{R}^{|\mathcal{S}|\times|\mathcal{S}|}$ be a diagonal matrix with diagonal entries $\{\kappa(s)\}_{s\in\mathcal{S}}$, and let $\mathcal{K}_{\min}=\min_{s\in\mathcal{S}}\kappa(s)$.
\begin{proposition}\label{prop:TDn}
	Under Assumption \ref{as:TDn}, the $n$-step TD-learning algorithm (\ref{algo:TDn}) has the following properties.
	\begin{enumerate}[(1)]
		\item The operator $F(\cdot)$ satisfies for all $V_1,V_2\in\mathbb{R}^{|\mathcal{S}|}$ and $y\in\mathcal{Y}$:
		\begin{enumerate}[(a)]
			\item $\|F(V_1,y)-F(V_2,y)\|_2\leq 2 \|V_1-V_2\|_2$.
			\item $\|F(\bm{0},y)\|_2\leq\frac{1}{1-\gamma}$.
		\end{enumerate}
		\item The Markov chain $\{Y_k\}$ has a unique stationary distribution, denoted by $\mu$. Moreover, there exists $C_3>0$ and $\sigma_3\in (0,1)$ such that $\max_{y\in\mathcal{Y}}\|P^{k+n}(y,\cdot)-\mu(\cdot)\|_{\text{TV}}\leq C_3\sigma_3^k$ for all $k\geq 0$.
		\item Define the expected operator $\bar{F}:\mathbb{R}^{|\mathcal{S}|}\mapsto\mathbb{R}^{|\mathcal{S}|}$ of $F(\cdot)$ by $\bar{F}(V)=\mathbb{E}_{Y\sim \mu}[F(V,Y)]$ for all $V\in\mathbb{R}^{|\mathcal{S}|}$. Then
		\begin{enumerate}[(a)]
			\item $\bar{F}(\cdot)$ is explicitly given by 
			\begin{align*}
				\bar{F}(V)=\left[I-\mathcal{K}\sum_{i=0}^{n-1}(\gamma P_{\pi})^i(I-\gamma P_{\pi})\right]V+\mathcal{K}\sum_{i=0}^{n-1}(\gamma P_{\pi})^iR_{\pi}.
			\end{align*}
			\item $\bar{F}(\cdot)$ is a contraction mapping with respect to the $\ell_p$-norm $\|\cdot\|_p$ for any $p\in [1,\infty]$, with a common contraction factor 
			\begin{align*}
				\beta_3:=1-\mathcal{K}_{\min}(1-\gamma ^n).
			\end{align*}
			\item $\bar{F}(\cdot)$ has a unique fixed-point $V_{\pi}$.
		\end{enumerate} 
	\end{enumerate}
\end{proposition}

From Proposition \ref{prop:TDn}, we see that the asynchronous Bellman operator $\bar{F}(\cdot)$ associated with the on-policy $n$-step TD-learning algorithm is a $\beta_3$-contraction with respect to $\|\cdot\|_p$ for any $p\in [1,\infty]$, which is a major difference compared to its off-policy variant V-trace. In particular, this implies that $\bar{F}(\cdot)$ is a contraction with respect to the standard Euclidean norm $\|\cdot\|_2$. This is the property we are going to exploit in establishing finite-sample bounds of $n$-step TD in the next section. 

To intuitively understand the $\|\cdot\|_2$-contraction property, recall a ``less known" property from \citep{tsitsiklis1997analysis,bertsekas1996neuro} that the $n$-step Bellman operator $\mathcal{T}_\pi^n(\cdot)$ is a contraction operator  with respect to the weighted $\ell_2$-norm $\|\cdot\|_\kappa$, with weights being the stationary distribution $\kappa$. Similar to $Q$-learning, the asynchronous Bellman operator $\bar{F}(\cdot)$ is a convex combination of the identity operator $I$ and the $n$-step Bellman operator $\mathcal{T}_\pi^n(\cdot)$, using the stationary distribution $\kappa$ as weights. Therefore, due to this ``normalization", the asynchronous Bellman operator is a contraction mapping with respect to the unweighted $\ell_2$-norm.

\subsubsection{Finite-Sample Bounds of $n$-Step TD}
In this section, we use the $\|\cdot\|_2$-contraction property from Proposition \ref{prop:TDn} to derive finite-sample convergence bounds of Algorithm (\ref{algo:TDn}). Note that Lemma \ref{le:constants} (1) is applicable in this case. The proof of the following result is presented in Appendix \ref{ap:TDn}.

\begin{theorem}\label{thm:TDn}
	Consider $\{V_k\}$ of Algorithm (\ref{algo:TDn}). Suppose that Assumption \ref{as:TDn} is satisfied, and $\alpha_k\equiv \alpha$ with $\alpha$ chosen such that $\alpha (t_\alpha(M_S)+n)\leq \hat{c}_0(1-\beta_3)$ ($\hat{c}_0$ is a numerical constant). Then we have for all $k\geq t_\alpha(\mathcal{M}_S)+n$: 
	\begin{align*}
		\mathbb{E}[\|V_k-V_\pi\|_2^2]\leq \hat{c}_1\left(1-(1-\beta_3)\alpha\right)^{k-(t_\alpha(\mathcal{M}_S)+n)}+\hat{c}_2\frac{\alpha(t_\alpha(\mathcal{M}_S)+n)}{(1-\gamma)^2(1-\beta_3)},
	\end{align*}
	where $\hat{c}_1=(\|V_0-V_\pi\|_2+\|V_0\|_2+4)^2$ and $\hat{c}_2=228(4(1-\gamma)\|V_\pi\|_2+1)^2$.
\end{theorem}

To analyze the impact of the parameter $n$, we begin by rewriting the convergence bounds in Theorem \ref{thm:TDn} focusing only on $n$-dependent terms. Using the explicit expression of the contraction factor $\beta_3$, in the $k$-th iteration, the bias term is of the size $(1-\Theta(1-\gamma^n))^k$. Since the mixing time $t_\alpha(\mathcal{M}_S)$ of the original Markov chain $\{S_k\}$ does not depend on $n$, the variance term is of the size $\mathcal{O}(n/(1-\gamma^n))$. Now we can clearly see that as $n$ increases to infinity, the bias goes down while the variance goes up, thereby demonstrating a bias-variance trade-off in the $n$-step TD-learning algorithm.

To formally characterize how the parameters of the $n$-step TD algorithm impact its convergence rate and computing an estimate of the optimal choice of $n$, we next derive the sample complexity of $n$-step TD based on Theorem \ref{thm:TDn}. The proof of the following result is identical to that of Corollary \ref{co:Q} and is omitted.

\begin{corollary}
	In order to make $\mathbb{E}][\|V_k-V_\pi\|_2]\leq \epsilon$, the number of samples required for the $n$-step TD-learning algorithm (\ref{algo:TDn}) is of the size
	\begin{align*}
		\underbrace{\mathcal{O}\left(\frac{\log^2(1/\epsilon)}{\epsilon^2}\right)}_{\text{Accuracy}}\underbrace{\tilde{\mathcal{O}}\left(\frac{1}{(1-\gamma)^2}\right)}_{\text{Effective horizon}}\underbrace{\tilde{\mathcal{O}}\left(\frac{n}{(1-\gamma^n)^2}\right)}_{\text{Parameter $n$}}\underbrace{\tilde{\mathcal{O}}(\mathcal{K}_{\min}^{-2})}_{\text{Quality of exploration}}\tilde{\mathcal{O}}(|\mathcal{S}|^{1/2})
	\end{align*}
\end{corollary}

Note that we use $\|V_\pi\|_2\leq |\mathcal{S}|^{1/2}/(1-\gamma)$ in deriving the sample complexity. Although the norms we used in the mean square distance are different for $n$-step TD and V-trace, since $\|x\|_\infty\leq \|x\|_2$ for any $x$, we clearly see that on-policy $n$-step TD has a better sample complexity over off-policy V-trace. First of all, it enjoys a better dependency on the effective horizon (set $n=1$ to see such dependence), which is $\tilde{\mathcal{O}}((1-\gamma)^{-4})$. In addition, since $\mathcal{K}_{\min}\leq 1/|\mathcal{S}|$, the dependency on $\mathcal{K}_{\min}$ is at most $\mathcal{K}_{\min}^{-2.5}$ for $n$-step TD while V-trace has $\mathcal{K}_{\min}^{-3}$ (cf. Corollary \ref{co:sc:Vtrace}). The main reason for such an improvement in sample complexity is that we are able to exploit the $\|\cdot\|_2$-contraction of the corresponding asynchronous Bellman operator $\bar{F}(\cdot)$ in $n$-step TD. 

In light of the dependence on the parameter $n$, $\tilde{\mathcal{O}}(n(1-\gamma^n)^{-2})$, the optimal choice of $n$ can be estimated by minimizing the function $n(1-\gamma^n)^{-2}$ over all positive integers. By doing that, we obtain the following estimate:
\begin{align*}
	n_{\text{optimal}}\sim \min\left(1,\lfloor1/\log(1/\gamma)\rceil\right),
\end{align*}
where $\lfloor x\rceil$ stands for the integer closest to $x$.
This result implies that when the discount factor $\gamma$ is small (specifically $\gamma\leq 1/e$), there is not much improvement in using multi-step TD-learning over using single step TD-learning, and when the discount factor is large, using $n$-step TD-learning with $n\sim \lfloor1/\log(1/\gamma)\rceil$ has provable improvement. 

\subsubsection{Related Literature on $n$-Step TD}

The notion of using multi-step returns instead of only one-step return was introduced in \citep{watkins1989learning}. See \citep{sutton2018reinforcement} [Chapter 7] for more details about $n$-step TD. The asymptotic convergence of $n$-step TD can be established using the general stochastic approximation algorithm under contraction assumption \citep{bertsekas1996neuro}. Regarding the choice of $n$, it was observed in empirical experiments that $n$-step TD (with a suitable choice of $n$) usually outperforms TD$(0)$ and Monte Carlo method \citep{singh1996reinforcement,sutton2018reinforcement}. However, theoretical understanding to this phenomenon is not well established in the literature. We derive finite-sample convergence bounds of the $n$-step TD-learning algorithm as an explicit function of $n$. This requires us to compute the exact expression of the contraction factor $\beta_3$ of the asynchronous Bellman operator (Proposition \ref{prop:TDn} (3)), and the mixing time (Proposition  \ref{prop:TDn} (2)).

\subsection{On-Policy Prediction: TD$(\lambda)$}\label{subsec:TDlambda}

We next consider the on-policy TD$(\lambda)$ algorithm, which effectively uses a convex combination of all the multi-step temporal differences at each update. We begin by describing the TD$(\lambda)$ algorithm for estimating the value function $V_\pi$ of a policy $\pi$. Suppose that we have collected a sample trajectory $\{(S_k,A_k)\}$ using the policy $\pi$. Then, with initialization $V_0\in\mathbb{R}^{|\mathcal{S}|}$, for any $\lambda\in (0,1)$, the estimate $V_k$ is iteratively updated according to
\begin{align}\label{algo:TDlambda}
	V_{k+1}(s)=V_k(s)+\alpha_k z_k(s)\Gamma_4(V_k,S_k,A_k,S_{k+1})
\end{align}
for all $s\in\mathcal{S}$, where $\Gamma_4(V_k,S_k,A_k,S_{k+1})=\mathcal{R}(S_k,A_k)+\gamma V_k(S_{k+1})-V_k(S_k)$ is the temporal difference, and $z_k(s)=\sum_{i=0}^{k}(\gamma \lambda)^{k-i}\mathbbm{1}_{\{S_i=s\}}$ is the eligibility trace \citep{bertsekas1996neuro,sutton2018reinforcement}.

A key idea in the TD$(\lambda)$ algorithm is to use the parameter $\lambda$ to adjust the bootstrapping effect. When $\lambda=0$, Algorithm (\ref{algo:TDlambda}) becomes the standard TD$(0)$ update, which is pure bootstrapping. Another extreme case is when $\lambda=1$. This corresponds to using pure Monte Carlo method. Theoretical understanding of the efficiency of bootstrapping is a long-standing open problem in RL \citep{sutton1999open}. 

In the following Section, we establish finite-sample convergence bounds of the TD$(\lambda)$ algorithm. By evaluating the resulting bound as a function of $\lambda$, we provide theoretical insight into the bias-variance trade-off in choosing $\lambda$. Similar to $n$-step TD, we make the following assumption. 

\begin{assumption}\label{as:TDlambda}
	The Markov chain $\mathcal{M}_\mathcal{S}=\{S_k\}$ induced by the target policy $\pi$ is irreducible and aperiodic.
\end{assumption}

As a result of Assumption \ref{as:TDlambda}, the Markov chain $\{S_k\}$ has a unique stationary distribution, denoted by $\kappa \in\Delta^{|\mathcal{S}|}$, and the geometric mixing property \citep{levin2017markov}. 

\subsubsection{Properties of the TD$(\lambda)$ Algorithm}
Unlike the previous algorithms we studied, the TD$(\lambda)$ algorithm
cannot be viewed as a direct variant of the SA algorithm (\ref{algo:sa}). This is because of the geometric averaging induced by the eligibility trace in TD($\lambda$), which creates dependencies over the \textit{entire} past trajectory. We overcome this difficulty by using an additional truncation argument, and separately handle the residual error due to truncation. For ease of exposition, we consider only using constant stepsize in the TD$(\lambda)$ algorithm, i.e., $\alpha_k=\alpha$ for all $k\geq 0$. 

For any $k\geq 0$, let $Y_k=(S_0,...,S_k,A_k,S_{k+1})$ (which takes value in $\mathcal{Y}_k:=\mathcal{S}^{k+2}\times\mathcal{A}$), and define a time-varying operator $F_k:\mathbb{R}^{|\mathcal{S}|}\times\mathcal{Y}_k \mapsto\mathbb{R}^{|\mathcal{S}|}$ by 
\begin{align*}
	[F_k(V,y)](s)=[F_k(V,s_0,...,s_k,a_k,s_{k+1})](s)
	=\Gamma_4(V,s_k,a_k,s_{k+1})\sum_{i=0}^{k}(\gamma \lambda)^{k-i}\mathbbm{1}_{\{s_{i}=s\}}+ V(s)
\end{align*}
for all $s\in\mathcal{S}$. Note that the sequence $\{Y_k\}$ is \textit{not} a Markov chain since it has a time-varying state-space. Using the notations of $\{Y_k\}$ and $F_k(\cdot,\cdot)$, we can rewrite the update equation of the TD$(\lambda)$ algorithm by
\begin{align}\label{algo:TDlambda_update}
	V_{k+1}=V_k+\alpha \left(F_k(V_k,Y_k)-V_k\right).
\end{align}
Although Eq. (\ref{algo:TDlambda_update}) is similar to the update equation for SA algorithm (\ref{algo:sa}), since the sequence $\{Y_k\}$ is not a Markov chain and the operator $F_k(\cdot,\cdot)$ is time-varying, our Theorem \ref{thm:sa} is not directly applicable. 

To overcome this difficulty, let us carefully look at the operator $F_k(\cdot,\cdot)$. Although $F_k(V_k,Y_k)$ depends on the whole trajectory of states visited before (through the term $\sum_{i=0}^{k}(\gamma\lambda)^{k-i}\mathbbm{1}_{\{S_i=s\}}$), due to the geometric factor $(\gamma\lambda)^{k-i}$, the states visited during the early stage of the iteration are not important. Inspired by this observation, we define the truncated sequence $\{Y_k^\tau\}$ of $\{Y_k\}$ by $Y_k^\tau=(S_{k-\tau},...,S_k,A_k,S_{k+1})$ for all $k\geq \tau$, where $\tau$ is a \textit{fixed} non-negative integer. Note that the random process $\mathcal{M}_Y=\{Y_k^\tau\}$ is now a Markov chain, whose state-space is denoted by $\mathcal{Y}_\tau$ and is finite. Similarly, we define the truncated operator $F_k^\tau:\mathbb{R}^{|\mathcal{S}|}\times\mathcal{Y}_\tau\mapsto\mathbb{R}^{|\mathcal{S}|}$ of $F_k(\cdot,\cdot)$ by 
\begin{align*}
	[F_k^\tau(V,s_{k-\tau},...,s_k,a_k,s_{k+1})](s)
	=\Gamma_4(V,s_k,a_k,s_{k+1})\sum_{i=k-\tau}^{k}(\gamma \lambda)^{k-i}\mathbbm{1}_{\{s_{i}=s\}}+ V(s)
\end{align*}
for all $s\in\mathcal{S}$. Using the above notation, we can further rewrite the update equation (\ref{algo:TDlambda_update}) by
\begin{align}\label{algo:TDlambda_new}
	V_{k+1}=\;&V_k+\alpha \left(F_k^\tau(V_k,Y_k^\tau)-V_k\right)+\underbrace{\alpha \left(F_k(V_k,Y_k)-F_k^\tau(V_k,Y_k^\tau)\right)}_{\text{The Error Term}}.
\end{align}
Now, we argue that when the truncation level $\tau$ is large enough, the last term on the RHS of the previous equation is negligible compared to the other two terms. In fact, we have the following result. See Appendix \ref{ap:TDlambda} for its proof.
\begin{lemma}\label{le:truncation}
	For all $k\geq 0$ and $\tau\in [0,k]$, denote $y=(s_0,...,s_k,a_k,s_{k+1})$ and $y_\tau=(s_{k-\tau},...,s_k,a_k,s_{k+1})$. Then the following inequality holds for all $V\in\mathbb{R}^{|\mathcal{S}|}$: $\|F_k^\tau(V,y_\tau)-F_k(V,y)\|_2\leq \frac{(\gamma\lambda)^{\tau+1}}{1-\gamma\lambda}(1+2\|V\|_2)$.
\end{lemma}
Lemma \ref{le:truncation} indicates that the error term in Eq. (\ref{algo:TDlambda_new}) is indeed  geometrically small. Suppose we ignore that error term. Then the update equation becomes $V_{k+1}\approx V_k+\alpha_k(F_k^\tau(V_k,Y_k^\tau)-V_k)$. Since the random process $\mathcal{M}_{Y}=\{Y_k^\tau\}$ is a Markov chain, once we establish the required properties for the truncated operator $F_k^\tau(\cdot,\cdot)$, our SA results become applicable. 

From now on, we will choose $\tau=\min\{k\geq 0:(\gamma\lambda)^{k+1}\leq \alpha\}\leq \frac{\log(1/\alpha)}{\log(1/(\gamma\lambda))}$, where $\alpha$ is the constant stepsize we use. This implies that the error term in Eq. (\ref{algo:TDlambda_new}) is of the size $\mathcal{O}(\alpha^2)$. Under this choice of $\tau$, we next investigate the properties of the operator $F_k^\tau(\cdot,\cdot)$ and the random process $\{Y_k^\tau\}$ in the following Proposition (See Appendix \ref{ap:TDlambda} for its proof). Let $\mathcal{K}\in\mathbb{R}^{|\mathcal{S}|\times|\mathcal{S}|}$ be a diagonal matrix with diagonal entries $\{\kappa(s)\}_{s\in\mathcal{S}}$, and let $\mathcal{K}_{\min}=\min_{s\in\mathcal{S}}\kappa(s)$.
\begin{proposition}\label{prop:TDlambda}
	Suppose Assumption \ref{as:TDn} is satisfied. Then we have the following results.
	\begin{enumerate}[(1)]
		\item For any $k\geq \tau$, the operator $F_k^\tau(\cdot,\cdot)$ satisfies  $\|F_k^\tau(V_1,y)-F_k^\tau(V_2,y)\|_2\leq \frac{3}{1-\gamma\lambda}\|V_1-V_2\|_2$, and $\|F_k^\tau(\bm{0},y)\|_2\leq \frac{1}{1-\gamma\lambda}$ for any $V_1,V_2\in\mathbb{R}^{|\mathcal{S}|}$ and $y\in\mathcal{Y}_\tau$.
		\item The Markov chain $\{Y_k^\tau\}_{k\geq \tau}$ has a unique stationary distribution, denoted by $\mu$. Moreover, there exists $C_4>0$ and $\sigma_4\in (0,1)$ such that $\max_{y\in\mathcal{Y}_\tau}\|P^{k+\tau+1}(y,\cdot)-\mu(\cdot)\|_{\text{TV}}\leq C_4\sigma_4^k$ for all $k\geq 0$.
		\item For any $k\geq \tau$, define the expected operator $\bar{F}_k^\tau:\mathbb{R}^{|\mathcal{S}|}\mapsto\mathbb{R}^{|\mathcal{S}|}$  by $\bar{F}_k^\tau(V)=\mathbb{E}_{Y\sim \mu}[F_k^\tau(V,Y)]$. Then
		\begin{enumerate}[(a)]
			\item $\bar{F}_k^\tau(\cdot)$ is explicitly given by 
			\begin{align*}
			    \bar{F}_k^\tau(V)=\left(I-\mathcal{K}\sum_{i=0}^{\tau}(\gamma \lambda P_	\pi)^{i}(I-\gamma P_{\pi})\right) V+\mathcal{K}\sum_{i=0}^{\tau}(\gamma 	\lambda P_{\pi})^{i}R_{\pi}.
			\end{align*}
			\item $\bar{F}_k^\tau(\cdot)$ is a contraction mapping with respect to $\|\cdot\|_p$ for any $p\in [1,\infty]$, with a common contraction factor 
			\begin{align*}
			    \beta_4=1-\mathcal{K}_{\min}\frac{(1-\gamma)(1-(\gamma\lambda)^{\tau+1})}{1-\gamma\lambda}.
			\end{align*}
			\item $\bar{F}_k^\tau(\cdot)$ has a unique fixed-point $V_{\pi}$.
		\end{enumerate}
	\end{enumerate}
\end{proposition}

Similar to $n$-step TD, the truncated asynchronous Bellman operator $\bar{F}(\cdot)$ associated with the TD$(\lambda)$ algorithm is a contraction with respect to the $\ell_p$-norm $\|\cdot\|_p$ for any $1\leq p\leq \infty$, with a common contraction factor $\beta_4$. This enables us to use our SA results along with Lemma \ref{le:constants} (1). 

\subsubsection{Finite-Sample Bounds of TD$(\lambda)$}\label{subsubsec:TDlambda_bounds}
We now present the finite-sample convergence bound of the TD$(\lambda)$ algorithm for using constant stepsize, where we exploit only the $\|\cdot\|_2$-contraction property from Proposition \ref{prop:TDlambda}. The proof is presented in Appendix \ref{pf:thm:TDlambda}.

\begin{theorem}\label{thm:TDlambda}
	Consider $\{V_k\}$ of Algorithm (\ref{algo:TDlambda}). Suppose that Assumption \ref{as:TDlambda} is satisfied and $\alpha_k\equiv \alpha$ with $\alpha$ chosen such that $\alpha (t_\alpha(\mathcal{M}_S)+2\tau+1)\leq \tilde{c}_{0}(1-\beta_4)(1-\gamma\lambda)^2$ ($\tilde{c}_0$ is a numerical constant). Then the following inequality holds for all $k\geq t_\alpha(\mathcal{M}_S)+2\tau+1$:
	\begin{align*}
		\mathbb{E}[\|V_k-V_{\pi}\|_2^2]\leq \tilde{c}_1\left(1-(1-\beta_4)\alpha\right)^{k-(t_\alpha(\mathcal{M}_S)+2\tau+1)}+\tilde{c}_2\frac{\alpha \left(t_\alpha(\mathcal{M}_S)+\tau+1\right)}{(1-\gamma\lambda)^2(1-\beta_4)},
	\end{align*}
	where $\tilde{c}_1=(\|V_0-V_\pi\|_2+\|V_0\|_2+1)^2$ and $\tilde{c}_2=114(4\|V_\pi\|_2+1)^2$.
\end{theorem}
\begin{remark}\label{rm:4}
	Under Assumption \ref{as:TDlambda}, the mixing time $t_\alpha(\mathcal{M}_S)$ is at most an affine function of $\log(1/\alpha)$. More importantly, it does not depend on the parameter $\lambda$.
\end{remark}
The convergence rate of TD$(\lambda)$ is similar to that of $n$-step TD. We here focus on the impact of the parameter $\lambda$. We begin by rewriting both the bias term and the variance term in the resulting convergence bound of Theorem \ref{thm:TDlambda} focusing only on $\lambda$-dependent terms. Then, the bias term is of the size $(1-\Theta(1/(1-\gamma\lambda)))^k$
while the variance term is between $ \Theta(1/(1-\gamma\lambda)\log(1/(\gamma\lambda)))$ and $ \Theta(1/(1-\gamma\lambda))$. Now observe that the bias term is in favor of large $\lambda$ (i.e., less bootstrapping, more Monte Carlo) while the variance term is in favor of small $\lambda$ (i.e., more bootstrapping, less Monte Carlo). This observation agrees with empirical results in the literature \citep{sutton2018reinforcement,kearns2000bias}. Therefore, we demonstrate a bias-variance trade-off in choosing $\lambda$, which addresses one of the open problems in \citep{sutton1999open} on the efficiency of bootstrapping in RL. 

\subsubsection{Related Literature on TD$(\lambda)$}

The idea of using $\lambda$-return and eligibility traces was introduced and developed in \citep{watkins1989learning,jaakkola1994convergence}. See \citep{sutton2018reinforcement} [Chapter 12] for more details. The convergence of TD$(\lambda)$ was established in \citep{dayan1994td}. 

Regarding the parameter $\lambda$, empirical observations indicate that a properly chosen intermediate value of $\lambda$ usually outperforms both TD$(0)$ and TD$(1)$ \citep{singh1996reinforcement}. Theoretical justification of this observation is, to some extend, provided in \citep{kearns2000bias}, where they study a variant of the TD$(\lambda)$ algorithm called phased TD. The TD$(\lambda)$ algorithm is often used along with function approximation in practice. The asymptotic convergence of TD$(\lambda)$ with linear function approximation was established in \citep{tsitsiklis1997analysis}. More recently, \citep{bhandari2018finite,srikant2019finite} established the finite-sample bounds of TD$(\lambda)$ with linear function approximation by modeling the algorithm as a linear stochastic approximation with Markovian noise. The result of \citep{bhandari2018finite} indicates that TD$(\lambda)$ in general outperforms TD$(0)$. However, \citep{bhandari2018finite} does not provide explicit trade-offs between the convergence bias and variance in choosing $\lambda$. Similarly, \citep{srikant2019finite} does not have an explicit bound, and thus do not study bias-variance trade-off, which is what we did in this paper. To achieve that, we need to carefully characterize the contraction factor $\beta_4$ of the truncated Bellman operator $\bar{F}_k^\tau(\cdot)$, as well as the mixing time of the truncated Markov chain $\{Y_k^\tau\}$.

\section{Conclusion}\label{sec:conclusion}
In this work, we provide a unified framework for establishing finite-sample convergence bounds of value-based Reinforcement Learning algorithms. The key idea is to first remodel the RL algorithm as a Markovian SA associated with a contractive asynchronous Bellman operator, and then derive the convergence bounds of such SA algorithm using a Lyapunov-drift argument. Based on the universal result on Markovian SA, we derive finite-sample convergence guarantees of $Q$-learning for solving the control problem, and various TD-learning algorithms (e.g. off-policy V-trace, $n$-step TD, and TD$(\lambda)$) for solving the prediction problem, where we also provide theoretical insight into the long-standing question about the efficiency of bootstrapping in RL.

\bibliographystyle{apalike}
\bibliography{references}

\begin{thebibliography}{}

\bibitem[Banach, 1922]{banach1922operations}
Banach, S. (1922).
\newblock Sur les op{\'e}rations dans les ensembles abstraits et leur
  application aux {\'e}quations int{\'e}grales.
\newblock {\em Fund. math}, 3(1):133--181.

\bibitem[Beck, 2017]{beck2017first}
Beck, A. (2017).
\newblock {\em First-order methods in optimization}, volume~25.
\newblock SIAM.

\bibitem[Beck and Srikant, 2013]{beck2013improved}
Beck, C.~L. and Srikant, R. (2013).
\newblock Improved upper bounds on the expected error in constant step-size
  {$Q$}-learning.
\newblock In {\em 2013 American Control Conference}, pages 1926--1931. IEEE.

\bibitem[Benveniste et~al., 2012]{benveniste2012adaptive}
Benveniste, A., M{\'e}tivier, M., and Priouret, P. (2012).
\newblock {\em Adaptive algorithms and stochastic approximations}, volume~22.
\newblock Springer Science \& Business Media.

\bibitem[Bertsekas and Tsitsiklis, 1996]{bertsekas1996neuro}
Bertsekas, D.~P. and Tsitsiklis, J.~N. (1996).
\newblock {\em Neuro-dynamic programming}.
\newblock Athena Scientific.

\bibitem[Bhandari et~al., 2018]{bhandari2018finite}
Bhandari, J., Russo, D., and Singal, R. (2018).
\newblock {A Finite Time Analysis of Temporal Difference Learning With Linear
  Function Approximation}.
\newblock In {\em Conference On Learning Theory}, pages 1691--1692.

\bibitem[Borkar, 2009]{borkar2009stochastic}
Borkar, V.~S. (2009).
\newblock {\em Stochastic approximation: a dynamical systems viewpoint},
  volume~48.
\newblock Springer.

\bibitem[Borkar and Meyn, 2000]{borkar2000ode}
Borkar, V.~S. and Meyn, S.~P. (2000).
\newblock The {ODE} method for convergence of stochastic approximation and
  reinforcement learning.
\newblock {\em SIAM Journal on Control and Optimization}, 38(2):447--469.

\bibitem[Bottou et~al., 2018]{bottou2018optimization}
Bottou, L., Curtis, F.~E., and Nocedal, J. (2018).
\newblock Optimization methods for large-scale machine learning.
\newblock {\em Siam Review}, 60(2):223--311.

\bibitem[Chen et~al., 2020a]{chen2020accelerating}
Chen, S., Devraj, A., Bernstein, A., and Meyn, S. (2020a).
\newblock {Accelerating Optimization and Reinforcement Learning with
  Quasi-Stochastic Approximation}.
\newblock {\em Preprint arXiv:2009.14431}.

\bibitem[Chen et~al., 2020b]{chen2020finite}
Chen, Z., Maguluri, S.~T., Shakkottai, S., and Shanmugam, K. (2020b).
\newblock {Finite-Sample Analysis of Contractive Stochastic Approximation Using
  Smooth Convex Envelopes}.
\newblock {\em Advances in Neural Information Processing Systems}, 33.

\bibitem[Chen et~al., 2019]{chen2019finitesample}
Chen, Z., Zhang, S., Doan, T.~T., Clarke, J.-P., and Maguluri, S.~T. (2019).
\newblock {Finite-Sample Analysis of Nonlinear Stochastic Approximation with
  Applications in Reinforcement Learning}.
\newblock {\em Preprint arXiv:1905.11425}.

\bibitem[Dann et~al., 2019]{dann2019policy}
Dann, C., Li, L., Wei, W., and Brunskill, E. (2019).
\newblock {Policy certificates: Towards accountable reinforcement learning}.
\newblock In {\em International Conference on Machine Learning}, pages
  1507--1516. PMLR.

\bibitem[Dayan and Sejnowski, 1994]{dayan1994td}
Dayan, P. and Sejnowski, T.~J. (1994).
\newblock {TD($\lambda$) converges with probability $1$}.
\newblock {\em Machine Learning}, 14(3):295--301.

\bibitem[Devraj et~al., 2018]{devraj2018zap}
Devraj, A.~M., Bu{\v{s}}ic, A., and Meyn, S. (2018).
\newblock {Zap meets momentum: Stochastic approximation algorithms with optimal
  convergence rate}.
\newblock {\em Preprint arXiv:1809.06277}.

\bibitem[Devraj and Meyn, 2017]{devraj2017zap}
Devraj, A.~M. and Meyn, S. (2017).
\newblock Zap {$Q$}-learning.
\newblock In {\em Advances in Neural Information Processing Systems}, pages
  2235--2244.

\bibitem[Espeholt et~al., 2018]{espeholt2018impala}
Espeholt, L., Soyer, H., Munos, R., Simonyan, K., Mnih, V., Ward, T., Doron,
  Y., Firoiu, V., Harley, T., Dunning, I., et~al. (2018).
\newblock {IMPALA}: {Scalable Distributed Deep-RL with Importance Weighted
  Actor-Learner Architectures}.
\newblock In {\em International Conference on Machine Learning}, pages
  1407--1416.

\bibitem[Even-Dar and Mansour, 2003]{even2003learning}
Even-Dar, E. and Mansour, Y. (2003).
\newblock Learning rates for {$Q$}-learning.
\newblock {\em Journal of Machine Learning Research}, 5(Dec):1--25.

\bibitem[Gosavi, 2006]{gosavi2006boundedness}
Gosavi, A. (2006).
\newblock Boundedness of iterates in {$Q$}-learning.
\newblock {\em Systems \& control letters}, 55(4):347--349.

\bibitem[Harutyunyan et~al., 2016]{harutyunyan2016q}
Harutyunyan, A., Bellemare, M.~G., Stepleton, T., and Munos, R. (2016).
\newblock {Q($\lambda$) with Off-Policy Corrections}.
\newblock In {\em International Conference on Algorithmic Learning Theory},
  pages 305--320. Springer.

\bibitem[Jaakkola et~al., 1994]{jaakkola1994convergence}
Jaakkola, T., Jordan, M.~I., and Singh, S.~P. (1994).
\newblock Convergence of stochastic iterative dynamic programming algorithms.
\newblock In {\em Advances in neural information processing systems}, pages
  703--710.

\bibitem[Jin et~al., 2018]{jin2018q}
Jin, C., Allen-Zhu, Z., Bubeck, S., and Jordan, M.~I. (2018).
\newblock {Is $Q$-learning provably efficient?}
\newblock In {\em Proceedings of the 32nd International Conference on Neural
  Information Processing Systems}, pages 4868--4878.

\bibitem[Kearns and Singh, 2000]{kearns2000bias}
Kearns, M.~J. and Singh, S.~P. (2000).
\newblock {Bias-Variance Error Bounds for Temporal Difference Updates}.
\newblock In {\em COLT}, pages 142--147. Citeseer.

\bibitem[Khodadadian et~al., 2021]{khodadadian2021finite}
Khodadadian, S., Chen, Z., and Maguluri, S.~T. (2021).
\newblock {Finite-Sample Analysis of Off-Policy Natural Actor-Critic
  Algorithm}.
\newblock {\em arXiv preprint arXiv:2102.09318}.

\bibitem[Kober et~al., 2013]{kober2013reinforcement}
Kober, J., Bagnell, J.~A., and Peters, J. (2013).
\newblock Reinforcement learning in robotics: A survey.
\newblock {\em The International Journal of Robotics Research},
  32(11):1238--1274.

\bibitem[Kushner, 2010]{kushner2010stochastic}
Kushner, H. (2010).
\newblock Stochastic approximation: a survey.
\newblock {\em Wiley Interdisciplinary Reviews: Computational Statistics},
  2(1):87--96.

\bibitem[Kushner and Clark, 2012]{kushner2012stochastic}
Kushner, H.~J. and Clark, D.~S. (2012).
\newblock {\em Stochastic approximation methods for constrained and
  unconstrained systems}, volume~26.
\newblock Springer Science \& Business Media.

\bibitem[Lan, 2020]{lan2020first}
Lan, G. (2020).
\newblock {\em {First-order and Stochastic Optimization Methods for Machine
  Learning}}.
\newblock Springer.

\bibitem[Levin and Peres, 2017]{levin2017markov}
Levin, D.~A. and Peres, Y. (2017).
\newblock {\em Markov chains and mixing times}, volume 107.
\newblock American Mathematical Soc.

\bibitem[Li et~al., 2023]{li2023q}
Li, G., Cai, C., Chen, Y., Wei, Y., and Chi, Y. (2023).
\newblock {Is $Q$-learning minimax optimal? a tight sample complexity
  analysis}.
\newblock {\em Operations Research}.

\bibitem[Li et~al., 2020]{li2020sample}
Li, G., Wei, Y., Chi, Y., Gu, Y., and Chen, Y. (2020).
\newblock Sample complexity of asynchronous {$Q$}-learning: Sharper analysis
  and variance reduction.
\newblock {\em Preprint arXiv:2006.03041}.

\bibitem[Munos et~al., 2016]{munos2016safe}
Munos, R., Stepleton, T., Harutyunyan, A., and Bellemare, M.~G. (2016).
\newblock {Safe and efficient off-policy reinforcement learning}.
\newblock In {\em Proceedings of the 30th International Conference on Neural
  Information Processing Systems}, pages 1054--1062.

\bibitem[Precup et~al., 2000]{precup2000eligibility}
Precup, D., Sutton, R.~S., and Singh, S.~P. (2000).
\newblock {Eligibility Traces for Off-Policy Policy Evaluation}.
\newblock In {\em Proceedings of the Seventeenth International Conference on
  Machine Learning}, pages 759--766.

\bibitem[Puterman, 1995]{puterman1995markov}
Puterman, M.~L. (1995).
\newblock Markov decision processes: Discrete stochastic dynamic programming.
\newblock {\em Journal of the Operational Research Society}, 46(6):792--792.

\bibitem[Qu and Wierman, 2020]{qu2020finite}
Qu, G. and Wierman, A. (2020).
\newblock {Finite-Time Analysis of Asynchronous Stochastic Approximation and
  $Q$-Learning}.
\newblock In {\em Conference on Learning Theory}, pages 3185--3205. PMLR.

\bibitem[Robbins and Monro, 1951]{robbins1951stochastic}
Robbins, H. and Monro, S. (1951).
\newblock A stochastic approximation method.
\newblock {\em The Annals of Mathematical Statistics}, pages 400--407.

\bibitem[Ryu and Boyd, 2016]{ryu2016primer}
Ryu, E.~K. and Boyd, S. (2016).
\newblock Primer on monotone operator methods.
\newblock {\em Appl. Comput. Math}, 15(1):3--43.

\bibitem[Silver et~al., 2017]{silver2017mastering}
Silver, D., Schrittwieser, J., Simonyan, K., Antonoglou, I., Huang, A., Guez,
  A., Hubert, T., Baker, L., Lai, M., Bolton, A., et~al. (2017).
\newblock Mastering the game of go without human knowledge.
\newblock {\em Nature}, 550(7676):354.

\bibitem[Singh and Sutton, 1996]{singh1996reinforcement}
Singh, S.~P. and Sutton, R.~S. (1996).
\newblock Reinforcement learning with replacing eligibility traces.
\newblock {\em Machine learning}, 22(1):123--158.

\bibitem[Srikant and Ying, 2019]{srikant2019finite}
Srikant, R. and Ying, L. (2019).
\newblock Finite-time error bounds for linear stochastic approximation and {TD}
  learning.
\newblock In {\em Conference on Learning Theory}, pages 2803--2830.

\bibitem[Sutton, 1999]{sutton1999open}
Sutton, R.~S. (1999).
\newblock Open theoretical questions in reinforcement learning.
\newblock In {\em European Conference on Computational Learning Theory}, pages
  11--17. Springer.

\bibitem[Sutton and Barto, 2018]{sutton2018reinforcement}
Sutton, R.~S. and Barto, A.~G. (2018).
\newblock {\em Reinforcement learning: {An} introduction}.
\newblock MIT press.

\bibitem[Szepesv{\'a}ri et~al., 1997]{szepesvari1997asymptotic}
Szepesv{\'a}ri, C. et~al. (1997).
\newblock {The asymptotic convergence-rate of $Q$-learning}.
\newblock In {\em NIPS}, volume~10, pages 1064--1070. Citeseer.

\bibitem[Tsitsiklis, 1994]{tsitsiklis1994asynchronous}
Tsitsiklis, J.~N. (1994).
\newblock Asynchronous stochastic approximation and {$Q$}-learning.
\newblock {\em Machine learning}, 16(3):185--202.

\bibitem[Tsitsiklis and Van~Roy, 1997]{tsitsiklis1997analysis}
Tsitsiklis, J.~N. and Van~Roy, B. (1997).
\newblock Analysis of temporal-difference learning with function approximation.
\newblock In {\em Advances in neural information processing systems}, pages
  1075--1081.

\bibitem[Tsitsiklis and Van~Roy, 1999]{tsitsiklis1999average}
Tsitsiklis, J.~N. and Van~Roy, B. (1999).
\newblock Average cost temporal-difference learning.
\newblock {\em Automatica}, 35(11):1799--1808.

\bibitem[Wainwright, 2019]{wainwright2019stochastic}
Wainwright, M.~J. (2019).
\newblock Stochastic approximation with cone-contractive operators: Sharp
  $\ell_\infty$-bounds for ${Q}$-learning.
\newblock {\em Preprint arXiv:1905.06265}.

\bibitem[Watkins and Dayan, 1992]{watkins1992q}
Watkins, C.~J. and Dayan, P. (1992).
\newblock {$Q$}-learning.
\newblock {\em Machine learning}, 8(3-4):279--292.

\bibitem[Watkins, 1989]{watkins1989learning}
Watkins, C. J. C.~H. (1989).
\newblock Learning from delayed rewards.

\end{thebibliography}

\begin{center}
    {\LARGE\bfseries Appendices}
\end{center}

\appendix

\section{Proof of Theorem \ref{thm:sa}}\label{pf:thm:sa}
We will state and prove a more general version of Theorem \ref{thm:sa}. To do that, we need to introduce more notation and explicitly specify the requirement for choosing the stepsize sequence $\{\alpha_k\}$. 

\begin{notation}
	Let $g(x)=\frac{1}{2}\|x\|_s^2$, where the norm $\|\cdot\|_s$ is properly chosen so that the function $g(\cdot)$ is a smooth function with respect to the norm $\|\cdot\|_s$. That is, the function $g(\cdot)$ is convex, differentiable, and there exists $L>0$ such that $g(x_2)\leq g(x_1)+\langle\nabla g(x_1),x_2-x_1\rangle+\frac{L}{2}\|x_1-x_2\|_s^2$ for any $x_1,x_2\in\mathbb{R}^d$. For example, $\ell_p$-norm with $p\in [2,\infty)$ works with $L=p-1$ \citep{beck2017first}. Since we work with finite-dimensional space $\mathbb{R}^d$, there exist $\ell_{cs},u_{cs}>0$ such that $\ell_{cs}\|\cdot\|_c\leq \|\cdot\|_c\leq u_{cs}\|\cdot\|_s$. Let $\theta>0$ be chosen such that $\beta^2< \frac{1+\theta \ell_{cs}^2}{1+\theta u_{cs}^2}$, which is always possible since $\beta\in(0,1)$. Denote 
	\begin{align}\label{eq:def:constants}
		\varphi_1=\frac{1+\theta u_{cs}^2}{1+\theta \ell_{cs}^2},\quad  \varphi_2=1-\beta\varphi_1^{1/2}, \quad \text{and} \quad \varphi_3=\frac{114L (1+\theta u_{cs}^2)}{\theta\ell_{cs}^2},
	\end{align}
	which are the constants we used to state Theorem \ref{thm:sa}. Note that $\varphi_2\in (0,1)$ under our choice of $\theta$. 
\end{notation}

Now we state the requirement in choosing the stepsizes $\{\alpha_k\}$. For simplicity, we use $\alpha_{i,j}$ for $\sum_{k=i}^{j}\alpha_k$.

\begin{condition}\label{con1:stepsize}
	The sequence $\{\alpha_k\}$ is non-increasing and satisfies $\alpha_{k-t_k,k-1}\leq \min(\frac{\varphi_2}{\varphi_3A^2},\frac{1}{4A})$ for all $k\geq t_k$.
\end{condition}
We next state a more general version of Theorem \ref{thm:sa}. Recall that $K=\min\{k:k\geq t_k\}$, which is well-defined under Assumption \ref{as:MC}.

\begin{theorem}\label{thm:main}
	Consider $\{x_k\}$ generated by Algorithm (\ref{algo:sa}). Suppose that Assumptions \ref{as:F}, \ref{as:barF}, \ref{as:MC} and \ref{as:noise_w} are satisfied, and the stepsize sequence $\{\alpha_k\}$ satisfies Condition \ref{con1:stepsize}. Then we have the following results.
	\begin{enumerate}[(1)]
		\item For any $k\in [0,K-1]$, we have: $\|x_k-x^*\|_c^2\leq c_1$ almost surely.
		\item For any $k\geq K$, we have
		\begin{align*}
			\mathbb{E}[\|x_k-x^*\|_c^2]
			\leq \varphi_1c_1\prod_{j=K}^{k-1}(1-\varphi_2\alpha_j)+\varphi_3c_2\sum_{i=K}^{k-1}\alpha_i\alpha_{i-t_i,i-1}\prod_{j=i+1}^{k-1}(1-\varphi_2\alpha_j),
		\end{align*}where $c_1=(\|x_0-x^*\|_c+\|x_0\|_c+B/A)^2$ and $c_2=(A\|x^*\|_c+B)^2$.	
	\end{enumerate}
\end{theorem}

Once we have Theorem \ref{thm:main}, we can evaluate the bound when $\alpha_k=\frac{\alpha}{(k+h)^\xi}$ to get Theorem \ref{thm:sa}. This is presented in Appendix \ref{ap:pf:sa:stepsizes}. In Appendix \ref{ap:pf:sa:stepsizes}, we also show how Condition \ref{con:stepsize} is obtained from Condition \ref{con1:stepsize} and the explicit requirements on the thresholds $\bar{c}$ and $\bar{h}$. We next present the proof of Theorem \ref{thm:main}.

\subsection{Proof of Theorem \ref{thm:main}}

\subsubsection{Step One: Constructing a Valid Lyapunov Function}
Let $f(x)=\frac{1}{2}\|x\|_c^2$. We will use the Generalized Moreau Envelope of $f(\cdot)$ with respect to $g(\cdot)$: 
\begin{align*}
    M_f^{\theta,g}(x)=\min_{u\in\mathbb{R}^d}\left\{f(u)+\frac{1}{\theta}g(x-u)\right\}
\end{align*}
as the Lyapunov function to study Algorithm (\ref{algo:sa}). We first summarize the properties of $M_f^{\theta,g}(\cdot)$ in the following proposition, which was established in \citep{chen2020finite}. For simplicity, we will just write $M(\cdot)$ for $M_f^{\theta,g}(\cdot)$ in the following unless we want to emphasize the dependence on the choices of $\theta$ and $g(\cdot)$.
\begin{proposition}\label{prop:Moreau}
	The function $M(x)$ has the following properties.
	\begin{enumerate}[(1)]
		\item $M(x)$ is convex, and $\frac{L}{\theta}$-smooth with respect to $\|\cdot\|_s$. That is, $M(y)\leq M(x)+\langle \nabla M(x),y-x\rangle+\frac{L}{2\theta}\|x-y\|_s^2$ for all $x,y\in\mathbb{R}^d$.
		\item There exists a norm, denoted by $\|\cdot\|_m$, such that $M(x)=\frac{1}{2}\|x\|_m^2$.
		\item Let $\ell_{cm}=(1+\theta \ell_{cs}^2)^{1/2}$ and $u_{cm}=(1+\theta u_{cs}^2)^{1/2}$. Then it holds that $\ell_{cm}\|\cdot\|_m\leq \|\cdot\|_c\leq u_{cm}\|\cdot\|_m$. 
	\end{enumerate}	 
\end{proposition}
Using Proposition \ref{prop:Moreau} and the update equation (\ref{algo:sa}), we have for any $k\geq 0$:
\begin{align}\label{eq:composition1}
	&M(x_{k+1}-x^*)\nonumber\\
	\leq \;&M(x_k-x^*)+\langle \nabla M(x_k-x^*),x_{k+1}-x_k\rangle+\frac{L}{2\theta}\|x_{k+1}-x_k\|_s^2\nonumber\\
	=\;&M(x_k-x^*)+\alpha_k\langle \nabla M(x_k-x^*),F(x_k,Y_k)-x_k+w_k\rangle+\frac{L\alpha_k^2}{2\theta}\|F(x_k,Y_k)-x_k+w_k\|_s^2\nonumber\\
	=\;&M(x_k-x^*)+\underbrace{\alpha_k\langle \nabla M(x_k-x^*),\bar{F}(x_k)-x_k\rangle}_{T_1:\text{ Expected update}}+\underbrace{\alpha_k\langle \nabla M(x_k-x^*),w_k\rangle}_{T_2:\text{ Error due to Martingale difference noise }w_k}\nonumber\\
	&+\underbrace{\alpha_k\langle \nabla M(x_k-x^*),F(x_k,Y_k)-\bar{F}(x_k)\rangle}_{T_3:\text{ Error due to Markovian noise } Y_k}+\underbrace{\frac{L\alpha_k^2}{2\theta}\|F(x_k,Y_k)-x_k+w_k\|_s^2}_{T_4:\text{ Error due to discretization and noises}}.
\end{align}
The term $T_1$ represents the expected update of the stochastic iterative algorithm (\ref{algo:sa}), and is bounded in the following lemma, whose proof can be found in \citep{chen2020finite}.
\begin{lemma}\label{le:T1}
	The following inequality holds for all $k\geq 0$: $T_1\leq -2\left(1-\beta \frac{u_{cm}}{\ell_{cm}}\right)\alpha_kM(x_k-x^*)$.
\end{lemma}

As we have seen in Lemma \ref{le:T1}, the term $T_1$ provides us the desired negative drift, i.e., the $-\mathcal{O}(\alpha_k)$ term in the target one-step contractive inequality (\ref{eq:one-step-contraction}). What remains to do is to control all the error terms $T_2$ to $T_4$ in Eq. (\ref{eq:composition1}).

\subsubsection{Step Two: Bounding the Error Terms}
We begin with the term $T_2$. Since $\{w_k\}$ is a martingale difference sequence with respect to the filtration $\mathcal{F}_k$ (cf. Assumption \ref{as:noise_w}), while $x_k$ is measurable with respect to $\mathcal{F}_k$, we have by the tower property of conditional expectation that
\begin{align*}
	\mathbb{E}[T_2]=\mathbb{E}[\mathbb{E}[T_2\mid \mathcal{F}_k]]=\alpha_k\mathbb{E}[\langle \nabla M(x_k-x^*),\mathbb{E}[w_k\mid\mathcal{F}_k]]\rangle=0.
\end{align*}

Next we analyze the error term $T_3$, which is due to the Markovian noise $\{Y_k\}$. We first decompose $T_3$ in the following way:
\begin{align}\label{eq:4}
	T_3=\;&\alpha_k\langle \nabla M(x_k-x^*),F(x_k,Y_k)-\bar{F}(x_k)\rangle\nonumber\\
	=\;&\alpha_k\underbrace{\langle \nabla M(x_k-x^*)-\nabla M(x_{k-t_k}-x^*),F(x_k,Y_k)-\bar{F}(x_k)\rangle}_{T_{31}}\nonumber\\
	&+\alpha_k\underbrace{\langle \nabla M(x_{k-t_k}-x^*),F(x_k,Y_k)-F(x_{k-t_k},Y_k)+\bar{F}(x_{k-t_k})-\bar{F}(x_k)\rangle}_{T_{32}}\nonumber\\
	&+\alpha_k\underbrace{\langle \nabla M(x_{k-t_k}-x^*),F(x_{k-t_k},Y_k)-\bar{F}(x_{k-t_k})\rangle}_{T_{33}}.
\end{align}

To proceed, we need the following lemma, which allows us to control the difference between $x_{k_1}$ and $x_{k_2}$ when $|k_1-k_2|$ is relatively small. The proof can be found in Appendix \ref{pf:le:difference}. 

\begin{lemma}\label{le:difference}
	Given non-negative integers $k_1\leq k_2$ satisfying $\alpha_{k_1,k_2-1}\leq \frac{1}{4A}$, we have for all $k\in [k_1,k_2]$:
	\begin{align*}
		\|x_k-x_{k_1}\|_c\leq 2\alpha_{k_1,k_2-1}(A\|x_{k_1}\|_c+B),\quad \text{and}\quad 
		\|x_k-x_{k_1}\|_c\leq 4\alpha_{k_1,k_2-1}(A\|x_{k_2}\|_c+B).
	\end{align*}
\end{lemma}
Using the assumption that $\alpha_{k_1,k_2-1}\leq \frac{1}{4A}$ in the resulting inequality of Lemma \ref{le:difference}, we have the following corollary, which will also be frequently used in the derivation.
\begin{corollary}\label{co:difference}
	Under same conditions given in Lemma \ref{le:difference}, we have for all $k\in [k_1,k_2]$:
	\begin{align*}
		\|x_k-x_{k_1}\|_c\leq
		\max(\|x_{k_1}\|_c,\|x_{k_2}\|_c)+\frac{B}{A}.
	\end{align*}
\end{corollary}

Recall that we require $\alpha_{k-t_k,k-1}\leq \frac{1}{4A}$ for all $k\geq t_k$ in Condition \ref{con1:stepsize}. Therefore, Lemma \ref{le:difference} is applicable when $k_1=k-t_k$ and $k_2=k-1$ for any $k\geq t_k$. 

Now we are ready to control the terms $T_{31}$, $T_{32}$, and $T_{33}$ in the following lemma. The terms $T_{31}$ and $T_{32}$ are controlled mainly by constantly applying Lemma \ref{le:difference} and the Lipschitz property of the operator $F(\cdot)$ (cf. Assumptions \ref{as:F}). Bounding the term $T_{33}$ requires using the geometric mixing of the Markov chain $\{Y_k\}$ (cf. Assumption \ref{as:MC}). The proof is presented in Appendix \ref{pf:le:T31T32T33}.
\begin{lemma}\label{le:T31T32T33}
	The following inequalities hold for all $k\geq t_k$:
	\begin{enumerate}[(1)]
		\item $T_{31}\leq \frac{16LA^2u_{cm}^2\alpha_{k-t_k,k-1}}{\theta\ell_{cs}^2}M(x_{k}-x^*)+\frac{8L\alpha_{k-t_k,k-1}}{\theta\ell_{cs}^2}(A\|x^*\|_c+B)^2$,
		\item $T_{32}\leq \frac{64LA^2u_{cm}^2\alpha_{k-t_k,k-1}}{\theta\ell_{cs}^2}M(x_k-x^*)+\frac{32L\alpha_{k-t_k,k-1}}{\theta\ell_{cs}^2}(A\|x^*\|_c+B)^2$,
		\item 
		$\mathbb{E}[T_{33}]\leq \frac{32LA^2u_{cm}^2\alpha_k}{\theta\ell_{cs}^2}\mathbb{E}[M(x_k-x^*)]+\frac{16L\alpha_k}{\theta\ell_{cs}^2}(A\|x^*\|_c+B)^2$.
	\end{enumerate}
\end{lemma}

Now that Lemma \ref{le:T31T32T33} provides upper bounds on the terms $T_{31}$, $T_{32}$, and $T_{33}$, using them in Eq. (\ref{eq:4}) and we have the following result. 

\begin{lemma}\label{le:T3}
	The following inequality holds for all $k\geq t_k$:
	\begin{align*}
		\mathbb{E}[T_3]\leq \frac{112LA^2u_{cm}^2\alpha_k\alpha_{k-t_k,k-1}}{\theta\ell_{cs}^2}\mathbb{E}[M(x_k-x^*)]+\frac{56L\alpha_k\alpha_{k-t_k,k-1}}{\theta\ell_{cs}^2}(A\|x^*\|_c+B)^2.
	\end{align*}
\end{lemma}

Lastly, we bound the error term $T_4$ in the following lemma, whose proof is provided in Appendix \ref{pf:le:T4}.

\begin{lemma}\label{le:T4}
	It holds for any $k\geq 0$ that $T_4\leq  \frac{2LA^2u_{cm}^2\alpha_k^2}{\theta\ell_{cs}^2}M(x_k-x^*)+\frac{L\alpha_k^2}{\theta\ell_{cs}^2}(A\|x^*\|_c+B)^2$.
\end{lemma}

Now we have control on all the error terms $T_1$ to $T_4$. Using them in Eq. (\ref{eq:composition1}), and we obtain the following result. The proof is presented in Appendix \ref{pf:le:recursion}
\begin{lemma}\label{le:recursion}
	The following inequality holds for all $k\geq t_k$:
	\begin{align*}
		\mathbb{E}[M(x_{k+1}-x^*)]\leq\;&
		\left(1-2\left(1-\beta \frac{u_{cm}}{\ell_{cm}}\right)\alpha_k+\frac{114LA^2u_{cm}^2\alpha_k\alpha_{k-t_k,k-1}}{\theta\ell_{cs}^2}\right)\mathbb{E}[M(x_k-x^*)]\\
		&+\frac{57L\alpha_k\alpha_{k-t_k,k-1}}{\theta\ell_{cs}^2}(A\|x^*\|_c+B)^2.
	\end{align*}
\end{lemma}

Note that Lemma \ref{le:recursion} provides the desired one-step contractive inequality. We next repeatedly use Lemma \ref{le:recursion} to derive finite-sample convergence bounds of Algorithm (\ref{algo:sa}). Using the constants $\{\varphi_i\}_{1\leq i\leq 3}$ defined in Eq. (\ref{eq:def:constants}), then Lemma \ref{le:recursion} reads:
\begin{align*}
	\mathbb{E}[M(x_{k+1}-x^*)]\leq
	\left(1-2\varphi_2\alpha_k+\varphi_3 A^2\alpha_k\alpha_{k-t_k,k-1}\right)\mathbb{E}[M(x_k-x^*)]+\frac{\varphi_3\alpha_k\alpha_{k-t_k,k-1}}{2u_{cm}^2}(A\|x^*\|_c+B)^2.
\end{align*}
Since $\alpha_{k-t_k,k-1}\leq \varphi_2/(\varphi_3A^2)$ for all $k\geq K$ (cf. Condition \ref{con1:stepsize}), we have by the previous inequality that
\begin{align*}
	\mathbb{E}[M(x_{k+1}-x^*)]\leq \left(1-\varphi_2\alpha_k\right)\mathbb{E}[M(x_{k}-x^*)]+\frac{\varphi_3\alpha_k\alpha_{k-t_k,k-1}}{2u_{cm}^2}(A\|x^*\|_c+B)^2
\end{align*}
for all $k\geq K$. Recursively using the previous inequality and we have for any $k\geq K$:
\begin{align*}
	&\mathbb{E}[\|x_k-x^*\|_c^2]\\
	\leq \;&2u_{cm}^2\mathbb{E}[M(x_k-x^*)]\tag{Proposition \ref{prop:Moreau}}\nonumber\\
	\leq\;& 2u_{cm}^2\mathbb{E}[M(x_K-x^*)]\prod_{j=K}^{k-1}(1-\varphi_2\alpha_j)+\varphi_3(A\|x^*\|_c+B)^2\sum_{i=K}^{k-1}\alpha_i\alpha_{i-t_i,i-1}\prod_{j=i+1}^{k-1}(1-\varphi_2\alpha_j)\nonumber\\
	\leq\;& \frac{u_{cm}^2}{\ell_{cm}^2}\mathbb{E}[\|x_K-x^*\|_c^2]\prod_{j=K}^{k-1}(1-\varphi_2\alpha_j)+\varphi_3(A\|x^*\|_c+B)^2\sum_{i=K}^{k-1}\alpha_i\alpha_{i-t_i,i-1}\prod_{j=i+1}^{k-1}(1-\varphi_2\alpha_j)\tag{Proposition \ref{prop:Moreau}}\\
	=\;& \varphi_1\mathbb{E}[\|x_K-x^*\|_c^2]\prod_{j=K}^{k-1}(1-\varphi_2\alpha_j)+\varphi_3c_2\sum_{i=K}^{k-1}\alpha_i\alpha_{i-t_i,i-1}\prod_{j=i+1}^{k-1}(1-\varphi_2\alpha_j).
\end{align*}
According to Condition \ref{con1:stepsize}, we also have $\alpha_{0,k-1}\leq 1/(4A)$ for any $k\in [0,K]$. Using Corollary \ref{co:difference} one more time and we have for any $k\in [0,K]$:
\begin{align*}
	\mathbb{E}[\|x_k-x^*\|_c^2]\leq \mathbb{E}[(\|x_k-x_0\|_c+\|x_0-x^*\|_c)^2]\leq  \left(\|x_0-x^*\|_c+\|x_0\|_c+\frac{B}{A}\right)^2=c_1.
\end{align*}
This proves Theorem \ref{thm:main} (1). Since the previous inequality implies $\mathbb{E}[\|x_K-x^*\|_c^2]\leq c_1$, we obtain for all $k\geq K$:
\begin{align}\label{eq:final}
	\mathbb{E}[\|x_k-x^*\|_c^2]\leq   \varphi_1c_1\prod_{j=K}^{k-1}(1-\varphi_2\alpha_j)+\varphi_3c_2\sum_{i=K}^{k-1}\alpha_i\alpha_{i-t_i,i-1}\prod_{j=i+1}^{k-1}(1-\varphi_2\alpha_j).
\end{align}
This proves Theorem \ref{thm:main} (2). 

\subsection{Finite-Sample Convergence Bounds for Using Various Stepsizes}\label{ap:pf:sa:stepsizes}

We next proceed to prove Theorem \ref{thm:sa} by evaluating the convergence bounds in Theorem \ref{thm:main} when the stepsize sequence is chosen by $\alpha_k=\frac{\alpha}{(k+h)^\xi}$, where $\alpha,h>0$ and $\xi\in (0,1)$. We begin by restating Theorem \ref{thm:sa} in full details.

\begin{theorem}\label{thm:sa1}
	Consider $\{x_k\}$ of Algorithm (\ref{algo:sa}). Suppose that Assumptions \ref{as:F}, \ref{as:barF}, \ref{as:MC} and \ref{as:noise_w} are satisfied. Then we have the following results.
	\begin{enumerate}[(1)]
		\item When $k\in [0,K-1]$, we have $\|x_k-x^*\|_c^2\leq c_1$ almost surely.
		\item When $k\geq K$, we have the following finite-sample convergence bounds.
		\begin{enumerate}[(a)]
			\item 
			Let $\bar{\alpha}\in (0,1)$ be chosen such that $\alpha t_\alpha \leq \min(\frac{\varphi_2}{\varphi_3A^2},\frac{1}{4A})$ for all $\alpha \in (0,\bar{\alpha})$. Then when $\alpha_k\equiv \alpha\in (0,\bar{\alpha})$, we have for all $k\geq t_\alpha$:
			\begin{align*}
				\mathbb{E}[\|x_k-x^*\|_c^2]
				\leq  \varphi_1c_1(1-\varphi_2\alpha)^{k-t_\alpha}+\frac{\varphi_3c_2}{\varphi_2}\alpha t_\alpha.
			\end{align*}
			\item When $\alpha_k=\frac{\alpha}{k+h}$, for any $\alpha>0$, let $\bar{h}$ be chosen such that $\alpha_{0,K-1}\leq\min(\frac{\varphi_2}{\varphi_3A^2},\frac{1}{4A}) $ for all $h\geq \bar{h}$. Then
			\begin{enumerate}[(i)]
				\item When $\alpha<1/\varphi_2$, we have for all $k\geq K$:
				\begin{align*}
					\mathbb{E}[\|x_k-x^*\|_c^2]\leq 
					\varphi_1c_1\left(\frac{K+h}{k+h}\right)^{\varphi_2\alpha}+\frac{8\alpha^2\varphi_3c_2}{1-\varphi_2\alpha}\frac{t_k}{(k+h)^{\varphi_2\alpha}}.
				\end{align*}
				\item When $\alpha=1/\varphi_2$, we have for all $k\geq K$:
				\begin{align*}
					\mathbb{E}[\|x_k-x^*\|_c^2]\leq 
					\varphi_1c_1\frac{K+h}{k+h}+8\alpha^2\varphi_3c_2\frac{t_k\log(k+h)}{k+h}.
				\end{align*}
				\item When $\alpha>1/\varphi_2$, we have for all $k\geq K$:
				\begin{align*}
					\mathbb{E}[\|x_k-x^*\|_c^2]\leq 
					\varphi_1c_1\left(\frac{K+h}{k+h}\right)^{\varphi_2\alpha}+\frac{8e\alpha^2\varphi_3c_2}{\varphi_2\alpha-1}\frac{t_k}{k+h}.
				\end{align*}
			\end{enumerate}
			\item When $\alpha_k=\frac{\alpha}{(k+h)^\xi}$, for any $\xi\in (0,1)$ and $\alpha>0$, let $\bar{h}$ be chosen such that $\bar{h}\geq  \left[2\xi/(\varphi_2\alpha)\right]^{1/(1-\xi)}$ and $\alpha_{0,K-1}\leq\min(\frac{\varphi_2}{\varphi_3A^2},\frac{1}{4A}) $ for any $h\geq \bar{h}$. Then we have for all $k\geq K$:
			\begin{align*}
				\mathbb{E}[\|x_k-x^*\|_c^2]\leq \varphi_1c_1 e^{-\frac{\varphi_2\alpha}{1-\xi}\left((k+h)^{1-\xi}-(K+h)^{1-\xi}\right)}+\frac{4\varphi_3c_2\alpha}{\varphi_2}\frac{t_k}{(k+h)^\xi}.
			\end{align*}
		\end{enumerate}
	\end{enumerate}
\end{theorem}
\begin{proof}[Proof of Theorem \ref{thm:sa1}]

\begin{enumerate}[(1)]
    \item Theorem \ref{thm:sa1} (1) directly follows from Theorem \ref{thm:main} (1).
    \item 
    \begin{enumerate}[(a)]
        \item When using constant stepsize $\alpha$, it is clear that Condition \ref{con1:stepsize} is satisfied when $\alpha t_\alpha\leq \min(\frac{\varphi_2}{\varphi_3A^2},\frac{1}{4A})$. We next verify the existence of such threshold $\bar{\alpha}$. Note that we have by definition of $t_\alpha$ and Assumption \ref{as:MC} that
	\begin{align*}
		t_\alpha\leq \min\left\{k\geq 0\;:\; C\sigma^k\leq \alpha\right\}=\min\left\{ k\geq 0\;:\; k\geq  \frac{\log(1/\alpha)+\log(C)}{\log(1/\sigma)}\right\}\leq \frac{\log(1/\alpha)+\log(C/\sigma)}{\log(1/\sigma)}.
	\end{align*}
	It follows that $\lim_{\alpha\rightarrow 0}\alpha t_\alpha=0$. Hence there exists $\bar{\alpha}\in (0,1)$ such that Condition \ref{con1:stepsize} is satisfied for all $\alpha\in (0,\bar{\alpha})$, which is stated in Condition \ref{con:stepsize} (1). We next evaluate Eq. (\ref{eq:final}).
	When $\alpha_k\equiv \alpha$, we have for all $k\geq t_\alpha$:
	\begin{align*}
		\mathbb{E}[\|x_k-x^*\|_c^2]&\leq   \varphi_1c_1\prod_{j=t_\alpha}^{k-1}(1-\varphi_2\alpha_j)+\varphi_3c_2\sum_{i=t_\alpha}^{k-1}\alpha_i\alpha_{i-t_i,i-1}\prod_{j=i+1}^{k-1}(1-\varphi_2\alpha_j)\\
		&=\varphi_1c_1(1-\varphi_2\alpha)^{k-t_\alpha}+\varphi_3c_2\sum_{i=t_\alpha}^{k-1}\alpha^2t_\alpha(1-\varphi_2\alpha)^{k-i-1}\\
		&\leq \varphi_1c_1(1-\varphi_2\alpha)^{k-t_\alpha}+\frac{\varphi_3c_2}{\varphi_2}\alpha t_\alpha.
	\end{align*}
	This proves Theorem \ref{thm:sa1} (2) (a).
	\item Consider the case where $\alpha_k=\frac{\alpha}{k+h}$. We first verify the existence of the threshold $\bar{h}$. We begin by comparing $\alpha_{k-t_k}$ with $\alpha_k$. Using Assumption \ref{as:MC} and we have 
	\begin{align*}
		t_k\leq  \frac{\log(k+h)+\log(C/(\sigma\alpha))}{\log(1/\sigma)}.
	\end{align*}
	It follows that 
	\begin{align*}
		\frac{\alpha_k}{\alpha_{k-t_k}}=1-\frac{t_k}{k+h}\rightarrow 1 \text{ as }(k+h)\rightarrow\infty.
	\end{align*}
	Therefore, there exists $\bar{h}_1>0$ such that $\alpha_{k-t_k}\leq 2\alpha_k$ holds for any $k\geq t_k$ when $h\geq \bar{h}_1$. Now consider the requirement stated in Condition \ref{con1:stepsize}. Using the fact that $\{\alpha_k\}$ is non-increasing, we have
	\begin{align*}
		\alpha_{k-t_k,k-1}\leq t_k\alpha_{k-t_k}\leq 2\alpha_kt_k\rightarrow 0\text{ as }(k+h)\rightarrow \infty.
	\end{align*}
	Hence there exists $\bar{h}_2>0$ such that $\alpha_{k-t_k,k-1}\leq \min(\frac{\varphi_2}{\varphi_3A^2},\frac{1}{4A})$ holds for any $k\geq t_k$ when $h\geq \bar{h}_2$. Now choosing $\bar{h}=\max(\bar{h}_1,\bar{h}_2)$, Condition \ref{con1:stepsize} is satisfied. This is stated in Condition \ref{con:stepsize} (2). Furthermore, by construction we have $\alpha_{k-t_k}\leq 2\alpha_k$ for any $k\geq t_k$. We next evaluate the RHS of Eq. (\ref{eq:final}) in the following lemma, whose proof is presented in Appendix \ref{pf:le:rate_linear_diminishing}.
	\begin{lemma}\label{le:rate_linear_diminishing}
		The following inequality hold for all $k\geq K$:
		\begin{align*}
			\mathbb{E}[\|x_k-x^*\|_c^2]\leq\begin{dcases}
				\varphi_1c_1\left(\frac{K+h}{k+h}\right)^{\varphi_2\alpha}+\frac{8\varphi_3c_2\alpha^2}{1-\varphi_2\alpha}\frac{t_k}{(k+h)^{\varphi_2\alpha}},&\alpha<\frac{1}{\varphi_2},\\
				\varphi_1c_1\frac{K+h}{k+h}+8\varphi_3c_2\alpha^2\frac{t_k\log(k+h)}{k+h},&\alpha=\frac{1}{\varphi_2},\\
				\varphi_1c_1\left(\frac{K+h}{k+h}\right)^{\varphi_2\alpha}+\frac{8e\varphi_3c_2\alpha^2}{\varphi_2\alpha-1}\frac{t_k}{k+h},&\alpha>\frac{1}{\varphi_2}.
			\end{dcases}
		\end{align*}
	\end{lemma}
	This proves Theorem \ref{thm:sa1} (2) (b).
	\item Now we consider using $\alpha_k=\frac{\alpha}{(k+h)^\xi}$, where $\xi\in (0,1)$ and $\alpha,h>0$. Using the same line of proof as in the previous section, one can show that for any $\xi\in (0,1)$ and $\alpha>0$, there exists $\bar{h}>0$ such that Condition \ref{con1:stepsize} is satisfied for all $h\geq \bar{h}$. Furthermore, we assume without loss of generality that $\alpha_{k-t_k}\leq 2\alpha_k$ for all $k\geq t_k$ and $\bar{h}\geq  \left[2\xi/(\varphi_2\alpha)\right]^{1/(1-\xi)}$. We next evaluate the RHS of Eq. (\ref{eq:final}) in the following lemma, whose proof is presented in Appendix \ref{pf:le:rate_polynomial_diminishing}.
	\begin{lemma}\label{le:rate_polynomial_diminishing}
		The following inequality hold for all $k\geq K$:
		\begin{align*}
			\mathbb{E}[\|x_k-x^*\|_c^2]\leq
			\varphi_1c_1 \exp\left[-\frac{\varphi_2\alpha}{1-\xi}\left((k+h)^{1-\xi}-(K+h)^{1-\xi}\right)\right]+\frac{4\varphi_3c_2\alpha}{\varphi_2}\frac{t_k}{(k+h)^\xi}.
		\end{align*}
	\end{lemma}
	This proves Theorem \ref{thm:sa1} (2) (c).
    \end{enumerate}
\end{enumerate}
\end{proof}

\subsection{Proof of Technical Lemmas}

\subsubsection{Proof of Lemma \ref{le:constants}}\label{pf:le:constants}
\begin{enumerate}[(1)]
	\item When $\|\cdot\|_c=\|\cdot\|_2$, we choose $\theta=1$ and $g(x)=\frac{1}{2}\|x\|_2^2$. It follows that $L=1$ and $u_{cs}=\ell_{cs}=1$. Therefore, we have by definition (\ref{eq:def:constants}) that $\varphi_1=1$, $\varphi_2=1-\beta$, and $\varphi_3=228$.
	\item Recall the definition of $\{\varphi_i\}_{1\leq  i\leq 3}$ in Eq. (\ref{eq:def:constants}). When $\|\cdot\|_c=\|\cdot\|_\infty$, we choose $\theta=\left(\frac{1+\beta}{2\beta}\right)^2-1$ and $g(x)=\frac{1}{2}\|x\|_p^2$ with $p=2\log(d)$, where $d$ is the dimension of the iterates $x_k$. It follows that $L=p-1\leq 2\log(d)$ \citep{beck2017first}, $u_{cs}=1$, and $\ell_{cs}=1/d^{1/p}=1/\sqrt{e}$. Therefore, we have
	\begin{align*}
		\varphi_1&=\frac{1+\theta u_{cs}^2}{1+\theta \ell_{cs}^2}=\frac{1+\theta}{1+\theta/\sqrt{e}}\leq \sqrt{e}\leq3,\\
		\varphi_2&=1-\beta\varphi_1^{1/2}\geq  1-\beta\frac{1+\beta}{2\beta}=\frac{1-\beta}{2},\\
		\varphi_3&=\frac{114L(1+\theta u_{cs}^2)}{\theta \ell_{cs}^2}\leq \frac{228e\log(d)(1+\theta)}{\theta}\leq \frac{456e\log(d)}{1-\beta}.
	\end{align*}
\end{enumerate}

\subsubsection{Proof of Lemma \ref{le:difference}}\label{pf:le:difference}
We first show that under Assumption \ref{as:F}, the size of $\|F(x,y)\|_c$ and $\|\bar{F}(x)\|_c$ can grow at most affinely in terms of $\|x\|_c$. Using Triangle inequality, we have 
\begin{align*}
	\|F(x,y)\|_c-\|F(\bm{0},y)\|_c\leq \|F(x,y)-F(\bm{0},y)\|_c\leq A_1\|x\|_c,\quad\forall\;x\in\mathbb{R}^d,y\in\mathcal{Y},
\end{align*}
where the last inequality follows from Assumption \ref{as:F}. 
It follows that
\begin{align*}
	\|F(x,y)\|_c\leq A_1\|x\|_c+\|F(\bm{0},y)\|_c\leq A_1\|x\|_c+B_1.
\end{align*}
Furthermore, we have by Jensen's inequality and the convexity of norms that 
\begin{align*}
	\|\bar{F}(x)\|_c=\|\mathbb{E}_{Y\sim \mu}[F(x,Y)]\|_c\leq \mathbb{E}_{Y\sim \mu}[\|F(x,Y)\|_c]\leq A_1\|x\|_c+B_1.
\end{align*}
The previous two inequalities will be frequently used in the derivation here after. Now we proceed to prove Lemma \ref{le:difference}. For any $k\in [k_1,k_2-1]$, using Triangle inequality, we have
\begin{align}\label{eq:5}
	\|x_{k+1}\|_c-\|x_k\|_c
	\leq\;& 	
	\|x_{k+1}-x_k\|_c\nonumber\\
	=\;&\alpha_k \|F(x_k,Y_k)-x_k+w_k\|_c\nonumber\\
	\leq \;&\alpha_k( \|F(x_k,Y_k)\|_c+\|x_k\|_c+\|w_k\|_c)\nonumber\\
	\leq \;&\alpha_k( A_1\|x_k\|_c+B_1+\|x_k\|_c+A_2\|x_k\|_c+B_2)\tag{Assumptions \ref{as:F}, \ref{as:noise_w}}.\\
	\leq \;&\alpha_k( (A_1+A_2+1)\|x_k\|_c+B_1+B_2)\nonumber\\
	= \;&\alpha_k(A\|x_k\|_c+B)\label{eq:10}.
\end{align}
Note that the previous inequality is equivalent to
\begin{align*}
	\|x_{k+1}\|_c+\frac{B}{A}\leq (1+A\alpha_k)\left(\|x_{k}\|_c+\frac{B}{A}\right),
\end{align*}
which implies for all $k\in [k_1,k_2]$:
\begin{align*}
	\|x_k\|_c\leq \prod_{j=k_1}^{k-1}(1+A\alpha_j)\left(\|x_{k_1}\|_c+\frac{B}{A}\right)-\frac{B}{A}.
\end{align*}
Using the fact that $1+x\leq  e^x\leq 1+2x$ for all $x\in [0,1/2]$, we have when $\alpha_{k_1,k_2-1}\leq \frac{1}{4A}$:
\begin{align*}
	\prod_{j=k_1}^{k-1}(1+A\alpha_j)\leq \exp\left(A\alpha_{k_1,k-1}\right)\leq 1+2A\alpha_{k_1,k-1}.
\end{align*}
It follows that for all $k\in [k_1,k_2]$:
\begin{align*}
	\|x_k\|_c\leq (1+2A\alpha_{k_1,k-1})\|x_{k_1}\|_c+2B\alpha_{k_1,k-1}.
\end{align*}
Using the previous inequality in Eq. (\ref{eq:10}) and we have for any $k\in [k_1,k_2-1]$:
\begin{align*}
	\|x_{k+1}-x_k\|_c&\leq \alpha_k(A\|x_k\|_c+B)\\
	&\leq \alpha_k A(1+2A\alpha_{k_1,k-1})\|x_{k_1}\|_c+2\alpha_kAB\alpha_{k_1,k-1}\\
	&\leq 2\alpha_k (A\|x_{k_1}\|_c+B).\tag{$\alpha_{k_1,k-1}\leq\frac{1}{4A}$}
\end{align*}
Hence, we have for any $k\in [k_1,k_2]$:
\begin{align*}
	\|x_k-x_{k_1}\|_c\leq \sum_{j=k_1}^{k-1}\|x_{j+1}-x_j\|_c
	\leq 2\sum_{j=k_1}^{k-1}\alpha_j(A\|x_{k_1}\|_c+B)
	=2\alpha_{k_1,k-1}(A\|x_{k_1}\|_c+B).
\end{align*}
Since $\alpha_{k_1,k-1}\leq \alpha_{k_1,k_2-1}$ when $k\in [k_1,k_2]$, we obtain the first claimed inequality:
\begin{align*}
	\|x_k-x_{k_1}\|_c\leq 2\alpha_{k_1,k_2-1}(A\|x_{k_1}\|_c+B),\quad \forall \;k\in [k_1,k_2].
\end{align*}
Now for the second claimed inequality, since
\begin{align*}
	\|x_{k_2}-x_{k_1}\|_c&\leq 2\alpha_{k_1,k_2-1}(A\|x_{k_1}\|_c+B)\\
	&\leq 2\alpha_{k_1,k_2-1}(A\|x_{k_1}-x_{k_2}\|_c+A\|x_{k_2}\|_c+B)\\
	&\leq \frac{1}{2}\|x_{k_2}-x_{k_1}\|_c+ 2\alpha_{k_1,k_2-1}(A\|x_{k_2}\|_c+B),
\end{align*}
we have $\|x_{k_2}-x_{k_1}\|_c\leq   4\alpha_{k_1,k_2-1}(A\|x_{k_2}\|_c+B)$. Therefore, we have for any $k\in [k_1,k_2]$:
\begin{align*}
	\|x_k-x_{k_1}\|_c&\leq 2\alpha_{k_1,k_2-1}(A\|x_{k_1}\|_c+B)\\
	&\leq 2\alpha_{k_1,k_2-1}(A\|x_{k_1}-x_{k_2}\|_c+A\|x_{k_2}\|_c+B)\\
	&\leq 2\alpha_{k_1,k_2-1}(4A\alpha_{k_1,k_2-1}(A\|x_{k_2}\|_c+B)+A\|x_{k_2}\|_c+B)\\\
	&\leq 4\alpha_{k_1,k_2-1}(A\|x_{k_2}\|_c+B)\tag{$\alpha_{k_1,k_2-1}\leq \frac{1}{4A}$},
\end{align*}
which is the second claimed inequality.

\subsubsection{Proof of Lemma \ref{le:T31T32T33}}\label{pf:le:T31T32T33}
\begin{enumerate}[(1)]
	\item For the term $T_{31}$, using H\"{o}lder's inequality and we have
	\begin{align*}
		T_{31}&=\langle \nabla M(x_k-x^*)-\nabla M(x_{k-t_k}-x^*),F(x_k,Y_k)-\bar{F}(x_k)\rangle\\
		&\leq \|\nabla M(x_k-x^*)-\nabla M(x_{k-t_k}-x^*)\|_s^*\|F(x_k,Y_k)-\bar{F}(x_k)\|_s\\
		&\leq \frac{1}{\ell_{cs}}\|\nabla M(x_k-x^*)-\nabla M(x_{k-t_k}-x^*)\|_s^*\|F(x_k,Y_k)-\bar{F}(x_k)\|_c,\tag{$\ell_{cs}\|\cdot\|_s\leq \|\cdot\|_c$}
	\end{align*}
	where $\|\cdot\|_s^*$ denotes the dual norm of $\|\cdot\|_s$. We first control the term $\|\nabla M(x_k-x^*)-\nabla M(x_{k-t_k}-x^*)\|_s^*$. Recall that an equivalent definition of a convex function $h(x)$ been $L$ -- smooth with respect to norm $\|\cdot\|$ is that
	\begin{align*}
		\|\nabla h(x_1)-\nabla h(x_2)\|_*\leq L\|x_1-x_2\|,\quad \forall\;x_1,x_2,
	\end{align*}
	where $\|\cdot\|_*$ is the dual norm of $\|\cdot\|$ \citep{beck2017first}. Therefore, since $M(x)$ is $\frac{L}{\theta}$-smooth with respect to $\|\cdot\|_s$, we have
	\begin{align}\label{eq:1}
		\|\nabla M(x_k-x^*)-\nabla M(x_{k-t_k}-x^*)\|_s^*
		\leq\;& \frac{L}{\theta}\|x_k-x_{k-t_k}\|_s\nonumber\\
		\leq\;& \frac{L}{\theta\ell_{cs}}\|x_k-x_{k-t_k}\|_c\nonumber\\
		\leq\;& \frac{4L\alpha_{k-t_k,k-1}}{\theta\ell_{cs}}(A\|x_{k}-x^*\|_c+A\|x^*\|_c+B),
	\end{align}
	where the last line follows from Lemma \ref{le:difference} and Triangle inequality.
	
	We next control the term $\|F(x_k,Y_k)-\bar{F}(x_k)\|_c$. Using Assumptions \ref{as:F}, \ref{as:barF}, and the fact that $\bar{F}(x^*)=x^*$, we  have
	\begin{align*}
		\|F(x_k,Y_k)-\bar{F}(x_k)\|_c&=\|F(x_k,Y_k)-\bar{F}(x_k)+\bar{F}(x^*)-x^*\|_c\\
		&\leq \|F(x_k,Y_k)\|_c+\|\bar{F}(x_k)-\bar{F}(x^*)\|_c+\|x^*\|_c\\
		&\leq A_1\|x_k\|_c+B_1+\|x_k-x^*\|_c+\|x^*\|_c\\
		&\leq (A_1+1)\|x_k-x^*\|_c+(A_1+1)\|x^*\|_c+B_1\\
		&\leq A\|x_k-x^*\|_c+A\|x^*\|_c+B.
	\end{align*}
	It follows that
	\begin{align*}
		T_{31}&\leq \frac{1}{\ell_{cs}}\|\nabla M(x_k-x^*)-\nabla M(x_{k-t_k}-x^*)\|_s^*\|F(x_k,Y_k)-\bar{F}(x_k)\|_c\\
		&\leq \frac{4L\alpha_{k-t_k,k-1}}{\theta\ell_{cs}^2}(A\|x_{k}-x^*\|_c+A\|x^*\|_c+B)^2\\
		&\leq \frac{8L\alpha_{k-t_k,k-1}}{\theta\ell_{cs}^2}A^2\|x_{k}-x^*\|_c^2+\frac{8L\alpha_{k-t_k,k-1}}{\theta\ell_{cs}^2}(A\|x^*\|_c+B)^2\tag*{$(a+b)^2\leq 2(a^2+b^2)$}\\
		&\leq \frac{16LA^2u_{cm}^2\alpha_{k-t_k,k-1}}{\theta\ell_{cs}^2}M(x_{k}-x^*)+\frac{8L\alpha_{k-t_k,k-1}}{\theta\ell_{cs}^2}(A\|x^*\|_c+B)^2.
	\end{align*}
	\item Consider the term $T_{32}$. Using H\"{o}lder's inequality and we have
	\begin{align*}
		T_{32}&=\langle \nabla M(x_{k-t_k}-x^*),F(x_k,Y_k)-F(x_{k-t_k},Y_k)+\bar{F}(x_{k-t_k})-\bar{F}(x_k)\rangle\\
		&\leq \|\nabla M(x_{k-t_k}-x^*)\|_s^* \|F(x_k,Y_k)-F(x_{k-t_k},Y_k)+\bar{F}(x_{k-t_k})-\bar{F}(x_k)\|_s\tag{H\"{o}lder's inequality}\\
		&\leq \frac{1}{\ell_{cs}}\|\nabla M(x_{k-t_k}-x^*)\|_s^* \|F(x_k,Y_k)-F(x_{k-t_k},Y_k)+\bar{F}(x_{k-t_k})-\bar{F}(x_k)\|_c.
	\end{align*}
	For the term $\|\nabla M(x_{k-t_k}-x^*)\|_s^*$, we have
	\begin{align}
		\|\nabla M(x_{k-t_k}-x^*)\|_s^*&=\|\nabla M(x_{k-t_k}-x^*)-\nabla M(x^*-x^*)\|_s^*\nonumber\\
		&\leq \frac{L}{\theta}\|x_{k-t_k}-x^*\|_s\tag{equivalent definition of smoothness}\nonumber\\
		&\leq \frac{L}{\theta\ell_{cs}}\|x_{k-t_k}-x^*\|_c\nonumber\\
		&\leq \frac{L}{\theta\ell_{cs}}(\|x_{k-t_k}-x_k\|_c+\|x_k-x^*\|_c)\nonumber\\
		&\leq \frac{2L}{\theta\ell_{cs}}\left(\|x_k-x^*\|_c+\|x^*\|_c+\frac{B}{A}\right)\label{eq:3},
	\end{align}
	where the last line follow from Corollary \ref{co:difference}. For the term $\|F(x_k,Y_k)-F(x_{k-t_k},Y_k)+\bar{F}(x_{k-t_k})-\bar{F}(x_k)\|_c$, using Assumptions \ref{as:F} and \ref{as:barF} and we obtain
	\begin{align*}
		&\|F(x_k,Y_k)-F(x_{k-t_k},Y_k)+\bar{F}(x_{k-t_k})-\bar{F}(x_k)\|_c\\
		\leq \;&\|F(x_k,Y_k)-F(x_{k-t_k},Y_k)\|_c+\|\bar{F}(x_{k-t_k})-\bar{F}(x_k)\|_c\\
		\leq \;&2A_1\|x_k-x_{k-t_k}\|_c\\
		\leq \;&2A\|x_k-x_{k-t_k}\|_c\\
		\leq \;&8A\alpha_{k-t_k,k-1}(A\|x_{k}-x^*\|_c+A\|x^*\|_c+B),
	\end{align*}
	where in the last line we used Lemma \ref{le:difference}. It follows that
	\begin{align*}
		T_{32}&\leq \frac{1}{\ell_{cs}}\|\nabla M(x_{k-t_k}-x^*)\|_s^* \|F(x_k,Y_k)-F(x_{k-t_k},Y_k)+\bar{F}(x_{k-t_k})-\bar{F}(x_k)\|_c\\
		&\leq \frac{16L\alpha_{k-t_k,k-1}}{\theta\ell_{cs}^2}(A\|x_k-x^*\|_c+A\|x^*\|_c+B)^2\\
		&\leq \frac{32LA^2\alpha_{k-t_k,k-1}}{\theta\ell_{cs}^2}\|x_k-x^*\|_c^2+\frac{32L\alpha_{k-t_k,k-1}}{\theta\ell_{cs}^2}(A\|x^*\|_c+B)^2\\
		&\leq \frac{64LA^2u_{cm}^2\alpha_{k-t_k,k-1}}{\theta\ell_{cs}^2}M(x_k-x^*)+\frac{32L\alpha_{k-t_k,k-1}}{\theta\ell_{cs}^2}(A\|x^*\|_c+B)^2.
	\end{align*}
	\item Consider the term $T_{33}$. We first take expectation conditioning on $x_{k-t_k}$ and $Y_{k-t_k}$ to obtain
	\begin{align*}
		&\mathbb{E}[T_{33}\mid x_{k-t_k},Y_{k-t_k}]\\
		=\;&\langle \nabla M(x_{k-t_k}-x^*),\mathbb{E}[F(x_{k-t_k},Y_k)\mid x_{k-t_k},Y_{k-t_k}]-\bar{F}(x_{k-t_k})\rangle\\
		\leq\;& \|\nabla M(x_{k-t_k}-x^*)\|_s^*\|\mathbb{E}[F(x_{k-t_k},Y_k)\mid x_{k-t_k},Y_{k-t_k}]-\bar{F}(x_{k-t_k})\|_s\\
		\leq\;& \frac{1}{\ell_{cs}}\|\nabla M(x_{k-t_k}-x^*)\|_s^*\|\mathbb{E}[F(x_{k-t_k},Y_k)\mid x_{k-t_k},Y_{k-t_k}]-\bar{F}(x_{k-t_k})\|_c.
	\end{align*}
	For the term $\|\nabla M(x_{k-t_k}-x^*)\|_s^*$, we have from Eq. (\ref{eq:3}) that
	\begin{align*}
		\|\nabla M(x_{k-t_k}-x^*)\|_s^*\leq \frac{2L}{\theta\ell_{cs}}\left(\|x_k-x^*\|_c+\|x^*\|_c+\frac{B}{A}\right).
	\end{align*}
	For the term $\|\mathbb{E}[F(x_{k-t_k},Y_k)\mid x_{k-t_k},Y_{k-t_k}]-\bar{F}(x_{k-t_k})\|_c$, using the geometric mixing of the Markov chain $\{Y_k\}$ (cf. Assumption \ref{as:MC}), we have
	\begin{align*}
		&\|\mathbb{E}[F(x_{k-t_k},Y_k)\mid x_{k-t_k},Y_{k-t_k}]-\bar{F}(x_{k-t_k})\|_c\\
		=\;&\left\|\mathbb{E}[F(x_{k-t_k},Y_k)\mid x_{k-t_k},Y_{k-t_k}]-\mathbb{E}_{Y\sim\mu}[F(x_{k-t_k},Y)]\right\|_c\\
		=\;&\left\|\sum_{y\in\mathcal{Y}}\left(P^{t_k}(Y_{k-t_k},y)-\mu(y)\right)F(x_{k-t_k},y)\right\|_c\\
		\leq \;&\sum_{y\in\mathcal{Y}}\left|P^{t_k}(Y_{k-t_k},y)-\mu(y)\right|\left\|F(x_{k-t_k},y)\right\|_c\\
		\leq \;&2\max_{y_0\in\mathcal{Y}}\|P^{t_k}(y_0,\cdot)-\mu(\cdot)\|_{\text{TV}}(A_1\|x_{k-t_k}\|_c+B_1)\\
		\leq \;&2C\sigma^{t_k}(A_1\|x_k-x_{k-t_k}\|_c+A_1\|x_k\|_c+B_1)\tag{Assumption \ref{as:MC}}\\
		\leq \;&2\alpha_k(A_1(\|x_{k}\|_c+B/A)+A_1\|x_k\|_c+B_1)\tag{Definition of $t_k$ and Corollary \ref{co:difference}}\\
		\leq \;&4\alpha_k(A\|x_k-x^*\|_c+A\|x^*\|_c+B).
	\end{align*}
	It follows that
	\begin{align*}
		\mathbb{E}[T_{33}\mid x_{k-t_k},Y_{k-t_k}]&\leq \frac{1}{\ell_{cs}}\|\nabla M(x_{k-t_k}-x^*)\|_s^*\|\mathbb{E}[F(x_{k-t_k},Y_k)\mid x_{k-t_k},Y_{k-t_k}]-\bar{F}(x_{k-t_k})\|_c\\
		&\leq \frac{8L\alpha_k}{\theta\ell_{cs}^2}(A\|x_k-x^*\|_c+A\|x^*\|_c+B)^2\\
		&\leq \frac{16L\alpha_k}{\theta\ell_{cs}^2}A^2\|x_k-x^*\|_c^2+\frac{16L\alpha_k}{\theta\ell_{cs}^2}(A\|x^*\|_c+B)^2\\
		&\leq \frac{32LA^2u_{cm}^2\alpha_k}{\theta\ell_{cs}^2}M(x_k-x^*)+\frac{16L\alpha_k}{\theta\ell_{cs}^2}(A\|x^*\|_c+B)^2.
	\end{align*}
	Taking the total expectation on both sides of the previous inequality yields the desired result.
\end{enumerate}

\subsubsection{Proof of Lemma \ref{le:T4}}\label{pf:le:T4}
Using Proposition \ref{prop:Moreau} (2), Assumption \ref{as:F}, and Assumption \ref{as:noise_w} (2), we have
\begin{align*}
	T_4&=\frac{L\alpha_k^2}{2\theta}\|F(x_k,Y_k)-x_k+w_k\|_s^2\\
	&\leq \frac{L\alpha_k^2}{2\theta\ell_{cs}^2}\|F(x_k,Y_k)-x_k+w_k\|_c^2\tag{Proposition \ref{prop:Moreau} (3)}\\
	&\leq \frac{L\alpha_k^2}{2\theta\ell_{cs}^2}(\|F(x_k,Y_k)\|_c+\|x_k\|_c+\|w_k\|_c)^2\\
	&\leq \frac{L\alpha_k^2}{2\theta\ell_{cs}^2}(A\|x_k\|_c+B)^2\tag{Assumptions \ref{as:F} and \ref{as:noise_w}}\\
	&\leq \frac{L\alpha_k^2}{2\theta\ell_{cs}^2}(A\|x_k-x^*\|_c+A\|x^*\|_c+B)^2\\
	&\leq \frac{L\alpha_k^2}{\theta\ell_{cs}^2}A^2\|x_k-x^*\|_c^2+\frac{L\alpha_k^2}{\theta\ell_{cs}^2}(A\|x^*\|_c+B)^2\\
	&\leq \frac{2LA^2u_{cm}^2\alpha_k^2}{\theta\ell_{cs}^2}M(x_k-x^*)+\frac{L\alpha_k^2}{\theta\ell_{cs}^2}(A\|x^*\|_c+B)^2.
\end{align*}

\subsubsection{Proof of Lemma \ref{le:recursion}}\label{pf:le:recursion}
Using the constants $\{\varphi_i\}_{1\leq i\leq }$ (cf. Eq. (\ref{eq:def:constants})) and Lemmas \ref{le:T1}, \ref{le:T3}, and \ref{le:T4} in Eq. (\ref{eq:composition1}) and we have for all $k\geq t_k$:
\begin{align*}
	\mathbb{E}[M(x_{k+1}-x^*)]\leq\;& \left(1-2\varphi_2\alpha_k+\frac{114LA^2u_{cm}^2\alpha_k\alpha_{k-t_k,k-1}}{\theta\ell_{cs}^2}\right)\mathbb{E}[M(x_k-x^*)]\\
	&+\frac{57L\alpha_k\alpha_{k-t_k,k-1}}{\theta\ell_{cs}^2}(A\|x^*\|_c+B)^2\\
	=\;&\left(1-2\varphi_2\alpha_k+\varphi_3A^2\alpha_k\alpha_{k-t_k,k-1}\right)\mathbb{E}[M(x_k-x^*)]+\frac{\varphi_3c_2\alpha_k\alpha_{k-t_k,k-1}}{2u_{cm}^2}.
\end{align*}

\subsubsection{Proof of Lemma \ref{le:rate_linear_diminishing}}\label{pf:le:rate_linear_diminishing}

We first simplify the RHS of Eq. (\ref{eq:final}) using $\alpha_k=\frac{\alpha}{k+h}$. Since we have chosen $h$ such that $\alpha_{k-t_k,k-1}\leq 2\alpha_k$ for any $k\geq t_k$, Eq. (\ref{eq:final}) implies
\begin{align}
	\mathbb{E}[\|x_k-x^*\|_c^2]
	\leq\;& \varphi_1c_1\prod_{j=K}^{k-1}(1-\varphi_2\alpha_j)+\varphi_3c_2\sum_{i=K}^{k-1}\alpha_i\alpha_{i-t_i,i-1}\prod_{j=i+1}^{k-1}(1-\varphi_2\alpha_j)\nonumber\\
	\leq\;& \varphi_1c_1\prod_{j=K}^{k-1}(1-\varphi_2\alpha_j)+2\varphi_3c_2\sum_{i=K}^{k-1}\alpha_i^2t_i\prod_{j=i+1}^{k-1}(1-\varphi_2\alpha_j)\nonumber\\
	=\;& \varphi_1c_1\underbrace{\prod_{j=K}^{k-1}\left(1-\frac{\varphi_2\alpha}{j+h}\right)}_{E_1}+2\varphi_3c_2t_k\underbrace{\sum_{i=K}^{k-1}\frac{\alpha^2}{(i+h)^2}\prod_{j=i+1}^{k-1}\left(1-\frac{\varphi_2\alpha}{j+h}\right)}_{E_2}\label{eq:mother}
\end{align}
For the term $E_1$, we have
\begin{align*}
	E_1\leq \exp\left(-\varphi_2\alpha\sum_{j=K}^{k-1}\frac{1}{j+h}\right)\leq \exp\left(-\varphi_2\alpha\int_{K}^{k}\frac{1}{x+h}dx\right)=\left(\frac{K+h}{k+h}\right)^{\varphi_2\alpha}.
\end{align*}
Now consider the term $E_2$. Similarly we have
\begin{align*}
	E_2&=\sum_{i=K}^{k-1}\frac{\alpha^2}{(i+h)^2}\prod_{j=i+1}^{k-1}\left(1-\frac{\varphi_2\alpha}{j+h}\right)\\
	&\leq \sum_{i=K}^{k-1}\frac{\alpha^2}{(i+h)^2}\left(\frac{i+1+h}{k+h}\right)^{\varphi_2\alpha}\\
	&\leq \frac{4\alpha^2}{(k+h)^{\varphi_2\alpha}}\sum_{i=K}^{k-1}\frac{1}{(i+1+h)^{2-\varphi_2\alpha}}\\
	&\leq \begin{dcases}
		\frac{4\alpha^2}{1-\varphi_2\alpha}\frac{1}{(k+h)^{\varphi_2\alpha}},&\varphi_2\alpha\in (0,1),\\
		\frac{4\alpha^2\log(k+h)}{k+h},&\varphi_2\alpha=1,\\
		\frac{4e\alpha^2}{\varphi_2\alpha-1}\frac{1}{k+h},&\varphi_2\alpha\in (1,\infty).
	\end{dcases}
\end{align*}
The result then follows from using the upper bounds we obtained for the terms $E_1$ and $E_2$ in inequality (\ref{eq:mother}).

\subsubsection{Proof of Lemma \ref{le:rate_polynomial_diminishing}}\label{pf:le:rate_polynomial_diminishing}
When $\alpha_k=\frac{\alpha}{(k+h)^\xi}$, similarly we have from Eq. (\ref{eq:final}) that
\begin{align}
	\mathbb{E}[\|x_k-x^*\|_c^2]
	\leq \varphi_1c_1\underbrace{\prod_{j=K}^{k-1}\left(1-\frac{\varphi_2\alpha}{(j+h)^\xi}\right)}_{E_1}+2\varphi_3c_2t_k\underbrace{\sum_{i=K}^{k-1}\frac{\alpha^2}{(i+h)^{2\xi}}\prod_{j=i+1}^{k-1}\left(1-\frac{\varphi_2\alpha}{(j+h)^\xi}\right)}_{E_2}\label{eq:mother1}
\end{align}
The term $E_1$ can be controlled in the following way:
\begin{align*}
	E_1&=\prod_{j=K}^{k-1}\left(1-\frac{\varphi_2\alpha}{(j+h)^\xi}\right)\\
	&\leq \exp\left(-\varphi_2\alpha\sum_{j=K}^{k-1}\frac{1}{(j+h)^\xi}\right)\\
	&\leq \exp\left(-\varphi_2\alpha\int_{K}^{k}\frac{1}{(x+h)^\xi}dx\right)\\
	&=\exp\left[-\frac{\varphi_2\alpha}{1-\xi}\left((k+h)^{1-\xi}-(K+h)^{1-\xi}\right)\right].
\end{align*}

As for the term $E_2$, we will show by induction that $E_2\leq \frac{2\alpha}{\varphi_2}\frac{1}{(k+h)^\xi}$ for all $k\geq 0$. Consider a sequence $\{u_k\}_{k\geq 0}$ (with $u_0=0$) defined by
\begin{align*}
	u_{k+1}=\left(1-\varphi_2\frac{\alpha}{(k+h)^\xi}\right)u_k+\frac{\alpha^2}{(k+h)^{2\xi}},\quad \forall \;k\geq 0.
\end{align*}
It can be easily verified that $u_k=E_2$. Since $u_{0}=0\leq \frac{2\alpha}{\varphi_2}\frac{1}{h^\xi}$, we have the base case. Now suppose $u_k\leq \frac{2\alpha}{\varphi_2}\frac{1}{(k+h)^\xi}$ for some $k> 0$. Consider $u_{k+1}$, and we have
\begin{align*}
	\frac{2\alpha}{\varphi_2}\frac{1}{(k+1+h)^\xi}-u_{k+1}
	=&\frac{2\alpha}{\varphi_2}\frac{1}{(k+1+h)^\xi}-\left(1-\varphi_2\frac{\alpha}{(k+h)^\xi}\right)u_k+\frac{\alpha^2}{(k+h)^{2\xi}}\\
	\geq& \frac{2\alpha}{\varphi_2}\frac{1}{(k+1+h)^\xi}-\left(1-\frac{\varphi_2\alpha}{(k+h)^\xi}\right)\frac{2\alpha}{\varphi_2}\frac{1}{(k+h)^\xi}-\frac{\alpha^2}{(k+h)^{2\xi}}\\
	=&\frac{2\alpha}{\varphi_2}\left[\frac{1}{(k+1+h)^\xi}-\frac{1}{(k+h)^\xi}+\frac{\varphi_2\alpha}{2}\frac{1}{(k+h)^{2\xi}}\right]\\
	=&\frac{2\alpha}{\varphi_2}\frac{1}{(k+h)^{2\xi}}\left[\frac{\varphi_2\alpha}{2}-(k+h)^\xi\left(1-\left(\frac{k+h}{k+1+h}\right)^{\xi}\right)\right].
\end{align*}
Note that
\begin{align*}
	\left(\frac{k+h}{k+1+h}\right)^{\xi}=\left[\left(1+\frac{1}{k+h}\right)^{k+h}\right]^{-\frac{\xi}{k+h}}
	\geq \exp\left(-\frac{\xi}{k+h}\right)\geq 1-\frac{\xi}{k+h},
\end{align*}
where we used $(1+\frac{1}{x})^x<e$ for all $x>0$ and
$e^x\geq 1+x$ for all $x\in\mathbb{R}$. Therefore, we obtain
\begin{align*}
	\frac{2\alpha}{\varphi_2}\frac{1}{(k+1+h)^\xi}-u_{k+1}&\geq \frac{2\alpha}{\varphi_2}\frac{1}{(k+h)^{2\xi}}\left[\frac{\varphi_2\alpha}{2}-(k+h)^\xi\left(1-\left(\frac{k+h}{k+1+h}\right)^{\xi}\right)\right]\\
	&\geq \frac{2\alpha}{\varphi_2}\frac{1}{(k+h)^{2\xi}}\left[\frac{\varphi_2\alpha}{2}-\frac{\xi}{(k+h)^{1-\xi}}\right]\\
	&\geq 0,
\end{align*}
where the last line follows from $h\geq\bar{h}\geq  \left[2\xi/(\varphi_2\alpha)\right]^{1/(1-\xi)} $. The induction is now complete, and we have $E_2\leq \frac{2\alpha}{\varphi_2}\frac{1}{(k+h)^{\xi}}$ for all $k\geq 0$. Using the upper bounds we obtained for the terms $E_1$ and $E_2$ in inequality (\ref{eq:mother1}) and we have the desired result.

\section{Q-Learning}\label{ap:Q}
\subsection{Proof of Proposition \ref{prop:Q-learning}}\label{pf:prop:Q}
\begin{enumerate}[(1)]
	\item For any $Q_1,Q_2\in\mathbb{R}^{|\mathcal{S}||\mathcal{A}|}$ and $y\in\mathcal{Y}$, we have 
	\begin{align*}
		\|F(Q_1,y)-F(Q_2,y)\|_\infty
		\leq \;&\max_{(s,a)}\left|\gamma\mathbbm{1}_{\{(s_0,a_0)=(s,a)\}}(\max_{a_1\in\mathcal{A}}Q_1(s_1,a_1)-\max_{a_2\in\mathcal{A}}Q_2(s_1,a_2))\right|\\
		&+\max_{s,a}\left|\mathbbm{1}_{\{(s_0,a_0)\neq (s,a)\}}(Q_1(s_0,a_0)-Q_2(s_0,a_0))\right|\\
		\leq \;&2\|Q_1-Q_2\|_\infty.
	\end{align*}
	Similarly, for any $y\in\mathcal{Y}$, we have
	\begin{align*}
		\|F(\bm{0},y)\|_\infty=\max_{(s,a)}\left|\mathbbm{1}_{\{(s_0,a_0)=(s,a)\}}\mathcal{R}(s_0,a_0)\right|\leq 1.
	\end{align*}
	\item It is clear from Assumption \ref{as:Q} that $\{Y_k\}$ has a unique stationary distribution, denoted by $\mu$. Moreover, we have $\mu(s,a,s')=\kappa_b(s)\pi_b(a|s)P_a(s,s')$ for any $(s,a,s')\in\mathcal{Y}$. Consider the second claim. Using the definition of total variation distance, we have for all $k\geq 0$:
	\begin{align*}
		\max_{y\in\mathcal{Y}}\|P^{k+1}(y,\cdot)-\mu(\cdot)\|_{\text{TV}}
		=\;&\frac{1}{2}\max_{(s_0,a_0,s_1)\in\mathcal{Y}}\sum_{s,a,s'}|P^{k+1}_{\pi_b}((s_0,a_0,s_1),(s,a,s'))-\kappa_b(s)\pi_b(a|s)P_a(s,s')|\\
		=\;&\frac{1}{2}\max_{s_1\in\mathcal{S}}\sum_{s,a,s'}|P_{\pi_b}^{k}(s_1,s)\pi_b(a|s)P_a(s,s')-\kappa_b(s)\pi_b(a|s)P_a(s,s')|\\
		=\;&\frac{1}{2}\max_{s_1\in\mathcal{S}}\sum_{s}|P_{\pi_b}^{k}(s_1,s)-\kappa_b(s)|\\
		=\;&\max_{s\in\mathcal{S}}\|P^{k}_{\pi_b}(s,\cdot)-\kappa_b(\cdot)\|_{\text{TV}}\\
		\leq \;& C_1\sigma_1^k,
	\end{align*}
	where $C_1>0$ and $\sigma_1\in(0,1)$ are constants. Note that the last line of the previous inequality follows from Assumption \ref{as:Q}.
	\item 
	\begin{enumerate}
		\item Using the Markov property, we have for any $Q\in\mathbb{R}^{|\mathcal{S}||\mathcal{A}|}$ and $(s,a)$:
		\begin{align*}
			&\mathbb{E}_{S_k\sim \kappa_b}\left[[F(Q,S_k,A_k,S_{k+1})](s,a)\right]\\
			=\;&\mathbb{E}_{S_k\sim \kappa_b}\left[\mathbbm{1}_{\{(S_k,A_k)=(s,a)\}}\left(\mathcal{R}(S_k,A_k)+\gamma\max_{a'\in\mathcal{A}}Q(S_{k+1},a')-Q(S_k,A_k)\right)+Q(s,a)\right]\\
			=\;&\mathbb{E}_{S_k\sim \kappa_b}\left[\mathbbm{1}_{\{(S_k,A_k)=(s,a)\}}\left(\mathcal{R}(S_k,A_k)+\gamma\max_{a'\in\mathcal{A}}Q(S_{k+1},a')\right)+(1-\mathbbm{1}_{\{(S_k,A_k)=(s,a)\}})Q(S_k,A_k)\right]\\
			=\;&\kappa_b(s)\pi_b(a|s)[\mathcal{H}(Q)](s,a)+(1-\kappa_b(s)\pi_b(a|s))Q(s,a),
		\end{align*}
		where $\mathcal{H}:\mathbb{R}^{|\mathcal{S}||\mathcal{A}|}\mapsto\mathbb{R}^{|\mathcal{S}||\mathcal{A}|}$ is the Bellman's optimality operator defined by 
		\begin{align*}
			[\mathcal{H}(Q)](s,a)=\mathbb{E}[\mathcal{R}(S_k,A_k)+\gamma\max_{a'\in\mathcal{A}}Q(S_{k+1},a')\mid S_k=s,A_k=a]
		\end{align*}
		for any $(s,a)$. Now use the definition of the matrix $N$ and we have $\bar{F}(Q)=N\mathcal{H}(Q)+(I-N)Q$.
		\item Since it is well-known that the Bellman's optimality operator $\mathcal{H}(\cdot)$ is a $\gamma$-contraction with respect to $\|\cdot\|_\infty$, we have for any $Q_1,Q_2\in\mathbb{R}^{|\mathcal{S}||\mathcal{A}|}$:
		\begin{align*}
			&\|\bar{F}(Q_1)-\bar{F}(Q_2)\|_\infty\\
			=\;&\|N(\mathcal{H}(Q_1)-\mathcal{H}(Q_2))+(I-N)(Q_1-Q_2)\|_\infty\\
			=\;&\max_{(s,a)}|\kappa_b(s)\pi_b(a|s)([\mathcal{H}(Q_1)](s,a)-[\mathcal{H}(Q_2)](s,a))+(1-\kappa_b(s)\pi_b(a|s))(Q_1(s,a)-Q_2(s,a))|\\
			\leq\;& \max_{(s,a)}\left[\kappa_b(s)\pi_b(a|s)|[\mathcal{H}(Q_1)](s,a)-[\mathcal{H}(Q_2)](s,a)|+(1-\kappa_b(s)\pi_b(a|s))|Q_1(s,a)-Q_2(s,a)|\right]\\
			\leq\;& \max_{(s,a)}\left[\kappa_b(s)\pi_b(a|s)\|\mathcal{H}(Q_1)-\mathcal{H}(Q_2)\|_\infty+(1-\kappa_b(s)\pi_b(a|s))\|Q_1-Q_2\|_\infty\right]\\
			\leq \;&\max_{(s,a)}\left[\kappa_b(s)\pi_b(a|s)\gamma\|Q_1-Q_2\|_\infty+(1-\kappa_b(s)\pi_b(a|s))\|Q_1-Q_2\|_\infty\right]\\
			=\;&\|Q_1-Q_2\|_\infty\max_{(s,a)}(1-(1-\gamma)\kappa_b(s)\pi_b(a|s))\\
			=\;&(1-N_{\min}(1-\gamma))\|Q_1-Q_2\|_\infty.
		\end{align*}
		Therefore, the operator $\bar{F}(\cdot)$ is a contraction mapping with respect to $\|\cdot\|_\infty$, with contraction factor $\beta_1=1-N_{\min}(1-\gamma)$.
		\item It is enough to show that $Q^*$ is a fixed-point of $\bar{F}(\cdot)$, the uniqueness part follows from $\bar{F}(\cdot)$ being a contraction \citep{banach1922operations}. Using the fact that $\mathcal{H}(Q^*)=Q^*$, we have 
		\begin{align*}
			\bar{F}(Q^*)=N\mathcal{H}(Q^*)+(I-N)Q^*=NQ^*+(I-N)Q^*=Q^*.
		\end{align*}
	\end{enumerate}
\end{enumerate}

\subsection{Proof of Theorem \ref{thm:Q}}\label{pf:thm:Q}
Since the contraction norm is $\|\cdot\|_\infty$, Lemma \ref{le:constants} (2) is applicable. To apply Theorem \ref{thm:sa}, we first identify the corresponding constants using Proposition \ref{prop:Q-learning} in the following:
\begin{align*}
	A&=A_1+A_2+1=3,\;B=B_1+B_2=1,\;\varphi_1\leq 3,\;\varphi_2\geq \frac{1-\beta_1}{2},\;\varphi_3\leq \frac{456e\log(|\mathcal{S}||\mathcal{A}|)}{1-\beta_1},\\
	c_1&\leq  (\|Q_0-Q^*\|_\infty+\|Q_0\|_\infty+1)^2,\;c_2=(3\|Q^*\|_\infty+1)^2.
\end{align*}
Now we apply Theorem \ref{thm:sa} (2) (a). When $\alpha_k\equiv\alpha$ with $\alpha$ chosen such that 
\begin{align*}
	\alpha t_\alpha(\mathcal{M}_Y)\leq \frac{\varphi_2}{\varphi_3A^2} \frac{(1-\beta_1)^2}{8208e\log(|\mathcal{S}||\mathcal{A}|)}.
\end{align*}
we have for all $k\geq t_\alpha(\mathcal{M}_Y)$:
\begin{align*}
	\mathbb{E}[\|Q_k-Q^*\|_\infty^2]
	\leq\;& \varphi_1c_1\left(1-\frac{1-\beta_1}{2}\alpha\right)^{k-t_\alpha(\mathcal{M}_Y)}+\frac{\varphi_3c_2}{\varphi_2}\alpha t_{\alpha}(\mathcal{M}_Y)\\
	\leq \;&3(\|Q_0-Q^*\|_\infty+\|Q_0\|_\infty+1)^2\left(1-\frac{1-\beta_1}{2}\alpha\right)^{k-t_\alpha(\mathcal{M}_Y)}\\
	&+\frac{912e\log(|\mathcal{S}||\mathcal{A}|)}{(1-\beta_1)^2}(3\|Q^*\|_\infty+1)^2\alpha t_\alpha(\mathcal{M}_Y)\\
	=\;&c_{Q,1}\left(1-\frac{1-\beta_1}{2}\alpha\right)^{k-t_\alpha(\mathcal{M}_Y)}+c_{Q,2}\frac{\log(|\mathcal{S}||\mathcal{A}|)}{(1-\beta_1)^2}\alpha t_\alpha(\mathcal{M}_Y),
\end{align*}
where $c_{Q,1}=3(\|Q_0-Q^*\|_\infty+\|Q_0\|_\infty+1)^2$ and $c_{Q,2}=912e(3\|Q^*\|_\infty+1)^2$.

\subsection{Q-Learning with Diminishing Stepsizes}
We next present the finite-sample bounds for $Q$-learning with diminishing stepsizes, whose proof follows by directly applying Theorem \ref{thm:sa} (2) (b) and (2) (c) and hence is omitted.

\begin{theorem}\label{thm:Q1}
	Consider $\{Q_k\}$ of Algorithm (\ref{eq:Q-learning}). Suppose that Assumption \ref{as:Q} is satisfied, then we have the following results.
	\begin{enumerate}[(1)]
		\item 
		\begin{enumerate}[(a)]
			\item When $\alpha_k=\frac{\alpha}{k+h}$ with $\alpha=\frac{1}{1-\beta_1}$ and properly chosen $h$, there exists $K_1'>0$ such that the following inequality holds for all $k\geq K_1'$:
			\begin{align*}
				\mathbb{E}[\|Q_k-Q^*\|_\infty^2]\leq c_{Q,1}'\left(\frac{K_1'+h}{k+h}\right)^{1/2}+2c_{Q,2}'\frac{\log(|\mathcal{S}||\mathcal{A})}{(1-\beta_1)^3}\frac{t_k(\mathcal{M}_Y)}{k+h},
			\end{align*}
			where $c_{Q,1}'=3(\|Q_0-Q^*\|_\infty+\|Q_0\|_\infty+1)^2$ and $c_{Q,2}'=3648e(3\|Q^*\|_\infty+1)^2$
			\item When $\alpha_k=\frac{\alpha}{k+h}$ with $\alpha=\frac{2}{1-\beta_1}$ and properly chosen $h$, there exists $K_1'>0$ such that the following inequality holds for all $k\geq K_1'$:
			\begin{align*}
				\mathbb{E}[\|Q_k-Q^*\|_\infty^2]\leq c_{Q,1}'\frac{K_1'+h}{k+h}+4c_{Q,2}'\frac{\log(|\mathcal{S}||\mathcal{A})}{(1-\beta_1)^3}\frac{t_k(\mathcal{M}_Y)\log(k+h)}{k+h}.
			\end{align*}
			\item When $\alpha_k=\frac{\alpha}{k+h}$ with $\alpha=\frac{4}{1-\beta_1}$ and properly chosen $h$, there exists $K_1'>0$ such that the following inequality holds for all $k\geq K_1'$:
			\begin{align*}
				\mathbb{E}[\|Q_k-Q^*\|_\infty^2]\leq c_{Q,1}'\left(\frac{K_1'+h}{k+h}\right)^2+16c_{Q,2}'\frac{\log(|\mathcal{S}||\mathcal{A})}{(1-\beta_1)^3}\frac{t_k(\mathcal{M}_Y)}{k+h}.
			\end{align*}
		\end{enumerate}
		\item When $\alpha_k=\frac{\alpha}{(k+h)^\xi}$ with $\xi\in (0,1)$, $\alpha>0$, and properly chosen $h$, there exists $K_1'>0$ such that the following inequality holds for all $k\geq K_1'$:
		\begin{align*}
			\mathbb{E}[\|Q_k-Q^*\|_\infty^2]\leq c_{Q,1}'\exp\left(-\frac{(1-\beta_1)\alpha}{2(1-\xi)}((k+h)^{1-\xi}-(K_1'+h)^{1-\xi})\right)+c_{Q,2}'\frac{\log(|\mathcal{S}||\mathcal{A})}{(1-\beta_1)^2}\frac{t_k(\mathcal{M}_Y)}{k+h}.
		\end{align*}
	\end{enumerate}
\end{theorem}

\subsection{Proof of Corollary \ref{co:Q}}\label{pf:co:Q}
We will derive a more general result, which implies Corollary \ref{co:Q}. Suppose we have a non-negative sequence $\{z_k\}$ and the following bound:
\begin{align*}
	z_k\leq (1-\tau_1 \alpha)^kz_0+\tau_2\alpha t_\alpha,
\end{align*}
where $\tau_1\in (0,1)$, $\tau_2>0$, and $t_\alpha\leq L(\log(1/\alpha)+1)$ for some $L>0$. Then, in order for $z_k\leq \epsilon^2$, in view of the term $\tau_2\alpha t_\alpha$, we need
\begin{align*}
	\alpha=\mathcal{O}\left(\frac{\epsilon^2}{\log(\frac{1}{\epsilon})}\right)\tilde{\mathcal{O}}\left(\frac{1}{\tau_2}\right).
\end{align*}
Using the bound of $\alpha$ in the term $(1-\tau_1 \alpha)^kz_0$, we have
\begin{align*}
	k=\mathcal{O}\left(\frac{\log^2(\frac{1}{\epsilon})}{\epsilon^2}\right)\tilde{\mathcal{O}}\left(\frac{\tau_1}{\tau_2}\right).
\end{align*}
Using this result along with Jensen's inequality in the finite-sample bound of $Q$-learning proves Corollary \ref{co:Q}.

\section{V-Trace}\label{ap:V-trace}
\subsection{Proof of Proposition \ref{prop:Vtrace}}
\begin{enumerate}[(1)]
	\item 
	Using the definition of $F(V,y)$, we have for any $V_1,V_2\in\mathbb{R}^{|\mathcal{S}|}$, $y\in\mathcal{Y}$, and $s\in\mathcal{S}$:
	\begin{align*}
		&|[F(V_1,y)](s)-[F(V_2,y)](s)|\\
		=\;&\left|\mathbbm{1}_{\{s_0=s\}}\sum_{i=0}^{n-1}\gamma^{i}\left(\prod_{j=0}^{i-1}c(s_j,a_j)\right)\rho(s_i,a_i)\left(\gamma (V_1(s_{i+1})-V_2(s_{i+1}))-(V_1(s_i)-V_2(s_i))\right)+V_1(s)-V_2(s)\right|\\
		\leq \;&2\|V_1-V_2\|_\infty\sum_{i=0}^{n-1}\gamma^{i}\bar{c}^i\bar{\rho}+\|V_1-V_2\|_\infty\\
		\leq \;&\begin{dcases}
			\frac{(2\bar{\rho}+1)(1-(\gamma\bar{c})^n)}{1-\gamma\bar{c}}\|V_1-V_2\|_\infty,&\gamma\bar{c}\neq 1,\\
			(2\bar{\rho}+1)n\|V_1-V_2\|_\infty,&\gamma\bar{c}=1.
		\end{dcases}
	\end{align*}
	It follows that $\|F(V_1,y)-F(V_2,y)\|_\infty\leq (2\bar{\rho}+1)\eta(\gamma,\bar{c})\|V_1-V_2\|_\infty$.
	
	For any $y\in\mathcal{Y}$ and $s\in\mathcal{S}$, we have
	\begin{align*}
		|[F(\bm{0},y)](s)|&=\left|\mathbbm{1}_{\{s_0=s\}}\sum_{i=0}^{n-1}\gamma^{i}\left(\prod_{j=0}^{i-1}c(s_j,a_j)\right)\rho(s_i,a_i)\mathcal{R}(s_i,a_i)\right|\\
		&\leq \sum_{i=0}^{n-1}\gamma^{i}\bar{c}^i\bar{\rho}\\
		&= \begin{dcases}
			\frac{\bar{\rho}(1-(\gamma\bar{c})^n)}{1-\gamma\bar{c}},&\gamma\bar{c}\neq 1,\\
			n\bar{\rho},&\gamma\bar{c}=1.
		\end{dcases}
	\end{align*}
	It follows that $\|F(\bm{0},y)\|_\infty\leq \eta(\gamma,\bar{c})$.
	\item The proof is identical to that of Proposition \ref{prop:Q-learning} (2).
	\item 
	\begin{enumerate}
		\item Using the definition of $\bar{F}(\cdot)$, we have for any $V\in\mathbb{R}^{|\mathcal{S}|}$ and $s\in\mathcal{S}$:
		\begin{align*}
			&\mathbb{E}_{Y \sim \mu}[[F(V,Y)](s)]\\
			=\;&\mathbb{E}_{S_0\sim \kappa}\left[\mathbbm{1}_{\{S_0=s\}}\sum_{i=0}^{n-1}\gamma^{i}\left(\prod_{j=0}^{i-1}c(S_j,A_j)\right)\rho(S_i,A_i)\left(\mathcal{R}(S_i,A_i)+\gamma V(S_{i+1})-V(S_i)\right)\right]+V(s).
		\end{align*}
		For any $0\leq i\leq n-1$, we have by the Markov property and the tower property of conditional expectation that
		\begin{align*}
			&\mathbb{E}_{S_0\sim \kappa}\left[\mathbbm{1}_{\{S_0=s\}}\gamma^{i}\left(\prod_{j=0}^{i-1}c(S_j,A_j)\right)\rho(S_i,A_i)\left(\mathcal{R}(S_i,A_i)+\gamma V(S_{i+1})-V(S_i)\right)\right]\\
			=\;&\mathbb{E}_{S_0\sim \kappa}\left[\mathbbm{1}_{\{S_0=s\}}\gamma^{i}\left(\prod_{j=0}^{i-1}c(S_j,A_j)\right)\mathbb{E}[\rho(S_i,A_i)\left(\mathcal{R}(S_i,A_i)+\gamma V(S_{i+1})-V(S_i)\right)\mid S_i]\right]\\
			=\;&\mathbb{E}_{S_0\sim \kappa}\left[\mathbbm{1}_{\{S_0=s\}}\gamma^{i}\left(\prod_{j=0}^{i-1}c(S_j,A_j)\right)[D(R_{\pi_{\bar{\rho}}}+\gamma P_{\pi_{\bar{\rho}}}V-V)](S_i)\right]\\
			=\;&\mathbb{E}_{S_0\sim \kappa}\left[\mathbbm{1}_{\{S_0=s\}}\gamma^{i}\left(\prod_{j=0}^{i-2}c(S_j,A_j)\right)\mathbb{E}[c(S_{i-1},A_{i-1})[D(R_{\pi_{\bar{\rho}}}+\gamma P_{\pi_{\bar{\rho}}}V-V)](S_i)\mid S_{i-1}]\right]\\
			=\;&\mathbb{E}_{S_0\sim \kappa}\left[\mathbbm{1}_{\{S_0=s\}}\gamma^{i}\left(\prod_{j=0}^{i-2}c(S_j,A_j)\right)[(CP_{\pi_{\bar{c}}})D(R_{\pi_{\bar{\rho}}}+\gamma P_{\pi_{\bar{\rho}}}V-V)](S_{i-1})\right]\\
			=\;&\cdots\\
			=\;&\mathbb{E}_{S_0\sim \kappa}\left[\mathbbm{1}_{\{S_0=s\}}\gamma^{i}[(CP_{\pi_{\bar{c}}})^iD(R_{\pi_{\bar{\rho}}}+\gamma P_{\pi_{\bar{\rho}}}V-V)](S_0)\right]\\
			=\;&[\mathcal{K}(\gamma CP_{\pi_{\bar{c}}})^iD(R_{\pi_{\bar{\rho}}}+\gamma P_{\pi_{\bar{\rho}}}V-V)](s).
		\end{align*}
		Therefore, we have
		\begin{align}
			\bar{F}(V)&=\sum_{i=0}^{n-1}\mathcal{K}(\gamma CP_{\pi_{\bar{c}}})^iD(R_{\pi_{\bar{\rho}}}+\gamma P_{\pi_{\bar{\rho}}}V-V)+V\label{eq:11}\\
			&=\left[I-\mathcal{K}\sum_{i=0}^{n-1}(\gamma CP_{\pi_{\bar{c}}})^iD(I-\gamma P_{\pi_{\bar{\rho}}})\right]V+\sum_{i=0}^{n-1}\mathcal{K}(\gamma CP_{\pi_{\bar{c}}})^iDR_{\pi_{\bar{\rho}}}.\nonumber
		\end{align}
		\item For any $V_1,V_2\in\mathbb{R}^{|\mathcal{S}|}$, we have 
		\begin{align*}	\left\|\bar{F}(V_1)-\bar{F}(V_2)\right\|_\infty&=\left\|\left[I-\mathcal{K}\sum_{i=0}^{n-1}(\gamma CP_{\pi_{\bar{c}}})^iD(I-\gamma P_{\pi_{\bar{\rho}}})\right](V_1-V_2)\right\|_\infty\\
			&\leq \left\|I-\mathcal{K}\sum_{i=0}^{n-1}(\gamma CP_{\pi_{\bar{c}}})^iD(I-\gamma P_{\pi_{\bar{\rho}}})\right\|_\infty\|V_1-V_2\|_\infty
		\end{align*}
		For simplicity of notation, denote $G=I-\mathcal{K}\sum_{i=0}^{n-1}(\gamma CP_{\pi_{\bar{c}}})^iD(I-\gamma P_{\pi_{\bar{\rho}}})$. To evaluate the $\ell_\infty$-norm of $G$, we first show that $G$ has non-negative entries. Note that $G$ can be equivalently written by
		\begin{align}
			G&=I-\mathcal{K}\sum_{i=0}^{n-1}(\gamma CP_{\pi_{\bar{c}}})^iD(I-\gamma P_{\pi_{\bar{\rho}}})\nonumber\\
			&=I-\mathcal{K}D-\mathcal{K}\sum_{i=0}^{n-2}(\gamma CP_{\pi_{\bar{c}}})^{i+1}D+\mathcal{K}\sum_{i=0}^{n-2}(\gamma CP_{\pi_{\bar{c}}})^i\gamma D P_{\pi_{\bar{\rho}}}+\mathcal{K}(\gamma CP_{\pi_{\bar{c}}})^{n-1}\gamma D P_{\pi_{\bar{\rho}}}\nonumber\\
			&=I-\mathcal{K}D+\mathcal{K}\sum_{i=0}^{n-2}(\gamma CP_{\pi_{\bar{c}}})^i\gamma( D P_{\pi_{\bar{\rho}}}-C P_{\pi_{\bar{c}}} D)+\mathcal{K}(\gamma CP_{\pi_{\bar{c}}})^{n-1}\gamma D P_{\pi_{\bar{\rho}}}.\label{eq:vtrace1}
		\end{align}
		In view of Eq. (\ref{eq:vtrace1}), it remains to show that the matrix $D P_{\pi_{\bar{\rho}}}-C P_{\pi_{\bar{c}}} D$ has non-negative entries. For any $s,s'\in\mathcal{S}$, we have
		\begin{align}
			[D P_{\pi_{\bar{\rho}}}-C P_{\pi_{\bar{c}}} D](s,s')&=D(s)P_{\pi_{\bar{\rho}}}(s,s')-C(s)P_{\pi_{\bar{c}}}(s,s')D(s')\nonumber\\
			&\geq D(s)P_{\pi_{\bar{\rho}}}(s,s')-C(s)P_{\pi_{\bar{c}}}(s,s')\tag{$D(s')\leq 1$}\nonumber\\
			&=D(s)\sum_{a\in\mathcal{A}}\pi_{\bar{\rho}}(a|s)P_a(s,s')-C(s)\sum_{a\in\mathcal{A}}\pi_{\bar{c}}(a|s)P_a(s,s')\nonumber\\
			&=\sum_{a\in\mathcal{A}}\left[\min\left(\bar{\rho}\pi_b(a|s),\pi(a|s)\right)-\min\left(\bar{c}\pi_b(a|s),\pi(a|s)\right)\right]P_a(s,s')\nonumber\\
			&\geq 0,\label{eq:vtrace2}
		\end{align}
		where the last line follows from $\bar{c}\leq \bar{\rho}$.
		
		Now since the matrix $G$ has non-negative entries, we have
		\begin{align*}
			\|G\|_\infty&=\|G\bm{1}\|_\infty\tag{$\bm{1}=(1,1,\cdots,1)^\top$}\\
			&=\left\|\left[I-\mathcal{K}\sum_{i=0}^{n-1}(\gamma CP_{\pi_{\bar{c}}})^iD(I-\gamma P_{\pi_{\bar{\rho}}})\right]\bm{1}\right\|_\infty\\
			&=\left\|\bm{1}-(1-\gamma)\mathcal{K}\sum_{i=0}^{n-1}(\gamma CP_{\pi_{\bar{c}}})^iD\bm{1}\right\|_\infty\\
			&\leq 1-\mathcal{K}_{\min}(1-\gamma)\sum_{i=0}^{n-1}(\gamma C_{\min})^iD_{\min}\tag{$0<C(s)\leq D(s)\leq 1$ for all $s$}\\
			&= 1-\mathcal{K}_{\min}\frac{(1-\gamma)(1-(\gamma C_{\min})^n)D_{\min}}{1-\gamma C_{\min}}\\
			&=\beta_2.
		\end{align*}
		It follows that $\left\|\bar{F}(V_1)-\bar{F}(V_2)\right\|_\infty\leq \|G\|_\infty\|V_1-V_2\|_\infty\leq \beta_2\|V_1-V_2\|_\infty$.
		\item It is enough to show that $V_{\pi_{\bar{\rho}}}$ is a fixed-point of $\bar{F}$, the uniqueness follows from $\bar{F}$ being a contraction operator. Using the Bellman equation $V_{\pi_{\bar{\rho}}}=R_{\pi_{\bar{\rho}}}+\gamma P_{\pi_{\bar{\rho}}}V_{\pi_{\bar{\rho}}}$, we have by Eq. (\ref{eq:11}) that
		\begin{align*}
			\bar{F}(V_{\pi_{\bar{\rho}}})
			=V_{\pi_{\bar{\rho}}}+\mathcal{K}\sum_{i=0}^{n-1}(\gamma CP_{\pi_{\bar{c}}})^iD(R_{\pi_{\bar{\rho}}}+\gamma P_{\pi_{\bar{\rho}}}V_{\pi_{\bar{\rho}}}-V_{\pi_{\bar{\rho}}})=V_{\pi_{\bar{\rho}}}.
		\end{align*}
	\end{enumerate}	
\end{enumerate}

\subsection{Proof of Theorem \ref{thm:Vtrace}}
We will apply Theorem \ref{thm:sa} and Lemma \ref{le:constants} (2) to the V-trace algorithm. We begin by identifying the constants:
\begin{align*}
	A&=A_1+A_2+1\leq 2(\bar{\rho}+1)\eta(\gamma,\bar{c}), \;B=B_1+B_2=\bar{\rho} \eta(\gamma,\bar{c}),\;\varphi_1\leq 3,\;\varphi_2\geq \frac{1-\beta_2}{2},\;\varphi_3\leq \frac{456e\log(|\mathcal{S}|)}{1-\beta_2}\\
	c_1&\leq (\|V_0-V_{\pi_{\bar{\rho}}}\|_\infty+\|V_0\|_\infty+1)^2,\;c_2=4(\bar{\rho}+1)^2\eta(\gamma,\bar{c})^2(\|V_{\pi_{\bar{\rho}}}\|_\infty+1)^2.
\end{align*}
Now we apply Theorem \ref{thm:sa} (2) (a). When $\alpha_k=\alpha$ for all $k\geq 0$, where $\alpha$ is chosen such that
\begin{align*}
    \alpha (t_\alpha(\mathcal{M}_S)+n)\leq\frac{\varphi_2}{\varphi_3A^2}= \frac{(1-\beta_2)^2}{7296e(\bar{\rho}+1)^2\eta^2(\gamma,\bar{c})\log(|\mathcal{S}|)},
\end{align*}
we have for all $k\geq t_\alpha+n$:
\begin{align*}
	\mathbb{E}[\|V_k-V_{\pi_{\bar{\rho}}}\|_\infty^2]
	\leq\;& \varphi_1c_1(1-\varphi_2\alpha)^{k-K_2}+\frac{\varphi_3c_2}{\varphi_2}\alpha t_\alpha(\mathcal{M}_Y)\\
	\leq\;& 3(\|V_0-V_{\pi_{\bar{\rho}}}\|_\infty+\|V_0\|_\infty+1)^2\left(1-\frac{1-\beta_2}{2}\alpha\right)^{k-K_2}\\
	&+\frac{3648e\log(|\mathcal{S}|)}{(1-\beta_2)^2}(\bar{\rho}+1)^2\eta(\gamma,\bar{c})^2(\|V_{\pi_{\bar{\rho}}}\|_\infty+1)^2\alpha(t_\alpha(\mathcal{M}_S)+n)\\
	=\;&c_{V,1}\left(1-\frac{1-\beta_2}{2}\alpha\right)^{k-K_2}+c_{V,2}\frac{\log(|\mathcal{S}|)}{(1-\beta_2)^2}(\bar{\rho}+1)^2\eta(\gamma,\bar{c})^2\alpha(t_\alpha(\mathcal{M}_S)+n),
\end{align*}
where $c_{V,1}=3(\|V_0-V_{\pi_{\bar{\rho}}}\|_\infty+\|V_0\|_\infty+1)^2$ and $c_{V,2}=3648e(\|V_{\pi_{\bar{\rho}}}\|_\infty+1)^2$.

\subsection{V-trace with Diminishing Stepsizes}
We here only present using linear stepsize that achieves the optimal convergence rate (Theorem \ref{thm:sa} (2) (b) (iii)).
\begin{theorem}\label{thm:V-trace-diminishing}
	Consider $\{V_k\}$ of Algorithm (\ref{algo:V-trace}). Suppose Assumption \ref{as:Vtrace} is satisfied and $\alpha_k=\frac{\alpha}{k+h}$ with $\alpha=\frac{4}{1-\beta_2}$ and properly chosen $h$. Then there exists $K_2'>0$ such that the following inequality holds for all $k\geq K_2'$:
	\begin{align*}
		\mathbb{E}[\|V_k-V_{\pi_{\bar{\rho}}}\|_\infty^2]\leq c_{V,1}'\frac{K_2'+h}{k+h}+c_{V,2}'\frac{\log(|\mathcal{S}|)}{(1-\beta_2)^3}(\bar{\rho}+1)^2\eta(\gamma,\bar{c})^2\frac{t_k(\mathcal{M}_S)+n}{k+h},
	\end{align*}
	where $c_{V,1}'=3\|V_0-V_{\pi_{\bar{\rho}}}\|_\infty+\|V_0\|_\infty+1)^2$ and $c_{V,2}'=233472e^2(\|V_{\pi_{\bar{\rho}}}\|_\infty+1)^2$.
\end{theorem}
\begin{proof}[Proof of Theorem \ref{thm:V-trace-diminishing}]
	The corresponding constants have been identified in the proof of Theorem \ref{thm:Vtrace}. Now apply Theorem \ref{thm:sa} (2) (c). When $\alpha_k=\frac{\alpha}{k+h}$ with $\alpha=\frac{4}{1-\beta_2}$ and properly chosen $h$, there exists $K_2'>0$ such that we have for all $k\geq K_2'$:
	\begin{align*}
		\mathbb{E}[\|V_k-V_{\pi_{\bar{\rho}}}\|_\infty^2]\leq\;& \varphi_1c_1\left(\frac{K_2'+h}{k+h}\right)^{\varphi_2\alpha}+\frac{8e\alpha^2\varphi_3c_2}{\varphi_2\alpha-1}\frac{t_k(\mathcal{M}_Y)}{k+h}\\
		\leq\;& 3\|V_0-V_{\pi_{\bar{\rho}}}\|_\infty+\|V_0\|_\infty+1)^2\frac{K_2'+h}{k+h}\\
		&+233472e^2 \frac{\log(|\mathcal{S}|)}{(1-\beta_2)^3}(\bar{\rho}+1)^2\eta(\gamma,\bar{c})^2(\|V_{\pi_{\bar{\rho}}}\|_\infty+1)^2\frac{t_k(\mathcal{M}_S)+n}{k+h}\\
		=\;&c_{V,1}'\frac{K_2'+h}{k+h}+c_{V,2}'\frac{\log(|\mathcal{S}|)}{(1-\beta_2)^3}(\bar{\rho}+1)^2\eta(\gamma,\bar{c})^2\frac{t_k(\mathcal{M}_S)+n}{k+h},
	\end{align*}
	where $c_{V,1}'=3\|V_0-V_{\pi_{\bar{\rho}}}\|_\infty+\|V_0\|_\infty+1)^2$ and $c_{V,2}'=233472e^2(\|V_{\pi_{\bar{\rho}}}\|_\infty+1)^2$.
\end{proof}

\section{$n$-step TD}\label{ap:TDn}

\subsection{Proof of Proposition \ref{prop:TDn}}
\begin{enumerate}[(1)]
	\item 
	\begin{enumerate}[(a)]
		\item For any $V_1,V_2\in\mathbb{R}^{|\mathcal{S}|}$ and $y\in\mathcal{Y}$, we have
		\begin{align*}
			&\|F(V_1,y)-F(V_2,y)\|_2\\
			=\;&\left(\sum_{s\in\mathcal{S}}\left[\mathbbm{1}_{\{s_0=s\}}\left(\gamma^n (V_1(s_n)-V_2(s_n))-(V_1(s_0)-V_2(s_0))\right)+V_1(s)-V_2(s)\right]^2\right)^{1/2}\\
			\leq \;&\left(\sum_{s\in\mathcal{S}}[\mathbbm{1}_{\{s_0=s\}}(\gamma^n+1)\|V_1-V_2\|_2]^2\right)^{1/2} +\|V_1-V_2\|_2\tag{Triangle inequality}\\
			\leq\;&  3\|V_1-V_2\|_2.
		\end{align*}
		\item For any $y\in\mathcal{Y}$, we have
		\begin{align*}
			\|F(\bm{0},y)\|_2^2=\sum_{s\in\mathcal{S}}\left(\mathbbm{1}_{\{s_0=s\}}\sum_{i=0}^{n-1}\gamma^{i}\mathcal{R}(s_i,a_i)\right)^2\leq \sum_{s\in\mathcal{S}}\mathbbm{1}_{\{s_0=s\}}\left(\sum_{i=0}^{n-1}\gamma^{i}\right)^2\leq \frac{1}{(1-\gamma)^2}.
		\end{align*}
		It follows that $\|F(\bm{0},y)\|_2\leq \frac{1}{1-\gamma}$.
	\end{enumerate}
	\item The proof is identical to that of Proposition \ref{prop:Q-learning} (2).
	\item 
	\begin{enumerate}
		\item Since $n$-step TD is a special case of V-trace, we can directly apply Proposition \ref{prop:Vtrace} (3) (a) here. Observe that when $\pi=\pi_b$ and $\bar{c}=\bar{\rho}=1$, we have $C=D=I$ and $P_{\pi_{\bar{c}}}=P_{\pi_{\bar{\rho}}}=P_{\pi}$. Hence we have
		\begin{align*}
			\bar{F}(V)=\left[I-\mathcal{K}\sum_{i=0}^{n-1}(\gamma P_{\pi})^i(I-\gamma P_{\pi})\right]V+\mathcal{K}\sum_{i=0}^{n-1}(\gamma P_{\pi})^iR_{\pi}.
		\end{align*}
		\item For any $V_1,V_2\in\mathbb{R}^{|\mathcal{S}|}$ and $p\geq 1$, we have
		\begin{align*}
			\|\bar{F}(V_1)-\bar{F}(V_2)\|_p&=\left\|\left[I-\mathcal{K}\sum_{i=0}^{n-1}(\gamma P_{\pi})^i(I-\gamma P_{\pi})\right](V_1-V_2)\right\|_p\\
			&\leq \left\|I-\mathcal{K}\sum_{i=0}^{n-1}(\gamma P_{\pi})^i(I-\gamma P_{\pi})\right\|_p\|V_1-V_2\|_p.
		\end{align*}
		For simplicity of notation, we denote $G=I-\mathcal{K}\sum_{i=0}^{n-1}(\gamma P_{\pi})^i(I-\gamma P_{\pi})$. Since $G$ has non-negative entries (established in the proof of Proposition \ref{prop:Vtrace} (3) (b)), we have
		\begin{align*}
			\|G\|_\infty=\|G\bm{1}\|_\infty=\left\|\bm{1}-\kappa\sum_{i=0}^{n-1}\gamma^i(1-\gamma)\right\|_\infty=1-\mathcal{K}_{\min}(1-\gamma^n).
		\end{align*}
		Moreover, using the fact that $\kappa$ is the stationary distribution of $P_\pi$ (i.e., $\kappa^\top P_\pi=\kappa^\top$), we have
		\begin{align*}
			\|G\|_1=\|\bm{1}^\top G\|_\infty=\left\|\bm{1}^\top-\kappa^\top\sum_{i=0}^{n-1}\gamma^i(1-\gamma)\right\|_\infty=1-\mathcal{K}_{\min}(1-\gamma^n).
		\end{align*}
		To proceed, we need the following lemma.
		\begin{lemma}\label{le:matrix}
			Let $G\in\mathbb{R}^{d\times d}$ be a matrix with non-negative entries. Then we have for all $p\in [1,\infty]$:
			\begin{align*}
				\|G\|_p\leq \|G\|_1^{1/p}\|G\|_\infty^{1-1/p}.
			\end{align*}
		\end{lemma}
		\begin{proof}[Proof of Lemma \ref{le:matrix}]
			The result clearly holds when $p=1$ or $p=\infty$. Now consider $p\in (1,\infty)$. Using the definition of induced matrix norm, we have for any $x\neq 0$:
			\begin{align*}
				\|Gx\|_p^p&=\sum_{i=1}^d\left(\sum_{j=1}^dG_{ij}x_j\right)^p\\
				&=\sum_{i=1}^d[G\bm{1}]_i^p\left(\sum_{j=1}^d\frac{G_{ij}}{[G\bm{1}]_i}x_j\right)^p\\
				&\leq \sum_{i=1}^d[G\bm{1}]_i^{p-1}\sum_{j=1}^dG_{ij}x_j^p\tag{Jensen's inequality}\\
				&\leq \|G\|_\infty^{p-1}\sum_{j=1}^dx_j^p\sum_{i=1}^dG_{ij}\\
				&= \|G\|_\infty^{p-1}\sum_{j=1}^dx_j^p[\bm{1}^\top G]_j\\
				&\leq  \|G\|_\infty^{p-1}\|G\|_1\|x\|^p_p.
			\end{align*}
			It follows that $\|G\|_p\leq \|G\|_1^{1/p}\|G\|_\infty^{1-1/p}$. 
		\end{proof}
		Using Lemma \ref{le:matrix} and we have
		\begin{align*}
			\|G\|_p\leq \|G\|_1^{1/p}\|G\|_\infty^{1-1/p}\leq 1-\mathcal{K}_{\min}(1-\gamma^n)=\beta_3.
		\end{align*}
		Therefore, we have $\|\bar{F}(V_1)-\bar{F}(V_2)\|_2\leq \beta_3\|V_1-V_2\|_2$. Hence the operator $\bar{F}(\cdot)$ is a contraction mapping with respect to $\|\cdot\|_2$, with contraction factor $\beta_3$.
		\item The proof is identical to that of Proposition \ref{prop:Vtrace} (3) (c).
	\end{enumerate}
\end{enumerate}

\subsection{Proof of Theorem \ref{thm:TDn}}

We will apply Theorem and Lemma \ref{le:constants} (1) to the $n$-step TD algorithm. We begin by identifying the constants:
\begin{align*}
	A&=A_1+A_2+1=4, \;B=B_1+B_2=\frac{1}{1-\gamma},\;\varphi_1\leq 1,\;\varphi_2\geq 1-\beta_3,\;\varphi_3\leq228\\
	c_1&\leq (\|V_0-V_\pi\|_2+\|V_0\|_2+4)^2,\;c_2=\frac{1}{(1-\gamma)^2}(4(1-\gamma)\|V_\pi\|_2+1)^2.
\end{align*}
Now apply Theorem \ref{thm:sa} (2) (a). When $\alpha_k=\alpha$ for all $k\geq 0$, where $\alpha$ is chosen such that
\begin{align*}
    \alpha (t_\alpha(\mathcal{M}_S)+n)\leq \frac{\varphi_2}{\varphi_3A^2}= \frac{1-\beta_3}{3648},
\end{align*}
we have for all $k\geq t_\alpha(\mathcal{M}_S)+n$:
\begin{align*}
	\mathbb{E}[\|V_k-V_\pi\|_2^2]
	\leq\;& \varphi_1c_1(1-\varphi_2\alpha)^{k-(\alpha (t_\alpha(\mathcal{M}_S)+n))}+\frac{\varphi_3c_2}{\varphi_2}\alpha t_\alpha(\mathcal{M}_Y)\\
	\leq \;&(\|V_0-V_\pi\|_2+\|V_0\|_2+4)^2(1-(1-\beta_3)\alpha)^{k-(\alpha (t_\alpha(\mathcal{M}_S)+n))}\\
	&+\frac{228}{1-\beta_3}\frac{1}{(1-\gamma)^2}(4(1-\gamma)\|V_\pi\|_2+1)^2\alpha(t_\alpha(\mathcal{M}_S)+n)\\
	=\;&\hat{c}_1(1-(1-\beta_3)\alpha)^{k-(\alpha (t_\alpha(\mathcal{M}_S)+n))}+\hat{c}_2\frac{\alpha(t_\alpha(\mathcal{M}_S)+n)}{(1-\beta_3)(1-\gamma)^2},
\end{align*}
where $\hat{c}_1=(\|V_0-V_\pi\|_2+\|V_0\|_2+4)^2$ and $\hat{c}_2=228(4(1-\gamma)\|V_\pi\|_2+1)^2$.

\subsection{$n$-Step TD with Diminishing Stepsizes}
For $n$-step TD with diminishing stepsize, we here only present the result for using linear diminishing stepsize that achieves the optimal convergence rate (Theorem \ref{thm:sa} (2) (b) (iii)).

\begin{theorem}\label{thm:TDn-diminishing}
	Consider $\{V_k\}$ of Algorithm (\ref{algo:TDn}). Suppose that Assumption \ref{as:TDn} is satisfied and $\alpha_k=\frac{\alpha}{k+h}$ with $\alpha=\frac{2}{1-\beta_3}$ and properly chosen $h$. Then there exists $K_3'>0$ such that the following inequality holds for all $k\geq K_3'$:
	\begin{align*}
		\mathbb{E}[\|V_k-V_\pi\|_c^2]\leq \hat{c}_1'\frac{K_3'+h}{k+h}+\hat{c}_2'\frac{t_k(\mathcal{M}_S)+n}{(1-\beta_3)^2(1-\gamma)^2(k+h)},
	\end{align*}
	where $\hat{c}_1'=(\|V_0-V_\pi\|_2+\|V_0\|_2+4)^2$ and $\hat{c}_2'=7296e(4(1-\gamma)\|V_\pi\|_2+1)^2$.
\end{theorem}
\begin{proof}[Proof of Theorem \ref{thm:TDn-diminishing}]
	The constants are already identified in the proof of Theorem \ref{thm:TDn}. Apply Theorem \ref{thm:sa}) (2) (b) (iii), when $\alpha_k=\frac{\alpha}{k+h}$ with $\alpha=\frac{2}{1-\beta_3}$ and properly chosen $h$, there exists $K_3'>0$ such that we have for all $k\geq K_3'$:
	\begin{align*}
		&\mathbb{E}[\|V_k-V_\pi\|_c^2]\\
		\leq\;& 
		\varphi_1c_1\left(\frac{K_3'+h}{k+h}\right)^{\varphi_2\alpha}+\frac{8e\alpha^2\varphi_3c_2}{\varphi_2\alpha-1}\frac{t_k(\mathcal{M}_Y)}{k+h}\\
		\leq \;&(\|V_0-V_\pi\|_2+\|V_0\|_2+4)^2\frac{K_3'+h}{k+h}+\frac{7296e}{(1-\beta_3)^2}\frac{1}{(1-\gamma)^2}(4(1-\gamma)\|V_\pi\|_2+1)^2\frac{t_k(\mathcal{M}_Y)}{k+h}\\
		=\;&\hat{c}_1'\frac{K_3'+h}{k+h}+\hat{c}_2'\frac{t_k(\mathcal{M}_S)+n}{(1-\beta_3)^2(1-\gamma)^2(k+h)},
	\end{align*}
	where $\hat{c}_1'=(\|V_0-V_\pi\|_2+\|V_0\|_2+4)^2$ and $\hat{c}_2'=7296e(4(1-\gamma)\|V_\pi\|_2+1)^2$.
\end{proof}

\section{TD$(\lambda)$}\label{ap:TDlambda}

The following lemma is useful when proving Lemma \ref{le:truncation} and Proposition \ref{prop:TDlambda}.
 
\begin{lemma}\label{le:summation}
    Let $\mathcal{I}$ be a finite set. For any $k\geq 0$, define two sequences $\{i_t\}_{0\leq t\leq k}$ and $\{a_t\}_{0\leq t\leq k}$ be such that $i_t\in\mathcal{I}$ and $a_t\geq 0$ for all $t=0,1,...,k$. Let $x\in\mathbb{R}^{|\mathcal{I}|}$ be defined by $x_i=\sum_{t=0}^ka_t\mathbbm{1}_{\{i_t=i\}}$ for all $i\in\mathcal{I}$. Then we have
    \begin{align*}
        \|x\|_2\leq \sum_{t=0}^ka_t.
    \end{align*}
\end{lemma}
\begin{proof}[Proof of Lemma \ref{le:summation}]
    Using the definition of $\|\cdot\|_2$, we have
    \begin{align*}
        \|x\|_2^2&=\sum_{i\in\mathcal{I}}\left(\sum_{t=0}^ka_t\mathbbm{1}_{\{i_t=i\}}\right)^2\\
        &=\sum_{i\in\mathcal{I}}\sum_{t=0}^k\sum_{\ell=0}^k a_ta_\ell\mathbbm{1}_{\{i_t=i,i_\ell=i\}}\\
        &=\sum_{t=0}^k\sum_{\ell=0}^k a_ta_\ell\sum_{i\in\mathcal{I}}\mathbbm{1}_{\{i_t=i,i_\ell=i\}}\\
        &\leq \sum_{t=0}^k\sum_{\ell=0}^k a_ta_\ell\\
        &=\left(\sum_{t=0}^ka_t\right)^2.
    \end{align*}
    The result follows by taking square root on both sides of the previous inequality.
\end{proof}

\subsection{Proof of Lemma \ref{le:truncation}}
For any $V\in\mathbb{R}^{|\mathcal{S}|}$ and $(s_0,...,s_k,a_k,s_{k+1})$, we have by definition of the operators $F_k^\tau(\cdot,\cdot)$ and $F_k(\cdot,\cdot)$ that
\begin{align*}
	&\|F_k^\tau(V,s_{k-\tau},...,s_k,a_k,s_{k+1})-F_k(V,s_0,...,s_k,a_k,s_{k+1})\|_2^2\\
	=\;&\sum_{s\in\mathcal{S}}\left[\left(\mathcal{R}(s_k,a_k)+\gamma V(s_{k+1})-V(s_k)\right)\sum_{i=0}^{k-\tau-1}(\gamma \lambda)^{k-i}\mathbbm{1}_{\{s_i=s\}}\right]^2\\
	\leq \;&(1+2\|V\|_2)^2\sum_{s\in\mathcal{S}}\left[\sum_{i=0}^{k-\tau-1}(\gamma \lambda)^{k-i}\mathbbm{1}_{\{s_i=s\}}\right]^2\\
	=\;&\frac{(\gamma\lambda)^{2(\tau+1)}}{(1-\gamma\lambda)^2}(1+2\|V\|_2)^2.\tag{Lemma \ref{le:summation}}
\end{align*}
The result follows by taking the square root on both sides of the previous inequality.

\subsection{Proof of Proposition \ref{prop:TDlambda}}
\begin{enumerate}[(1)]
	\item
	For any $V_1,V_2\in\mathbb{R}^{|\mathcal{S}|}$ and $y\in \mathcal{Y}_\tau$, we have by Triangle inequality that
	\begin{align*}
	    &\|F_k^\tau(V_1,y)-F_k^\tau(V_2,y)\|_2\\
	    \leq\;& \|V_1-V_2\|_2+\left(\sum_{s\in\mathcal{S}}\left[\left(\gamma (V_1(s_{k+1})-V_2(s_{k+1}))-(V_1(s_k)-V_2(s_k))\right)\sum_{i=k-\tau}^{k}(\gamma \lambda)^{k-i}\mathbbm{1}_{\{s_i=s\}}\right]^2\right)^{1/2}\\
	    \leq\;& \|V_1-V_2\|_2+2\|V_1-V_2\|_2\left(\sum_{s\in\mathcal{S}}\left[\sum_{i=k-\tau}^{k}(\gamma \lambda)^{k-i}\mathbbm{1}_{\{s_i=s\}}\right]^2\right)^{1/2}\\
	    \leq \;&\|V_1-V_2\|_2+\frac{2}{1-\gamma\lambda}\|V_1-V_2\|_2\tag{Lemma \ref{le:summation}}\\
	    \leq \;&\frac{3}{1-\gamma\lambda}\|V_1-V_2\|_2.
	\end{align*}
	
	Similarly, for any $y\in\mathcal{Y}_\tau$, we have
	\begin{align*}
		\|F_k^\tau(\bm{0},y)\|_2^2&=\sum_{s\in\mathcal{S}}\left[\mathcal{R}(s_k,a_k)\sum_{i=k-\tau}^{k}(\gamma \lambda)^{k-i}\mathbbm{1}_{\{s_i=s\}}\right]^2\\
		&\leq \sum_{s\in\mathcal{S}}\left[\sum_{i=k-\tau}^{k}(\gamma \lambda)^{k-i}\mathbbm{1}_{\{s_i=s\}}\right]^2\tag{$\mathcal{R}(s,a)\in [0,1]$ for all $(s,a)$}\\
		&\leq \frac{1}{(1-\gamma\lambda)^2}.\tag{Lemma \ref{le:summation}}
	\end{align*}
	It follows that $\|F_k^\tau(\bm{0},y)\|_2\leq \frac{1}{1-\gamma\lambda}$.
	\item The proof is identical to that of Propositon \ref{prop:Q-learning} (2).
	\item 
	\begin{enumerate}
		\item For any $V\in\mathbb{R}^{|\mathcal{S}|}$ and $s\in\mathcal{S}$, we have
		\begin{align*}
			&\mathbb{E}_{Y\sim \mu}\left[[F_k^\tau(V,Y)](s)\right]\\
			=\;&\mathbb{E}_{Y\sim \mu}\left[\left(\mathcal{R}(S_k,A_k)+\gamma V(S_{k+1})-V(S_k)\right)\sum_{i=k-\tau}^{k}(\gamma \lambda)^{k-i}\mathbbm{1}_{\{S_i=s\}}\right]+V(s)\\
			=\;&\mathbb{E}_{Y\sim \mu}\left[\sum_{i=k-\tau}^{k}(\gamma \lambda)^{k-i}\mathbbm{1}_{\{S_i=s\}}\mathbb{E}\left[\left(\mathcal{R}(S_k,A_k)+\gamma V(S_{k+1})-V(S_k)\right)\;\middle|\;S_k,S_{k-1},...,S_0\right]\right]+V(s)\\
			=\;&\mathbb{E}_{Y\sim \mu}\left[\sum_{i=k-\tau}^{k}(\gamma \lambda)^{k-i}\mathbbm{1}_{\{S_i=s\}}(R_{\pi}(S_k)+\gamma [P_{\pi} V](S_k)-V(S_k))\right]+V(s)\\
			=\;&\sum_{i=k-\tau}^{k}(\gamma \lambda)^{k-i}\sum_{s_0\in\mathcal{S}}\kappa(s_0)P_{\pi}^i(s_0,s)\sum_{s'\in\mathcal{S}}P_{\pi}^{k-i}(s,s')(R_{\pi}(s')+\gamma [P_{\pi} V](s')-V(s')) +V(s)\\
			=\;&\kappa(s)\sum_{i=k-\tau}^{k}(\gamma \lambda)^{k-i}\sum_{s'\in\mathcal{S}}P_{\pi}^{k-i}(s,s')(R_{\pi}(s')+\gamma [P_{\pi} V](s')-V(s')) +V(s)\\
			=\;&\kappa(s)\sum_{i=k-\tau}^{k}(\gamma \lambda)^{k-i}[P_{\pi}^{k-i}(R_{\pi}+\gamma P_{\pi} V-V)](s)+V(s).
		\end{align*}
		It follows that
		\begin{align*}
			\bar{F}_k^\tau(V)&=\mathcal{K}\sum_{i=k-\tau}^{k}(\gamma \lambda P_{\pi})^{k-i}(R_{\pi}+\gamma P_{\pi} V-V)+V\\
			&=\mathcal{K}\sum_{i=0}^{\tau}(\gamma \lambda P_{\pi})^{i}(R_{\pi}+\gamma P_{\pi} V-V)+V\\
			&=\left[I-\mathcal{K}\sum_{i=0}^{\tau}(\gamma \lambda P_{\pi})^{i}( I-\gamma P_{\pi})\right]V+\mathcal{K}\sum_{i=0}^{\tau}(\gamma \lambda P_{\pi})^{i}R_{\pi}.
		\end{align*}
		\item For any $V_1,V_2\in\mathbb{R}^{|\mathcal{S}|}$ and $p\in [1,\infty]$, we have
		\begin{align*}
			\|\bar{F}_k^\tau(V_1)-\bar{F}_k^\tau(V_2)\|_p&=\left\|\left[I-\mathcal{K}\sum_{i=0}^{\tau}(\gamma \lambda P_{\pi})^{i}( I-\gamma P_{\pi})\right](V_1-V_2)\right\|_p\\
			&\leq \left\|I-\mathcal{K}\sum_{i=0}^{\tau}(\gamma \lambda P_{\pi})^{i}( I-\gamma P_{\pi})\right\|_p\|V_1-V_2\|_p.
		\end{align*}
		Denote $G=I-\mathcal{K}\sum_{i=0}^{\tau}(\gamma \lambda P_{\pi})^{i}( I-\gamma P_{\pi})$. It remains to provide an upper bound on $\|G\|_p$. Since
		\begin{align*}
			G&=I-\mathcal{K}\sum_{i=0}^{\tau}(\gamma \lambda P_{\pi})^{i}+\mathcal{K}\sum_{i=0}^{\tau}(\gamma \lambda P_{\pi})^{i}\gamma P_{\pi}\\
			&=I-\mathcal{K}-\mathcal{K}\sum_{i=1}^{\tau}(\gamma \lambda P_{\pi})^{i}+\mathcal{K}\sum_{i=0}^{\tau}(\gamma \lambda P_{\pi})^{i}\gamma P_{\pi}\\
			&=I-\mathcal{K}-\mathcal{K}\sum_{i=0}^{\tau-1}(\gamma \lambda P_{\pi})^{i+1}+\mathcal{K}\sum_{i=0}^{\tau}(\gamma \lambda P_{\pi})^{i}\gamma P_{\pi}\\
			&=I-\mathcal{K}+\mathcal{K}\sum_{i=0}^{\tau-1}(\gamma \lambda P_{\pi})^{i}\gamma P_{\pi}(1-\lambda)+\mathcal{K}(\gamma \lambda P_{\pi})^{\tau}\gamma P_{\pi},
		\end{align*}
		the matrix $G_{\lambda,\tau}$ has non-negative entries. Therefore, we have
		\begin{align*}
			\|G_{\lambda,\tau}\|_\infty=\|G_{\lambda,\tau}\bm{1}\|_\infty=\left\|\bm{1}-\kappa\frac{(1-\gamma)(1-(\gamma\lambda)^{\tau+1})}{1-\gamma\lambda}\right\|_\infty=1-\mathcal{K}_{\min}\frac{(1-\gamma)(1-(\gamma\lambda)^{\tau+1})}{1-\gamma\lambda}
		\end{align*}
		and
		\begin{align*}
			\|G_{\lambda,\tau}\|_1=\|\bm{1}^\top G_{\lambda,\tau}\|_\infty=\left\|\bm{1}^\top-\kappa^\top\frac{(1-\gamma)(1-(\gamma\lambda)^{\tau+1})}{1-\gamma\lambda}\right\|_\infty=1-\mathcal{K}_{\min}\frac{(1-\gamma)(1-(\gamma\lambda)^{\tau+1})}{1-\gamma\lambda}.
		\end{align*}
		It then follows from Lemma \ref{le:matrix} that
		\begin{align*}
			\|G_{\lambda,\tau}\|_p\leq \|G_{\lambda,\tau}\|_1^{1/p}\|G_{\lambda,\tau}\|_\infty^{1-1/p}\leq 1-\mathcal{K}_{\min}\frac{(1-\gamma)(1-(\gamma\lambda)^{\tau+1})}{1-\gamma\lambda}.
		\end{align*}
		Hence the operator $F_k^\tau(\cdot,\cdot)$ is a contraction with respect to $\|\cdot\|_p$, with a common contraction factor $\beta_4= 1-\mathcal{K}_{\min}\frac{(1-\gamma)(1-(\gamma\lambda)^{\tau+1})}{1-\gamma\lambda}$.
		\item It is enough to show that $V_{\pi}$ is a fixed-point of $\bar{F}_k^\tau(\cdot)$, the uniqueness follows from $\bar{F}_k^\tau(\cdot)$ being a contraction. Using the Bellman equation $R_{\pi}+\gamma P_{\pi} V_{\pi}-V_{\pi}=0$, we have
		\begin{align*}
			\bar{F}_k^\tau(V_{\pi})=\mathcal{K}\sum_{i=0}^{\tau}(\gamma \lambda P_{\pi})^{i}(R_{\pi}+\gamma P_{\pi} V_{\pi}-V_{\pi})+V_{\pi}=V_{\pi}.
		\end{align*}
	\end{enumerate}
\end{enumerate}

\subsection{Proof of Theorem \ref{thm:TDlambda}}\label{pf:thm:TDlambda}
We will exploit the $\|\cdot\|_2$-contraction property of the operator $\bar{F}_k^\tau(\cdot)$ provided in Proposition \ref{prop:TDlambda}. Let $M(x)=\|x\|_2^2$ be our Lyapunov function. Using the update equation (\ref{algo:TDlambda_new}), and we have for all $k\geq 0$:
\begin{align}
	&\|V_{k+1}-V_{\pi}\|_2^2\nonumber\\
	=\;&\|V_k-V_\pi\|_2^2+\underbrace{2\alpha(V_k-V_\pi)^\top \left(\bar{F}_k^\tau(V_k)-V_k\right)}_{\circled{1}}+\underbrace{2\alpha(V_k-V_\pi)^\top \left(F_k^\tau(V_k,Y_k^\tau)-\bar{F}_k^\tau(V_k)\right)}_{\circled{2}}\nonumber\\
	&+\underbrace{\alpha^2\|F_k^\tau(V_k,Y_k^\tau)-V_k\|_2^2}_{\circled{3}}+\underbrace{\alpha^2\|F_k(V_k,Y_k)-F_k^\tau(V_k,Y_k^\tau)\|_2^2}_{\circled{4}}\nonumber\\
	&+\underbrace{2\alpha(V_k-V_\pi)^\top \left(F_k(V_k,Y_k)-F_k^\tau(V_k,Y_k^\tau)\right)}_{\circled{5}}+ \underbrace{2\alpha\left(F_k^\tau(V_k,Y_k^\tau)-V_k\right)^\top\left(F_k(V_k,Y_k)-F_k^\tau(V_k,Y_k^\tau)\right)}_{\circled{6}}.\label{eq:TDlambda:composition}
\end{align}
The terms $\circled{1}$, $\circled{2}$, and $\circled{3}$ correspond to the terms  $T_1$, $T_3$, and $T_4$ in Eq. (\ref{eq:composition1}), and hence can be controlled in the exact same way as provided in Lemmas \ref{le:T1}, \ref{le:T3}, and \ref{le:T4}. The upper bounds of $\circled{1}$, $\circled{2}$, and $\circled{3}$ are summarized in the following lemma, whose proof is omitted. 
\begin{lemma}\label{le:TDlambda1}
    The following inequalities hold:
    \begin{enumerate}[(1)]
        \item $\circled{1}\leq -2\alpha(1-\beta_4)\|V_k-V_\pi\|_2^2$ for any $k\geq \tau$.
        \item $\mathbb{E}[\circled{2}]\leq \frac{662\alpha^2(t_\alpha(\mathcal{M}_S)+\tau)}{(1-\gamma\lambda)^2}\|V_k-V_\pi\|_2^2+\frac{102\alpha^2(t_\alpha(\mathcal{M}_S)+\tau)}{(1-\gamma\lambda)^2}(4\|V_\pi\|_2+1)^2$ for all $k\geq 2\tau+t_\alpha(\mathcal{M}_S)$.
        \item $\circled{3}\leq \frac{32\alpha^2}{(1-\gamma\lambda)^2}\|V_k-V_\pi\|_2^2+\frac{2\alpha^2}{(1-\gamma\lambda)^2}(4\|V_\pi\|_2+1)^2$ for all $k\geq \tau$.
    \end{enumerate}
\end{lemma}

As for the terms $\circled{3}$, $\circled{4}$, and $\circled{5}$, we can easily use Lemma \ref{le:TDlambda} along with the Cauchy-Schwarz inequality to bound them, which gives the following result. 
\begin{lemma}\label{le:TDlambda}
	The following inequalities hold:
	\begin{enumerate}[(1)]
		\item $\circled{4}\leq \frac{8\alpha^2}{(1-\gamma\lambda)^2}\|V_k-V_\pi\|_2^2+\frac{2\alpha^2}{(1-\gamma\lambda)^2}(4\|V_\pi\|_2+1)^2$ for all $k\geq \tau$.
		\item $\circled{5}\leq \frac{16\alpha ^2}{(1-\gamma\lambda)}\|V_k-V_\pi\|_2^2+\frac{4\alpha ^2}{(1-\gamma\lambda)}(4\|V_\pi\|_2+1)^2$ for all $k\geq \tau$.
		\item $\circled{6}\leq \frac{64\alpha^2}{(1-\gamma\lambda)^2}\|V_k-V_\pi\|_2^2+\frac{4\alpha^2}{(1-\gamma\lambda)^2}(4\|V_\pi\|_2+1)^2$ for all $k\geq \tau$.
	\end{enumerate}
\end{lemma}
\begin{proof}[Proof of Lemma \ref{le:TDlambda}]
	\begin{enumerate}[(1)]
		\item For all $k\geq \tau$, we have
		\begin{align*}
			\circled{4}&=\alpha^2\|F_k(V_k,Y_k)-F_k^\tau(V_k,Y_k^\tau)\|_2^2\\
			&\leq \frac{\alpha^2(\gamma\lambda)^{2(\tau+1)}}{(1-\gamma\lambda)^2}(2\|V_k\|_2+1)^2\tag{Lemma \ref{le:truncation}}\\
			&\leq \frac{\alpha^4}{(1-\gamma\lambda)^2}(2\|V_k-V_\pi\|_2+2\|V_\pi\|_2+1)^2\\
			&\leq \frac{8\alpha^2}{(1-\gamma\lambda)^2}\|V_k-V_\pi\|_2^2+\frac{2\alpha^2}{(1-\gamma\lambda)^2}(4\|V_\pi\|_2+1)^2.
		\end{align*}
		\item For all $k\geq \tau$, we have
		\begin{align*}
			\circled{5}&=2\alpha(V_k-V_\pi)^\top \left(F_k(V_k,Y_k)-F_k^\tau(V_k,Y_k^\tau)\right)\\
			&\leq 2\alpha \|V_k-V_\pi\|_2\|F_k(V_k,Y_k)-F_k^\tau(V_k,Y_k^\tau)\|_2\\
			&\leq \frac{2\alpha (\gamma\lambda)^{\tau+1}}{(1-\gamma\lambda)}\|V_k-V_\pi\|_2(2\|V_k\|_2+1)\tag{Proposition \ref{prop:TDlambda} (1)}\\
			&\leq \frac{2\alpha (\gamma\lambda)^{\tau+1}}{(1-\gamma\lambda)}(2\|V_k-V_\pi\|_2+2\|V_\pi\|_2+1)^2\\
			&\leq \frac{16\alpha (\gamma\lambda)^{\tau+1}}{(1-\gamma\lambda)}\|V_k-V_\pi\|_2^2+\frac{4\alpha (\gamma\lambda)^{\tau+1}}{(1-\gamma\lambda)}(4\|V_\pi\|_2+1)^2\\
			&\leq \frac{16\alpha ^2}{(1-\gamma\lambda)}\|V_k-V_\pi\|_2^2+\frac{4\alpha ^2}{(1-\gamma\lambda)}(4\|V_\pi\|_2+1)^2,\tag{The choice of $\tau$}.
		\end{align*}
		\item For all $k\geq \tau$, we have
		\begin{align*}
			\circled{6}&=2\alpha\left(F_k^\tau(V_k,Y_k^\tau)-V_k\right)^\top\left(F_k(V_k,Y_k)-F_k^\tau(V_k,Y_k^\tau)\right)\\
			&\leq 2\alpha\|F_k^\tau(V_k,Y_k^\tau)-V_k\|_2\|F_k(V_k,Y_k)-F_k^\tau(V_k,Y_k^\tau)\|_2\\
			&\leq \frac{2\alpha(\gamma\lambda)^{\tau+1}}{1-\gamma\lambda}\left(\frac{3}{1-\gamma\lambda}\|V_k\|_2+\frac{1}{1-\gamma\lambda}+\|V_k\|_2\right)\left(2\|V_k\|_2+1\right)\\
			&\leq \frac{2\alpha(\gamma\lambda)^{\tau+1}}{(1-\gamma\lambda)^2}(4\|V_k\|_2+1)(2\|V_k\|_2+1)\\
			&\leq \frac{2\alpha(\gamma\lambda)^{\tau+1}}{(1-\gamma\lambda)^2}(4\|V_k-V_\pi\|_2+4\|V_\pi\|_2+1)^2\\
			&\leq \frac{64\alpha(\gamma\lambda)^{\tau+1}}{(1-\gamma\lambda)^2}\|V_k-V_\pi\|_2^2+\frac{4\alpha(\gamma\lambda)^{\tau+1}}{(1-\gamma\lambda)^2}(4\|V_\pi\|_2+1)^2\\
			&\leq \frac{64\alpha^2}{(1-\gamma\lambda)^2}\|V_k-V_\pi\|_2^2+\frac{4\alpha^2}{(1-\gamma\lambda)^2}(4\|V_\pi\|_2+1)^2\tag{The choice of $\tau$}.
		\end{align*}
	\end{enumerate}
\end{proof}

The rest of the proof is to use the upper bounds we derived for the terms $\circled{1}$ to $\circled{6}$ in Eq. (\ref{eq:TDlambda:composition}) to obtain the one-step contractive inequality. Repeatedly using such one-step inequality and we get the finite-sample bounds stated in Theorem \ref{thm:TDlambda}.
\end{document}